%% file: main.tex
\documentclass[11pt]{article}

\usepackage[sort&compress,square,comma,authoryear]{natbib}
\usepackage[a4paper, total={6.1in, 10in}]{geometry}

\usepackage[utf8]{inputenc}
\usepackage[T1]{fontenc}
\usepackage{url}
\usepackage{booktabs}
\usepackage{tabularx}
\newcolumntype{Y}{>{\centering\arraybackslash}X}
\newcolumntype{Z}{>{\raggedleft\arraybackslash}X}
\usepackage[parfill]{parskip}
\usepackage{amsfonts}
\usepackage{nicefrac}
\usepackage{microtype}
\usepackage[table,dvipsnames]{xcolor}
\usepackage{amsmath,amssymb,amsthm}
\usepackage{bbm}                 
\usepackage{colortbl}
\usepackage{graphicx}
\graphicspath{{./figures/}}
\usepackage{mathrsfs,mathtools}
\usepackage{caption}
\usepackage{subcaption}          
\usepackage{wrapfig}
\usepackage{tikz}
\usetikzlibrary{calc, fadings, decorations, shapes, positioning, arrows}
\usepackage[graphicx]{realboxes}
\usepackage{enumitem}
\usepackage{algorithm}
\usepackage{algpseudocode}
\usepackage{makecell}

\usepackage{subcaption}
\usepackage{adjustbox}

\usepackage{bm}
\usepackage{booktabs}
\usepackage{multirow}
\usepackage{comment}

\usepackage{hyperref}
\usepackage[capitalize]{cleveref}

\definecolor{salmon}{RGB}{234,153,153}
\definecolor{cornflowerblue}{RGB}{100,149,237}
\hypersetup{
    colorlinks,
    linkcolor={cornflowerblue},
    citecolor={cornflowerblue},
    urlcolor={salmon}
}

\theoremstyle{plain}
\newtheorem{theorem}{Theorem}
\newtheorem{proposition}[theorem]{Proposition}

\newtheorem{corollary}[theorem]{Corollary}
\theoremstyle{definition}

\theoremstyle{remark}

\newcommand{\change}[1]{{#1}}

\title{Amortized Counterfactual Prediction with Graph Neural Networks}
\date{}

\usepackage{authblk}

\author[1]{Hendrik Stumpf\textsuperscript{\ensuremath{\dagger}}}
\author[2]{Francesco Quinzan\textsuperscript{\ensuremath{\dagger}}}

\affil[1]{Department of Computer Science,
University of Oxford}
\affil[2]{Department of Engineering, University of Oxford}
\affil[ ]{\textsuperscript{\ensuremath{\dagger}}Corresponding authors. Emails: hendrik.stumpf@cs.ox.ac.edu, francesco.quinzan@eng.ox.ac.uk}

\begin{document}

\maketitle

\begin{abstract}
\input{abstract}
\end{abstract}

\input{introduction}

\input{related_work}

\input{background}

\input{methodology}

\input{results}

\input{conclusion}

\subsection*{AI use statement}

In this work, we used generative AI tools for refining writing, formulating mathematical claims, assisting in the writing of proofs, and feedback on research methodology or experiments.
We have not used generative AI tools for generating synthetic data sets, helping develop theoretical models or conceptual frameworks, providing critical ingredients for proving mathematical claims, implementing methods, assisting with translation, cleaning and reformatting datasets, and interpreting results.
Support for qualitative and thematic data analysis and proposing or refining hypotheses are not applicable to this work.
We have reviewed all AI-assisted work. 
We checked that the generated formulation of the claim matches the intended claim. We take responsibility for the final content of this work,
including text, claims or artifacts produced with the aid of generative AI.

\subsection*{Ethics statement}

This work uses existing datasets and collects no new human-subject data. While counterfactual models may have applications in healthcare, their predictions can reflect biases and limitations in the data and should not be used for clinical decision-making without appropriate validation.

\subsection*{Reproducibility statement}

To support reproducibility, we provide the source code and experimental setup. The repository can be found under https://anonymous.4open.science/r/CFNP-2150/. Model architectures and hyperparameters are detailed in the appendix. 

\bibliographystyle{apalike}
\bibliography{references, other_refs_corrected}

\clearpage
\appendix

\section*{Appendix}

\input{appendix/appendix}

\end{document}

%% file: abstract.tex
Counterfactual inference asks how an observed outcome might have differed under different circumstances, such as whether a patient would have recovered had they received a different treatment. Such questions arise whenever decisions depend on the outcomes of alternatives that were not observed. Answering counterfactual queries generally requires knowledge of the underlying causal structure, which is often unavailable in practice. A common approach is therefore to first estimate the causal graph from observational data and then perform counterfactual inference. However, errors in the estimated causal graph can propagate to downstream counterfactual predictions. To address this challenge, we introduce the \textbf{Counterfactual Graph Neural Process (CFGNP)}, an end-to-end approach that answers counterfactual queries directly from observational data without first estimating a causal graph. CFGNP uses Graph Neural Networks (GNNs) to exploit the dual-graph structure by restricting cross-world message passing to corresponding factual--counterfactual variables. \change{CFGNP both offers an approach that does not require knowledge of the causal graph and achieves stronger predictive performance than existing approaches.}
\change{When trained on counterfactual pairs from the same structural causal model used for testing, CFGNP achieves $1.6-7 \times$ lower MSE than baselines that receive the true causal graph and the same counterfactual supervision, across seven benchmark variants.}
Furthermore, on synthetic causal graphs, a model trained only on 10-node graphs maintains strong predictive performance on graphs with up to 30 nodes, without retraining.

%% file: introduction.tex
\section{Introduction}
Understanding cause-and-effect relationships is a central goal in many scientific disciplines. In medicine, researchers may wish to determine whether a treatment improves patient outcomes \citep{Stovitz109, mapk_causal}. In economics, policymakers often aim to understand how changes in policy affect employment or inflation \citep{NBERw29787}. Similar questions arise in biology, social sciences, and machine learning \citep{morgan2007counterfactuals, sanderson2022mendelian, scholkopf2021toward}. Especially with the increased deployment of machine learning systems in settings where interpretability is crucial, understanding cause-and-effect relationships becomes more important \citep{needed_explainability}. In all of these settings, the objective is not merely to observe statistical associations, but to understand how outcomes would change if different actions or conditions were imposed.

Causal inference \citep{Pearl_2009} provides a framework for studying such questions. Unlike conventional statistical learning, causal inference seeks to uncover cause-effect relationships between variables, rather than statistical associations. The framework formalises this with the structural causal model (SCM), which specifies how variables are generated, and how they causally affect one another. Within the SCM framework, causal questions can be distinguished by the type of reasoning they require. Pearl \citep{Pearl_2009} differentiates between three levels: associational, interventional, and counterfactual queries. Associational queries concern relationships observed in data, while interventional queries ask how outcomes change when a variable is actively manipulated \citep{pearl1996causal}. Counterfactual queries go one step further by asking what would have happened in an alternative scenario, given what was observed in the factual world \citep{shpitser2012counterfactuals}.
%
Counterfactual queries are particularly important because they allow reasoning about individualized outcomes and explanations, rather than only population-level effects \citep{nguyen2020counterfactual}.
This makes them harder to solve than the other two query types as we additionally have to take the alternative hypothesis into account.

Traditional approaches to counterfactual inference typically rely on substantial prior knowledge about the underlying cause-effect relationships between variables \citep{xia_neural_2023, rasal_diffusion_2025, de_sousa_ribeiro_counterfactual_2025, nasr-esfahany_counterfactual_2023}. In practice, this often means assuming access to the \emph{causal graph}, which describes these relationships. Sometimes, less information, such as the topological node ordering of the causal graph, is required \citep{le_learning_2025}. However, in many real-world settings, such information is unavailable or highly uncertain. This is particularly the case in domains such as medicine, biology, and the social sciences, where causal relationships may be obscured by latent confounders, incomplete observations, and complex dependencies among variables \citep{weberpals2024principled, ruiz2024unfolding, loh2025causal}.

Hence, if there is no prior knowledge about the causal graph, we must first use observational samples to estimate the causal graph \citep{spirtes2001causation} and then use the resulting causal graph with the existing methods for counterfactual inference. However, this has two drawbacks; first, the search space of causal graphs scales super-exponentially in the number of variables \citep{zheng2018dagstearscontinuousoptimization} and second, choosing the wrong causal graph amplifies errors and may cause the counterfactual to be incorrect \citep{karimi_algorithmic_2020}. 
To avoid explicitly estimating the causal graph, we instead condition the counterfactual prediction directly on observational samples. This allows the model to extract information about the underlying causal structure that is relevant to the counterfactual query, without committing to a single estimated graph. We approach this problem through meta-learning using Neural Processes (NPs) \citep{garnelo_neural_2018}, which learn to map a set of observations to a predictive distribution. Building on Graph Neural Processes \citep{carr2019graphneuralprocessesbayesian}, we further incorporate the structural constraints imposed by the dual graph into the prediction.

Based on this approach, we introduce the \emph{Counterfactual Graph Neural Process} (CFGNP), an end-to-end framework for counterfactual inference under unknown causal structure. An overview of our setting is shown in Figure~\ref{fig:abstract_pipeline}.

\noindent\textbf{Our contributions.}
\begin{itemize}
\item We introduce the \emph{Counterfactual Graph Neural Process} (CFGNP), an end-to-end framework for learning counterfactual distributions without prior knowledge of the underlying causal graph (Section~\ref{sec:method}).
\item We develop a Graph Neural Process architecture that conditions counterfactual predictions on observational samples (Section~\ref{sec:architecture}).
\item \change{We demonstrate the validity of the CFGNP approach by showing that it outperforms existing counterfactual inference methods} on both synthetic and real-world data (Sections~\ref{sec:explicit-experiments} and~\ref{sec:image-experiments}).
\item We evaluate CFGNP on counterfactual queries involving \change{more features than during training and observe strong performance in this setting (Section~\ref{sec:explicit-experiments}).}
\end{itemize}

\begin{figure}[t]
    \centering
    \includegraphics[width=\linewidth]{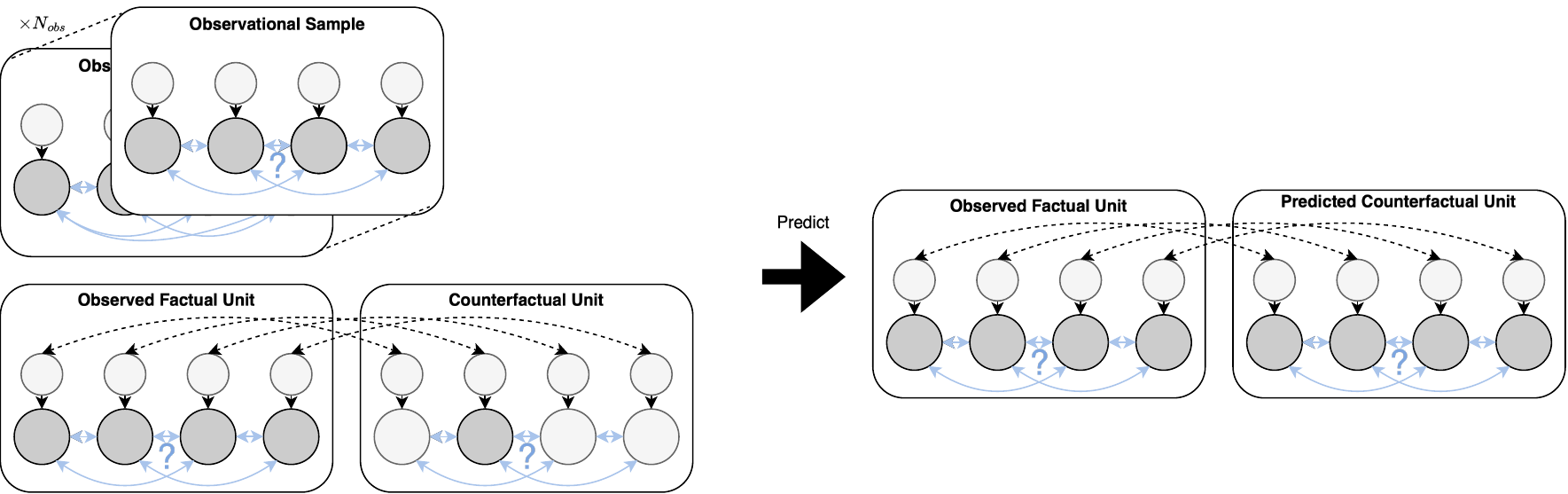}
    \caption{Overview of the task solved by CFGNP. Unknown variables are in light grey and known variables in dark grey. Blue edges represent missing knowledge about the causal graph. The counterfactual is predicted based on the counterfactual query and additional observational samples.}
    \label{fig:abstract_pipeline}
\end{figure}

%% file: related_work.tex
\section{Related Work}
\noindent\textbf{Causal Discovery.}
For causal inference tasks such as evaluating counterfactual queries, 
most existing methods rely on explicit knowledge of the causal graph. 
Estimating this graph from associational data is known as causal discovery \citep{verma2022equivalence}. 
The primary goal of causal discovery is to uncover true cause-and-effect relationships by learning the underlying structural dependencies rather than relying on statistical correlations alone \citep{Pearl_2009,peters2017elements}. Classical methodologies encompass constraint-based approaches, such as the PC algorithm \citep{spirtes2001causation}, and score-based frameworks, such as Greedy Equivalence Search \citep{chickering2002optimal}. More recently, modern techniques including LiNGAM \citep{shimizu06a}, NOTEARS \citep{zheng2018dagstearscontinuousoptimization}, and DAG-GNN \citep{pmlr-v97-yu19a} have advanced the field by introducing functional form assumptions, continuous optimization, and deep learning architectures. Nevertheless, key challenges remain, such as unobserved confounding \citep{pmlr-v202-kaltenpoth23a, pmlr-v213-diepen23a}, complex nonlinearities \citep{li2024nonlinear}, and computational scalability \citep{pmlr-v235-amin24a, cheng2024cuts+}. Addressing these challenges continues to make causal inference an active area of research.

\noindent\textbf{Counterfactual Inference.}
Existing approaches attempt to learn counterfactual distributions under varying degrees of prior structural knowledge \citep{khemakhem_causal_2021, karimi_algorithmic_2020, caron_counterfactual_2022}. For instance, \cite{le_learning_2025} and \cite{javaloy_causal_2023} use normalizing flows that require only a topological ordering of the graph nodes, whereas most alternatives assume the full causal graph. Within deep learning for counterfactual inference, \cite{pawlowski_deep_2020} use normalizing flows to parameterize structural causal models, while recent flow-matching approaches similarly require full graph specification \citep{de_sousa_ribeiro_counterfactual_2025}. Other methods explicitly exploit relational structure through Graph Neural Networks \citep{zecevic_relating_2021, sanchez-martin_vaca_2022}. This reliance on a known, estimated, or otherwise fixed causal structure can be restrictive \citep{wu_counterfactual_2025, caron_counterfactual_2022}: the space of DAGs grows super-exponentially with the number of variables \citep{zheng2018dagstearscontinuousoptimization}, and errors in the assumed causal structure can propagate to downstream counterfactual predictions \citep{karimi_algorithmic_2020}. 
\change{End-to-end learning under uncertain causal structure has been explored for interventional inference \citep{geffner2022deependtoendcausalinference, dhir_estimating_2025}, and Do-PFN \citep{robertson2025dopfnincontextlearningcausal} and CausalPFN \citep{balazadeh2025causalpfnamortizedcausaleffect} pre-train transformers on simulated data answering causal queries in a forward pass. No such approach exists for counterfactual inference, despite exploitable dual-graph constraints \citep{von_kugelgen_backtracking_2023}.
}

\noindent\textbf{Image Counterfactuals.}
Several works focus specifically on image-based counterfactual generation and analysis \citep{pawlowski_deep_2020, rasal_diffusion_2025, luedtke_doublegen_2025, ma2026aialignmentmedicalimaging}. Among these, DeepBC \citep{kladny_deep_2024} performs counterfactual inference through latent exogenous variables under a known causal graph. Other recent approaches based on diffusion models \citep{rasal_diffusion_2025} and optimal transport \citep{de_sousa_ribeiro_counterfactual_2025} similarly assume that the causal graph is known. In contrast, VCI \citep{wu_counterfactual_2025} does not take an explicit causal graph as input, but instead assumes a prescribed treatment--covariate--outcome structure and uses a variational formulation to generate counterfactuals without requiring counterfactual training pairs.

\noindent\textbf{Neural Processes.}
To address these limitations, our approach builds on Neural Processes (NPs) \citep{garnelo_neural_2018}, which have been successfully applied across a range of domains including Bayesian optimization and other scientific applications \citep{kim2019attentive, foong2020meta, bruinsma2023iclr-autoregressive}. NPs and \change{the closely related prior-fitted networks (PFNs; \citet{müller2024transformersbayesianinference})} have recently been explored in causal settings, including treatment-effect estimation and Bayesian causal inference \citep{bynum2026computationalidentifiability, dhir_estimating_2025, robertson2025dopfnincontextlearningcausal, balazadeh2025causalpfnamortizedcausaleffect}. Unlike these works focused on related causal queries, we target counterfactual inference, using GNNs to capture the structure linking factual and counterfactual variables \citep{von_kugelgen_backtracking_2023}.
Graph Neural Processes extend NPs to relational and graph-structured data \citep{carr2019graphneuralprocessesbayesian,liang2021neuralprocessesimprovegraph,Hu_2023}. By combining these ideas, we directly map observational data to the target counterfactual distribution using GNNs over the dual-graph structure, without first estimating a causal graph.

%% file: background.tex
\section{Background}
We now introduce notions of causal inference, neural processes, and graph neural networks.


\noindent\textbf{Notation.}
Unstarred and starred variables denote the factual and counterfactual worlds, respectively. Let $z\in\mathbb{R}^{n\times d_{\mathrm{data}}}$ be the factual observed unit, $T^*=t^*$ a counterfactual condition, and $Y^*\subseteq V$ the target variables. The observational context is $O=\{O_i\}_{i=1}^{N_{\mathrm{obs}}}$, consisting of i.i.d.\ samples from the same SCM $\mathcal{M}$, where $n$ and $N_{\mathrm{obs}}$ denote the number of endogenous variables and context samples, respectively. We write $p(Y^*\mid O,z,T^*=t^*)$ when $\mathcal{M}$ is informed by the observational context $O$, and $p(Y^*\mid\mathcal{M},z,T^*=t^*)$ when emphasizing the model. If $\mathcal{M}$ is unspecified, we assume the true model.

\subsection{Structural Causal Models and Causal Inference}
\noindent\textbf{Structural Causal Model.}
A structural causal model (SCM) describes the causal relationships among a set of endogenous variables $V = \{V_1, \dots, V_n\}$ through a collection of structural equations $\mathbf{f} = \{\mathbf{f}_1, \dots, \mathbf{f}_n\}$. Each variable $V_k$ is determined by its direct causes $Pa_k \subseteq V \setminus {V_k}$ and an exogenous noise variable $U_k$ according to
\[
V_k = \mathbf{f}_k(Pa_k, U_k).
\]
The exogenous variables $U = {U_1, \dots, U_n}$ capture factors outside the model and are assumed to be mutually independent. The causal relationships encoded by the structural equations can be represented by a directed and acyclic causal graph $\mathcal{G}$, where an edge $V_j \rightarrow V_k$ indicates that $V_j \in Pa_k$.
In particular, with $U$ coming from a probability distribution we can define associational queries ($\mathcal{L}_1$) to concern distributions $p(Y)$ with $Y \subset V$. Observational samples $O_i$ are sampled from $p(V)$. Interventional queries ($\mathcal{L}_2$) concern questions $p(Y\mid \operatorname{do}(T=t))$ where we actively set nodes $T \subset V$ to a specific value and consider the outcome $Y \subset V$.
(Backtracking) counterfactual queries ($\mathcal{L}_3$) are probability distributions, with $T^*, Y^*, Z \subset V$ and $T^* \cap Y^* = \emptyset$, of the form
\[
P(Y^*=y^*\mid T^*=t^*,Z=z),
\]
concerning questions such as: given that $Z=z$ happened, what would have happened to variables $Y^*$ if alternatively $T^*=t^*$.
In particular, we use notation $X^*=x^*$ instead of the $\operatorname{do}$-operator to make clear that we are using backtracking counterfactuals \citep{von_kugelgen_backtracking_2023} and not interventionist counterfactuals \citep{Pearl_2009}. While interventionist counterfactuals artificially modify the structural
equations in the counterfactual world, they thereby locally violate the
laws of nature. In contrast, backtracking counterfactuals keep the
structural equations intact and instead consider alternative exogenous
conditions that could have produced $X^*=x^*$.

\begin{wrapfigure}{r}{0.34\textwidth}
    \centering
    \vspace{-1\baselineskip}
    \includegraphics[width=\linewidth]{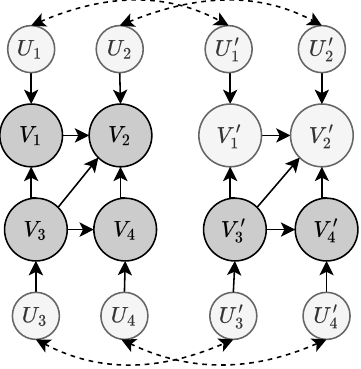}
    \caption{Example of the dual graph for a counterfactual query with a causal graph of four nodes. Light grey nodes are unknown.}
    \label{fig:causal_graph}
    \vspace{-1.5\baselineskip}
\end{wrapfigure}
\noindent\textbf{Dual Graph.}
To answer backtracking counterfactuals, we need the dual graph. Figure~\ref{fig:causal_graph} shows an example of the dual graph for
backtracking counterfactuals. The dual graph consists of two copies of
the same SCM. The left part, the observed factual unit, represents what
actually happened, whereas the copy on the right, the counterfactual
unit, represents the alternative circumstances. Both worlds are
connected via the exogenous nodes, which share a probability
distribution $p(U,U^*)$.
%
%

We can solve the counterfactual query via the
abduction--marginalisation--prediction framework \citep{von_kugelgen_backtracking_2023}, consisting of the
following steps:
\textbf{1.) Abduction:} We infer the joint posterior over the exogenous
nodes $P(U,U^* \mid Z=z,T^*=t^*)$.
\textbf{2.) Marginalisation:} We marginalise out $U$ to obtain
$P(U^* \mid Z=z,T^*=t^*)$.
\textbf{3.) Prediction:} We use the SCM to predict
$P(Y^* \mid Z=z,T^*=t^*)$. 

\noindent\textbf{Identifiability.}
When the underlying causal model is only partially known, a causal query
may not be uniquely determined by the available information. An
$\mathcal{L}_i$ query is identifiable from $\mathcal{L}_j$ information
if all models in the considered model class that agree on the available
$\mathcal{L}_j$ information yield the same answer \citep{Pearl_2009}.
This is relevant to our setting because counterfactual queries are not
necessarily identifiable from observational data alone.

\subsection{Neural Processes and Graph Neural Networks}
\noindent\textbf{Neural Processes.}
Neural Processes (NPs) \citep{garnelo_neural_2018} are probabilistic models that learn distributions over functions from a collection of related tasks. Each task is represented by a context set $C=\{(x_i,y_i)\}_{i=1}^{N_C}$ of observed input-output pairs and a target set $G=\{(x_j,y_j)\}_{j=1}^{N_G}$ at which predictions are made. NPs learn the conditional predictive distribution
$
    p(y_G \mid x_G, C),
$
thereby amortizing inference across tasks by mapping context observations directly to predictive distributions. Extensions of NPs to graph-structured domains replace neural architectures with graph-based models to capture relational structure \citep{carr2019graphneuralprocessesbayesian, ashman2024translation}.

\noindent\textbf{Graph Neural Networks.}
Graph Neural Networks (GNNs) learn representations of graph-structured data by iteratively propagating information between connected nodes. Let $h_i^{(k)}$ denote the representation of node $v_i$ at layer $k$, with $h_i^{(0)}=x_i$. Following the message-passing framework \citep{gilmer2017neural}, a GNN layer aggregates messages from the neighbourhood $\mathcal{N}(i)$ of each node,
\[
m_i^{(k)}
=
\operatorname{AGG}^{(k)}
\left(
\left\{
\phi^{(k)}(h_i^{(k)},h_j^{(k)},e_{ji})
\,\middle|\,
v_j \in \mathcal{N}(i)
\right\}
\right),
\qquad
h_i^{(k+1)}
=
\psi^{(k)}(h_i^{(k)},m_i^{(k)}),
\]
where $\phi^{(k)}$ and $\psi^{(k)}$ are learnable message and update functions, respectively, and $\operatorname{AGG}^{(k)}$ is permutation invariant. For directed graphs, the neighbourhood can be partitioned into incoming and outgoing neighbours, $\mathcal{N}^{-}(i)$ and $\mathcal{N}^{+}(i)$, allowing messages to be propagated separately along each direction. This distinction is particularly relevant for causal graphs, where incoming and outgoing edges represent causal relations in opposite directions. 

%% file: methodology.tex
\section{A GNN Model for Counterfactual Inference without Causal Graph Knowledge} \label{sec:method}
This work focuses on predicting counterfactual distributions directly from observational data without prior knowledge about the causal graph. For each task, we observe a context set
$O=\{z_i\}_{i=1}^{N_{\mathrm{obs}}}$ of i.i.d. observational samples from an unknown SCM $\mathcal{M}=(\mathcal{G},\mathbf{f})$. Given a factual observation $z$ and a counterfactual condition $T^*=t^*$, our goal is to predict the counterfactual distribution $p(Y^* \mid O,z,T^*=t^*)$ over a set of query variables $Y^*$.
The observational context $O$ provides evidence about $\mathcal{M}$. We aim to directly learn the model-averaged counterfactual predictive distribution
\begin{equation}
    p(Y^* \mid O,z,T^*=t^*)
    =
    \mathbb{E}_{\mathcal{M}\sim p(\mathcal{M}\mid O, z)}
    \left[
        p(Y^* \mid \mathcal{M},z,T^*=t^*)
    \right],
\end{equation}
thereby accounting for uncertainty over causal models without explicitly inferring or committing to a particular causal graph.
Existing methods on the other hand try to estimate the causal graph first.

Since we intend to approximate the counterfactual probability distribution, we use the Kullback-Leibler (KL) divergence \citep{kullback1951information} as the minimization objective.
Suppose we have a list $L$ of samples $(\tau,y^*)\sim q$, where $\tau=((z,T^*,t^*, Y^*),O)$ contains the factual representation $z$, counterfactual query, and observation pattern $O$, while $y^*$ contains the corresponding counterfactual target values. Let $p_\theta(y^*\mid\tau)$ denote the distribution predicted by the model with weights $\theta$, where $y^*$ may contain multiple target variables.

We then want to obtain the weights that minimize the KL divergence, hence
\[
    \begin{aligned}
        \theta^* &= \arg \min\limits_\theta \mathbb{E}_{\tau \sim q} \left[
        \text{KL}(p(Y^*=y^* | z, T^*=t^*)\  || \ p_\theta(Y^*=y^* | z, T^*=t^*, O))
        \right] \\
        &= \arg \max \limits_\theta \mathbb{E}_{\tau \sim q}\Big[\mathbb{E}_{y^*| \tau}[\log p_\theta(Y^*=y^* | z, T^*=t^*, O)]\Big] + C
    \end{aligned}
\]
for some constant $C$ independent of the model parameters. During training we assume access to counterfactual pairs, as we sample $y^* |\ \tau \sim p(Y^*=y^* | z, T^*=t^*, O)$.

\noindent\textbf{Recovery of the Counterfactual Distribution.}
In the limit of infinite data and model capacity, minimizing the KL divergence recovers the conditional predictive distribution \citep{bruinisma_thesis}
$
    p(Y^* \mid O,z,T^*=t^*)
$,
with latent quantities, such as the exogenous variables, implicitly marginalized out \citep{gordon2019metalearningprobabilisticinferenceprediction}. If the counterfactual distribution is identifiable from observational data, this predictive distribution coincides with the counterfactual distribution of the ground-truth causal model. Otherwise, the model instead recovers the model-averaged counterfactual distribution over SCMs compatible with the observations.

\subsection{Approach} \label{sec:architecture}
\begin{figure}[t]
\vspace{-5pt}
    \centering
    \includegraphics[width=\linewidth]{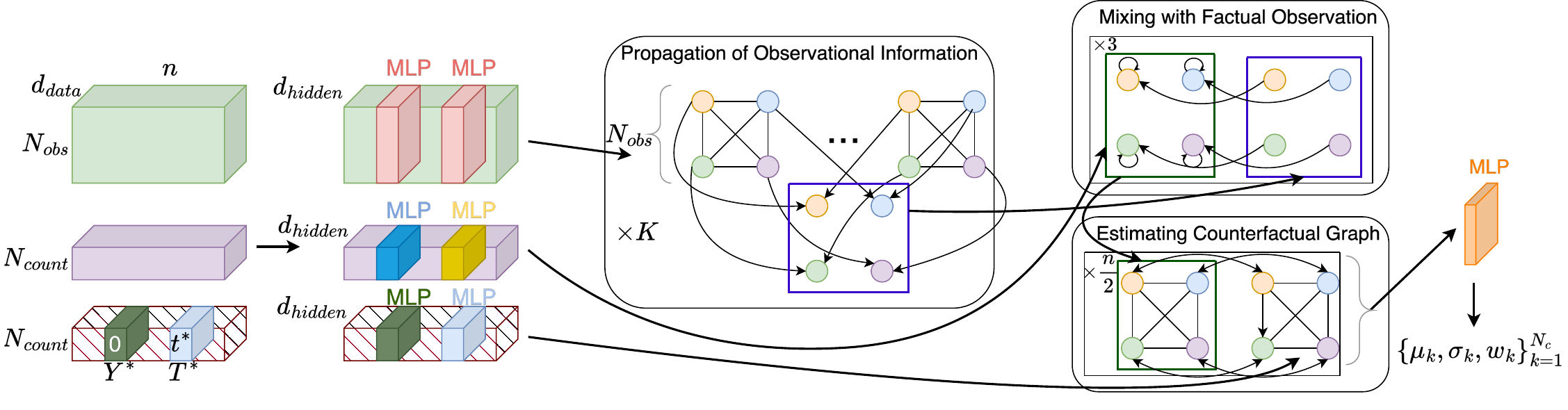}
    \vspace{-20pt}
    \caption{
    Overview of CFGNP. Observational samples are first aggregated into
    node-wise context representations, which are combined with the observed factual
    unit. Message passing between factual and counterfactual
    unit then produces parameters of the counterfactual
    predictive distribution.
    }
    \vspace{-15pt}
    \label{fig:model_architecture}
\end{figure}

Figure~\ref{fig:model_architecture} illustrates the architecture of the Counterfactual Graph Neural Process (CFGNP). Given observational samples $O$, a factual observation $z$, and a counterfactual intervention $x^*$, CFGNP predicts the corresponding counterfactual outcome distribution.
The architecture consists of three stages: summarizing the observational context (middle), combining it with the factual observation (top right), and propagating the resulting representation to the counterfactual unit (bottom right). We describe each component below.

\noindent\textbf{Embedding Input.}
We project all nodes to the hidden dimension $d_{\mathrm{hidden}}$ using node-wise MLPs. For $O$, we apply the same MLP $h_O$ across all nodes. For $z$ and $x^*$, we distinguish intervention nodes $T^*$ from the remaining nodes using separate MLPs: $h_{z,\mathrm{int}}$ and $h_{z,\mathrm{common}}$ for $z$, and $h_{x,\mathrm{int}}$ and $h_{x,\mathrm{common}}$ for $x^*$. In $x^*$, all entries outside $T^*$ are masked.

\noindent\textbf{Observational Context Encoding.}
We next incorporate the observational samples $O$ by allowing nodes within each sample $O_j$ to communicate and compressing their information into $q$. For each variable $i$, we introduce an aggregation node $q_i$, initialized to zero, and connect the corresponding node $i$ from every observational sample to $q_i$. Since we have no prior knowledge of the causal graph, we additionally fully connect the nodes within each observational sample.

The resulting graph consists of $N_{\mathrm{obs}}$ disjoint fully connected $n$-node subgraphs and $n$ aggregation nodes. A $K$-layer GAT \citep{brody2022attentivegraphattentionnetworks} performs message passing,
$
    [O',q'] = m_\theta^{(1)}([O,q]),
$
such that $q'$ captures the compressed observational context.

We then integrate this context with the factual unit $z$. To keep the graph compact, we connect each $q_i'$ only to its corresponding $z_i$ and add a self-loop to each $z_i$. A 3-layer GAT propagates the contextual information to the factual unit,
$
    z' = m_\theta^{(2)}(z,q').
$

\noindent\textbf{Predicting Counterfactual Distribution.}
The final stage aims to approximate the abduction--marginalisation--prediction steps using the context-informed factual representation $z'$ and counterfactual condition $x^*$. The dual graph enables these steps by coupling the factual and counterfactual worlds through connected exogenous nodes. Approximating these steps generally requires bidirectional message passing.





Because we do not have prior knowledge of the causal graph, factual and counterfactual units are fully connected. Thus, each endogenous node has direct access to the information needed to reproduce its corresponding exogenous state, allowing us to remove explicit exogenous nodes and connect only corresponding factual and counterfactual variables across worlds. In this case, a unidirectional GNN can approximate the corresponding bidirectional GNN arbitrarily well (App.~\ref{proposition_bidirectional}).



As a design choice, we apply $\min\{\lfloor n/2 \rfloor, 5\}$ GAT layers, scaling depth with graph size while capping it at five to limit over-smoothing \citep{li2018deeper}. This yields
$
[z'', h^*] = m_\theta^{(3)}(z',x^*)
$,
where $x^*$ encodes $t^*$ at conditioned nodes and masks the rest.

The resulting $h^*$ represents the counterfactual unit. To obtain the counterfactual predictive distribution, a node-wise MLP maps $h^*$ to its distributional parameters. We use a mixture of $N_c$ Gaussians, allowing CFGNP to represent uncertainty arising from ambiguity over the underlying SCM.

\noindent\textbf{Training Objective.}
During training CFGNP minimizes the negative-log-likelihood objective, as this is equivalent to minimizing the KL divergence \citep{murphy2022probabilistic}:
\[
\mathcal{L}_\theta = -\sum\limits_{i=1}^{\vert L \vert} \sum\limits_{j=1}^d \log \Big(\sum_{k=1}^{N_c} w_k(\tau_i)_j \mathcal{N}\left(y^*_{i, j} | \mu_k(\tau_i)_j,\sigma_k(\tau_i)_j\right)\Big).
\]

\noindent\textbf{Model Properties.}
The structure of CFGNP reflects two desirable properties. 
First, predictions should be independent of the ordering of variables. Permuting target nodes $Y^*$ should induce the same permutation in the corresponding predictions, while permutations of the remaining nodes $(Y^*)^c$ should leave predictions for $Y^*$ unchanged. This is provided by design of GNNs (App. \ref{ssec:permutation_proof}).

Second, a trained model should be applicable to graphs of different sizes and arbitrary numbers $N_{obs}$ of observational samples (App. \ref{app:graph_size_trans}). This enables CFGNP to generalize beyond the graph sizes and context sizes encountered during training.

We note for completeness that, for $n$ nodes, data/hidden dimensions $d_{\mathrm{data}}/d_{\mathrm{hidden}}$, $N_{\mathrm{obs}}$ observations, $K$ GNN iterations, and $N_C$ mixture components, one forward pass has complexity (\ref{app:complexity_proof}):
\[
\begin{aligned}
\Theta\!\Big(
& n\,d_{\mathrm{data}}d_{\mathrm{hidden}}N_{\mathrm{obs}}
+ KN_{\mathrm{obs}} n d_{\mathrm{hidden}}(d_{\mathrm{hidden}} + n) + nd_{\mathrm{hidden}}^2 + n^2 d_{\mathrm{hidden}}
+ n d_{\mathrm{hidden}} N_C
\Big)
\end{aligned}.
\]

%% file: results.tex
\section{Results}
We evaluate CFGNP in terms of counterfactual prediction performance, generalization across causal graph sizes, and computational efficiency with additional experiments assessing its ability to recover analytically available counterfactual distributions with increasing context size (App. \ref{sec:id_experiments}).
\change{
App. \ref{app:candle} shows that on CANDLE \citep{reddy2022causally} CFGNP trades effectiveness for preservation. CFGNP has the highest preservation; DeepBC has higher effectiveness and better composition. Ablations in App.~\ref{app:transformer_ablation} show that training on the MSE or using a standard transformer performs worse than CFGNP, and trivial controls (App.~\ref{app:trivial_controls}) perform much worse. App.~\ref{app:chexpert_ood_app} reports an experiment on unseen CheXpert features, where CFGNP does not perform well.
}
\noindent\textbf{Baselines.} We compare against the following baselines \change{which all receive knowledge of the true causal graph (except for VCI which implicitly assumes a 3-node graph)}:
\begin{enumerate}[label={$\bullet$},itemsep=0pt,topsep=0pt,leftmargin=2em]
    \item \textbf{VCI} \citep{wu_counterfactual_2025}: uses a conditional VAE with a counterfactual variational objective under a restricted causal structure.
    \item \textbf{DeepBC} \citep{kladny_deep_2024}: performs counterfactual inference through exogenous variables in a learned SCM, assuming a known causal graph (CG).
    \item \textbf{VACA} \citep{sanchez-martin_vaca_2022}: uses a variational graph autoencoder on a known CG.
    \item \textbf{BGM} \citep{nasr-esfahany_counterfactual_2023}: represents structural equations as invertible mappings between exogenous and observed variables given the causal structure.
    \item \textbf{CFM} \citep{le_learning_2025}: learns flow-based causal mechanisms assuming a given causal ordering.
\end{enumerate}
\change{In App. \ref{app:notear} we test what would happen if we need to learn the graph from the observation; the results get worse than with the true causal graph. Only for Triangle-NADD they slightly improve, even over CFGNP.}
Unlike the baseline methods, CFGNP uses counterfactual pairs during training, but no causal graph. 

\noindent\textbf{Metrics.}
We evaluate counterfactual predictions using the following metrics:
\begin{enumerate}[label={$\bullet$},itemsep=0pt,topsep=0pt,leftmargin=2em]
    \item \textbf{MSE/MAE:} Mean squared/absolute error (MSE/MAE) for numerical variables.

    \item \textbf{Accuracy/ROC-AUC:} Accuracy and ROC-AUC for categorical variables.
\end{enumerate}
For datasets containing both numerical and categorical variables, we report metrics separately for each variable. For synthetic experiments with multiple variables, we follow \citet{sanchez-martin_vaca_2022} and report MSE aggregated across all variables and samples.

\noindent\textbf{Synthetic Datasets.} 
For the synthetic experiments, we use the Triangle, M-shape, and German Loan benchmarks from \citet{sanchez-martin_vaca_2022}, together with randomly generated SCMs for the graph-size experiments. Triangle and M-shape have 3 and 5 nodes respectively and are evaluated with linear (LIN), nonlinear additive (NLIN), and nonlinear non-additive (NADD) mechanisms, while German Loan is a 7-node semi-synthetic seven-variable SCM inspired by the German Credit domain. Detailed definitions can be found in \ref{app:synth_exp_gen}.

Further, we conduct experiments on causal graphs of increasing size where each sample comes from its own SCM. We follow \citet{soleymani2020causal}.
We generate the SCMs by randomly permuting the nodes into a topological order and adding each admissible edge $X_i \rightarrow X_j$ with probability $\frac{1}{2}$. Each variable $X$ is generated from its direct causes using one of seven mechanisms $f$ per sample, followed by posterior additive noise (App.~\ref{datageneration_sizeexp}). We consider causal graphs of sizes $\mathcal{S}={10,15,20,25,30}$.

\noindent\textbf{Real-world Datasets.}
For the real-world experiments, we use the CheXpert dataset \citep{irvin2019chexpertlargechestradiograph}, consisting of chest X-rays (App. \ref{app:chexpert_generation}). We augment each image with four explicit variables: age, sex, disease, and AP/PA, where AP/PA denotes the radiographic viewing direction. Following \citet{de_sousa_ribeiro_counterfactual_2025}, we supply the baselines with a causal structure in which age affects disease, while age, sex, disease, and AP/PA determine the image. Since ground-truth counterfactuals are unavailable \citep{ma2026aialignmentmedicalimaging}, we evaluate whether generated images exhibit the intended counterfactual attributes using classifiers. Details on the classifiers and image-generating CFGNP variant are provided in App.~\ref{app:chexpert_generation}. We assess whether images resemble real-world images qualitatively.



\subsection{Synthetic Experiments} \label{sec:explicit-experiments}
On synthetic experiments, we compare against VACA \citep{sanchez-martin_vaca_2022}, BGM \citep{nasr-esfahany_counterfactual_2023}, and CFM \citep{le_learning_2025}, which have demonstrated strong performance in prior work. \change{Additionally, we compare to VCI \citep{wu_counterfactual_2025} because it does not explicitly require a causal graph.} VACA and BGM require the full causal graph, whereas CFM only requires its topological ordering, allowing comparison across different levels of causal knowledge. For the graph-size experiments, we omit CFM as it was not computationally viable on larger causal graphs. \change{For each of the baselines we add an additional variant where we extend the train loss by a counterfactual supervision loss (App. \ref{app:baseline_cf_supervision}). This allows us to draw a fair comparison to CFGNP.}

\noindent\textbf{Counterfactual Prediction.}
Table \ref{tab:explicit_results} shows the results for the synthetic datasets (uncertainty evaluation in App. \ref{app:pred_uncertainty_synth}). CFGNP improves across all variants of the datasets and in many by a substantial amount. For instance, for German loan, CFGNP has an MSE of \change{$6.57 \times 10^{-5}$ compared to $2.11 \times 10^{-4}$ of BGM-CF}. However, on the \change{NADD variant of triangle, the MSE of $4.08 \times 10^{-5}$ is much closer to BGM's with $8.63 \times 10^{-5}$.
}
\begin{table*}[h]
\vspace{-5pt}
\centering
\small
\setlength{\tabcolsep}{2pt}
\newcommand{\res}[3]{\begin{tabular}{@{}c@{}}$#1 \times 10^{#3}$\\[-2pt]{\scriptsize$\pm #2 \times 10^{#3}$}\end{tabular}}
\newcommand{\best}[3]{{\boldmath\res{#1}{#2}{#3}}}
\begin{adjustbox}{max width=\linewidth}
\begin{tabular}{@{}lccccccc@{}}
\toprule
\multirow{2}{*}{Method} & \multicolumn{3}{c}{Triangle} & \multicolumn{3}{c}{M-shape} & \multirow{2}{*}{German Loan} \\
\cmidrule(lr){2-4} \cmidrule(lr){5-7}
 & LIN & NLIN & NADD & LIN & NLIN & NADD & \\
\midrule
BGM       & \res{4.52}{1.42}{-4} & \res{9.72}{7.23}{-4} & \res{2.11}{0.73}{-4} & \res{2.35}{0.93}{-5} & \res{2.52}{1.36}{-5} & \res{2.18}{0.09}{-3} & \res{1.91}{0.25}{-3} \\
VACA      & \res{7.02}{0.79}{-3} & \res{4.12}{0.45}{-3} & \res{7.86}{0.13}{-3} & \res{3.97}{0.09}{-3} & \res{3.93}{0.06}{-3} & \res{3.78}{0.04}{-3} & \res{11.41}{1.34}{-2} \\
VCI       & \res{2.05}{0.16}{-2} & \res{2.05}{0.86}{-2} & \res{1.23}{0.22}{-2} & \res{1.15}{0.61}{-2} & \res{1.60}{1.29}{-2} & \res{1.42}{2.63}{-1} & \res{1.95}{0.48}{-1} \\
CFM       & \res{1.72}{0.86}{-4} & \res{5.54}{2.54}{-4} & \res{5.26}{3.87}{-4} & \res{3.71}{0.50}{-5} & \res{5.16}{1.41}{-5} & \res{2.07}{0.05}{-3} & \res{1.52}{0.15}{-3} \\
BGM-CF    & \res{1.56}{0.39}{-5} & \res{1.02}{0.28}{-5} & \res{8.63}{0.66}{-5} & \res{1.43}{0.69}{-5} & \res{1.55}{0.76}{-5} & \res{2.06}{0.15}{-3} & \res{2.11}{0.28}{-4} \\
VACA-CF   & \res{7.16}{2.16}{-5} & \res{1.47}{0.39}{-4} & \res{4.28}{0.89}{-4} & \res{6.03}{2.9}{-5} & \res{1.07}{0.28}{-4} & \res{1.37}{0.28}{-3} & \res{2.63}{0.43}{-4} \\
VCI-CF    & \res{1.62}{0.67}{-4} & \res{5.00}{0.49}{-4} & \res{4.46}{0.99}{-4} & \res{2.13}{1.24}{-3} & \res{9.40}{1.98}{-5} & \res{5.76}{2.76}{-3} & \res{2.34}{0.36}{-4} \\
CFM-CF    & \res{4.26}{1.12}{-5} & \res{2.65}{0.35}{-5} & \res{2.62}{0.79}{-4} & \res{3.89}{1.47}{-5} & \res{5.78}{1.54}{-5} & \res{1.94}{0.04}{-3} & \res{1.06}{0.10}{-3} \\
\midrule
CFGNP (Ours) & \best{2.98}{0.52}{-6} & \best{1.44}{0.26}{-6} & \best{4.08}{0.42}{-5} & \best{1.67}{0.66}{-6} & \best{4.55}{1.08}{-6} & \best{2.04}{0.32}{-4} & \best{6.57}{3.01}{-5} \\
\bottomrule
\end{tabular}
\end{adjustbox}
\vspace{-5pt}
\caption{Mean MSE $\pm$ standard deviation achieved by BGM, VACA, CFM, VCI,
    and CFGNP (ours) across the Triangle, M-shape, and German Loan datasets. Methods with suffix -CF are trained on counterfactual pairs.
    Lower values indicate better performance. Best scores are marked in bold.}
\label{tab:explicit_results}
\vspace{-10pt}
\end{table*}

\change{\noindent\textbf{Zero-shot Transfer.}
In the experiments above, CFGNP is trained on counterfactual pairs of the benchmark SCM. To test whether it requires this, we train CFGNP only on randomly sampled SCMs and apply it zero-shot to the German loan, M-shape, and Triangle test sets, where each context--query pair represents a separate task (App. \ref{app:prior_scms}). For a fair comparison, the baselines are trained from scratch on the observational context of each task, using NOTEARS to obtain the causal graph, so that no method receives counterfactual supervision from the target SCM. CFGNP performs best on all datasets.}

\change{\noindent\textbf{Importance of Observations.}
To separate the effect of counterfactual supervision from that of the observational context and knowledge of the true causal graph, we train three ablations of CFGNP on the same counterfactual pairs: 
(i) without observational context, (ii) without observational context but supplied with the true causal graph, (iii) with observational context and the true causal graph. We supply the true causal graph via edge labels. An ablation on that is in App. \ref{app:true_cg_cfgnp_ablation}. As expected, (i) performs worst. (ii) performs best on two of the three triangle datasets. (iii) performs best on 2 of the Mshape experiments and German loan. Further, (iii) is within one standard deviation of the best score in all datasets. Normal CFGNP performs best on two of the datasets.


\begin{table}[t]
\centering
\resizebox{0.9\linewidth}{!}{
\begin{tabular}{llcccc}
\toprule
Dataset & Variant & w/o context & w/o context, w/ true graph & w/ context, true graph & CFGNP (Ours) \\
\midrule
\multirow{3}{*}{Triangle}
 & LIN  & $(7.91 \pm 0.82) \times 10^{-4}$ & $\mathbf{(2.74 \pm 0.81) \times 10^{-6}}$ & $(2.94 \pm 0.87) \times 10^{-6}$ & $(2.98 \pm 0.52) \times 10^{-6}$ \\
 & NLIN & $(3.09 \pm 0.06) \times 10^{-4}$ & $(4.22 \pm 1.73) \times 10^{-6}$ & $(1.49 \pm 0.25) \times 10^{-6}$ & $\mathbf{(1.44 \pm 0.26) \times 10^{-6}}$ \\
 & NADD & $(1.70 \pm 0.26) \times 10^{-4}$ & $\mathbf{(2.99 \pm 0.87) \times 10^{-5}}$ & $(3.82 \pm 0.44) \times 10^{-5}$ & $(4.08 \pm 0.42) \times 10^{-5}$ \\
\midrule
\multirow{3}{*}{M-shape}
 & LIN  & $(2.00 \pm 0.04) \times 10^{-3}$ & $(7.26 \pm 0.45) \times 10^{-4}$ & $\mathbf{(1.56 \pm 0.39) \times 10^{-6}}$ & $(1.67 \pm 0.66) \times 10^{-6}$ \\
 & NLIN & $(1.52 \pm 0.19) \times 10^{-3}$ & $(5.48 \pm 0.76) \times 10^{-4}$ & $(4.60 \pm 0.71) \times 10^{-6}$ & $\mathbf{(4.55 \pm 1.08) \times 10^{-6}}$ \\
 & NADD & $(1.35 \pm 0.07) \times 10^{-3}$ & $(5.26 \pm 2.56) \times 10^{-4}$ & $\mathbf{(1.91 \pm 0.27) \times 10^{-4}}$ & $(2.04 \pm 0.32) \times 10^{-4}$ \\
\midrule
German Loan & -- & $(1.74 \pm 0.11) \times 10^{-3}$ & $(6.92 \pm 1.08) \times 10^{-5}$ & $\mathbf{(5.44 \pm 0.47) \times 10^{-5}}$ & $(6.57 \pm 3.01) \times 10^{-5}$ \\
\bottomrule
\end{tabular}
\vspace{-20pt}
}
\caption{Ablation of the observational context and graph knowledge for CFGNP (mean MSE $\pm$ standard deviation over five seeds). All variants are trained on the same counterfactual pairs. Best scores are marked in bold.}
\label{tab:ablation}
\vspace{-10pt}
\end{table}
}

\noindent\textbf{Graph-Size Experiments.}
For experiments with increasing graph size we use two variants of CFGNP. In one variant, CFGNP is retrained for each $s\in \mathcal{S}$, while in the other, CFGNP is trained only on graphs with $s=10$ nodes and applied to the other graph sizes without training. We refer to this model as $10$-CFGNP.
Additionally, as the benchmarks require a fixed causal graph, we provide the fullest DAG, connecting each node to all topological successors. This highlights an advantage of CFGNP: when causal graphs vary across the dataset, baselines must still assume a fixed graph.

Figure~\ref{fig:size_experiment_plots} shows the MSE for increasing graph sizes (note the non-linear y-axis). The MSE increases with graph size for all methods on most mechanisms. \change{In addition to the standard baselines, we fit each baseline to the observational context of every sample individually, with the causal graph estimated by NOTEARS (denoted + NOTEARS). Both retrained CFGNP and $10$-CFGNP outperform the standard baselines on all mechanisms. The strongest baseline is BGM + NOTEARS: retrained CFGNP achieves a lower MSE in 25 of the 35 dataset--size combinations, and $10$-CFGNP in 10. $10$-CFGNP performs similarly to retrained CFGNP, except on sum-sine.} Results for the remaining mechanisms are in App.~\ref{addition_size_results}.
\begin{figure}[h]
    \vspace{-10pt} 
    \centering
    \begin{subfigure}{0.19\textwidth}
        \centering
        \includegraphics[width=\linewidth]
        {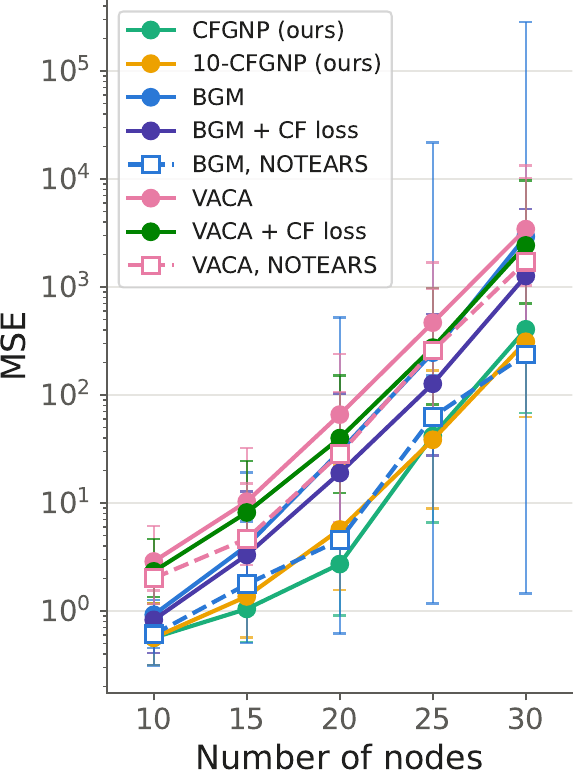}
        \label{fig:linear_quantile}
    \end{subfigure}
    \begin{subfigure}{0.19\textwidth}
        \centering
        \includegraphics[width=\linewidth]
        {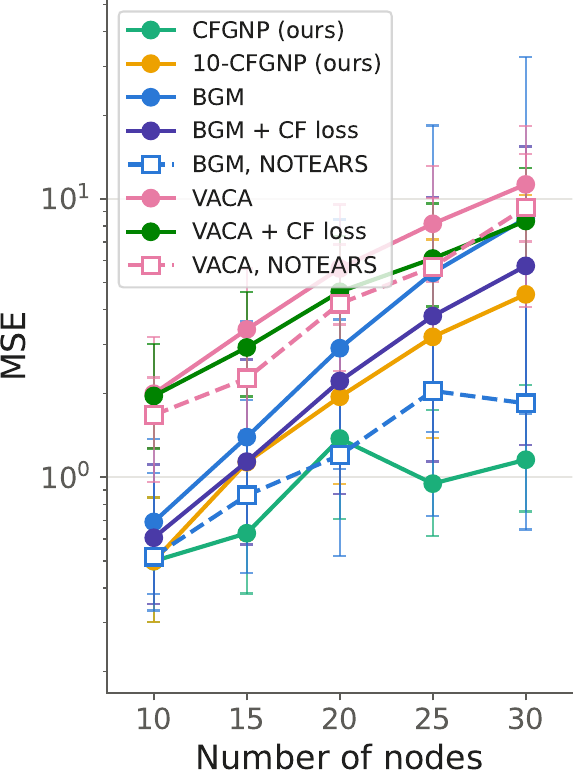}
        \label{fig:sum_sine_quantile}
    \end{subfigure}
    \begin{subfigure}{0.19\textwidth}
        \centering
        \includegraphics[width=\linewidth]
        {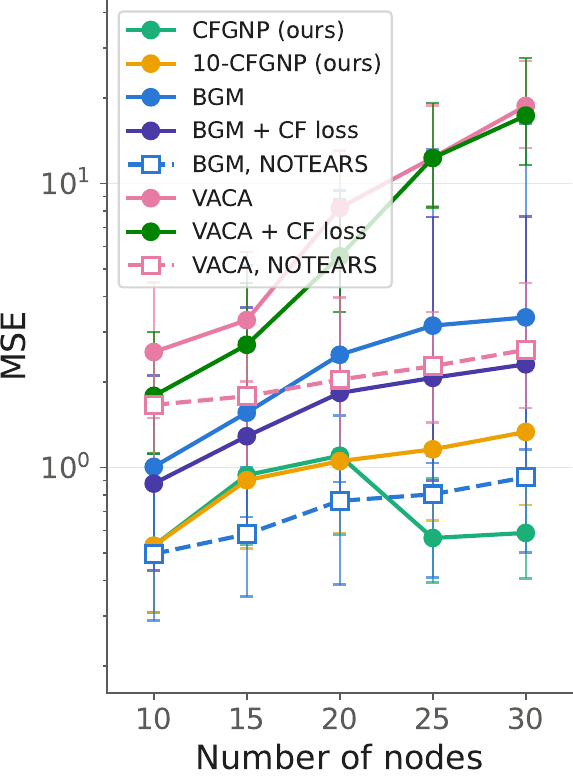}
        \label{fig:logsumexp_quantile}
    \end{subfigure}
    \begin{subfigure}{0.19\textwidth}
        \centering
    \includegraphics[width=\linewidth]
    {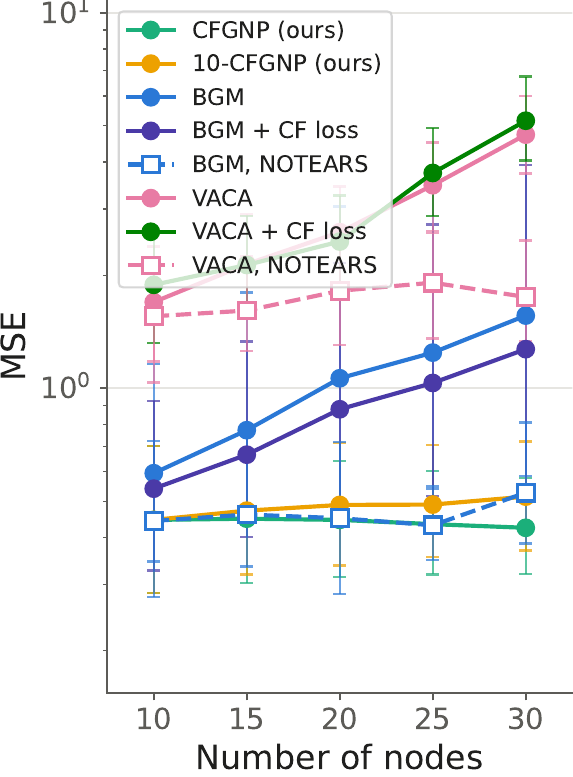}
        \label{fig:sqrt_sum_quantile}
    \end{subfigure}
    \begin{subfigure}{0.19\textwidth}
        \centering
    \includegraphics[width=\linewidth]
    {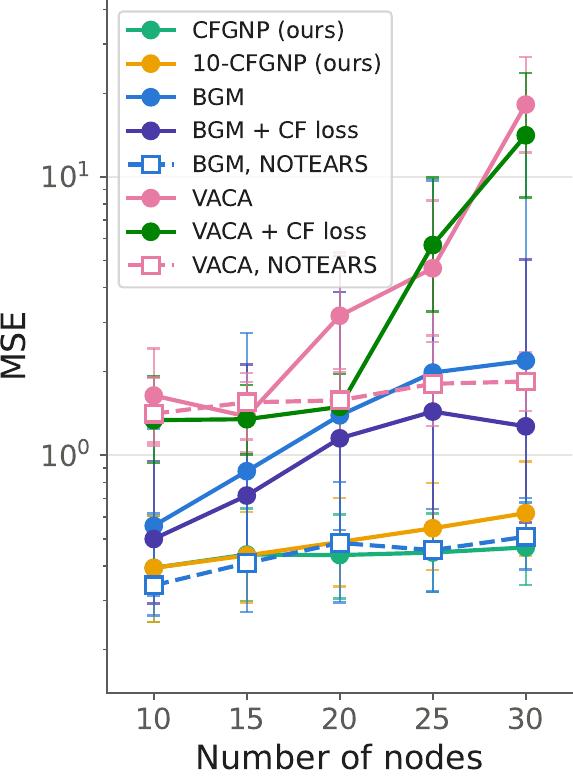}
        \label{fig:sum_sqrt_quantile}
    \end{subfigure}
    
\vspace{-20pt} 
    \caption{Performance (MSE) of models with regard to size of the causal graph. Models are compared on causal graphs with Linear (1.), Sum-sine (2.), Log-sum-exp (3.), Sqrt-sum (4.), and Sum-sqrt (5.) structural equations. MSE mean over distinct causal graphs. Vertical lines show $25^{th}$ and $75^{th}$ quantiles.}
    \label{fig:size_experiment_plots}
\vspace{-10pt}
\end{figure}

\subsection{CheXpert Experiments}\label{sec:image-experiments}
We compare against VCI \citep{wu_counterfactual_2025} and DeepBC \citep{kladny_deep_2024} which both have previously demonstrated strong performance. Again, VCI assumes treatment–covariate–outcome structure while DeepBC needs full graph knowledge.

\noindent\textbf{Counterfactual Prediction.}
Table \ref{tab:chexpert_results} shows that CFGNP performs particularly well on age, reducing MAE from $10.88$ for DeepBC to $2.71$, and consistently achieves high ROC-AUC on categorical variables. For disease, for example, accuracies are around $0.87$, while CFGNP achieves a ROC-AUC of $0.77$, compared with $0.75$ for DeepBC and $0.45$ for VCI. Generated counterfactual images are provided in Appendix~\ref{addition_size_results}.

\begin{table}[h]
    \centering
    \resizebox{0.7\linewidth}{!}{%
    \begin{tabular}{llccc}
        \toprule
        \textbf{Feature} & \textbf{Metric}
        & \textbf{CFGNP (Ours)}
        & \textbf{DeepBC}
        & \textbf{VCI} \\
        \midrule

        \multirow{2}{*}{Sex}
            & Accuracy $\uparrow$ & $\mathbf{0.8452 \pm 0.0608}$ & $0.8222 \pm 0.1220$ & $0.5173 \pm 0.0296$ \\
            & ROC-AUC $\uparrow$  & $0.9155 \pm 0.0540$ & $\mathbf{0.9204 \pm 0.0752}$ & $0.5238 \pm 0.0122$ \\

        \addlinespace

        Age
            & MAE $\downarrow$     & $\mathbf{2.7136 \pm 0.7326}$ & $10.8750 \pm 0.8732$ & $66.8172 \pm 62.1769$ \\

        \addlinespace

        \multirow{2}{*}{AP/PA}
            & Accuracy $\uparrow$ & $\mathbf{0.9664 \pm 0.0140}$ & $0.9572 \pm 0.0552$ & $0.8597 \pm 0.0000$ \\
            & ROC-AUC $\uparrow$  & $\mathbf{0.9787 \pm 0.0136}$ & $0.9658 \pm 0.0584$ & $0.5741 \pm 0.0342$ \\

        \addlinespace

        \multirow{2}{*}{Disease}
            & Accuracy $\uparrow$ & $\mathbf{0.8830 \pm 0.0368}$ & $0.8720 \pm 0.0240$ & $0.8548 \pm 0.0667$ \\
            & ROC-AUC $\uparrow$  & $\mathbf{0.7710 \pm 0.0572}$ & $0.7517 \pm 0.0327$ & $0.4516 \pm 0.0202$ \\

        \bottomrule
    \end{tabular}
    }
    \vspace{-4pt}
    \caption{Performance of CFGNP, VCI, and DeepBC on the CheXpert dataset.
    MAE is reported for the continuous feature,
    while accuracy and ROC-AUC are reported for categorical features.
    Arrows indicate whether higher or lower values are preferred. Best values marked in bold.}
    \label{tab:chexpert_results}
    \vspace{-6pt}
\end{table}

\change{
\noindent\textbf{Predictive Uncertainty.}
We test whether CFGNP's predictive uncertainty reflects how much the observational context reveals about the mechanism, using a two-node SCM whose counterfactual distribution can be computed analytically (App.~\ref{app:uncertainty}). We compare the predictive distribution to the Bayes-optimal distribution given the context and to the true distribution (Fig.~\ref{fig:uncertainty}). The predicted standard deviation closely matches the Bayes-optimal one, so it reflects the information contained in the context. The KL divergence to the true distribution also decreases as the context grows, partly due to a lower bias.
\begin{figure}[t]
  \centering
  \includegraphics[width=0.7\linewidth]{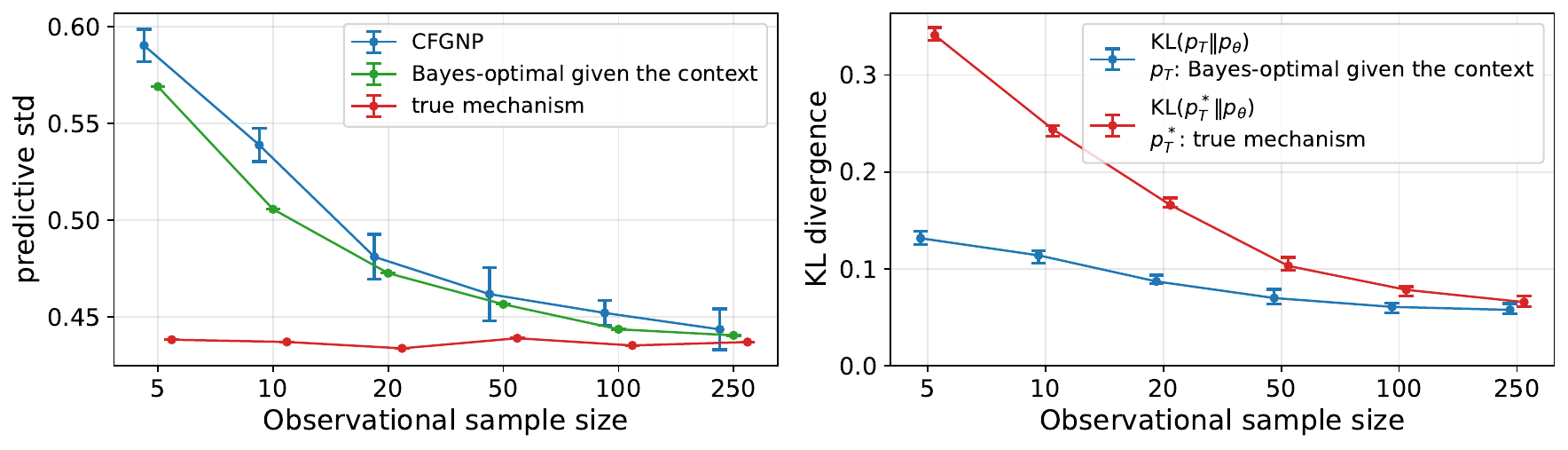}
  \vspace{-10pt}
  \caption{Two-node backtracking task. \textbf{Left:} predictive standard deviation of CFGNP, the Bayes-optimal
  prediction ($p_T$), and the true mechanism ($p^*_T$), as mean $\pm$ std.
  \textbf{Right:} KL divergence of CFGNP's predictive distribution to $p_T$ and to $p^*_T$, as median and
  10--90\% range.}
  \label{fig:uncertainty}
  \vspace{-10pt}
\end{figure}
}

\change{
\noindent\textbf{When does CFGNP help?}
CFGNP helps most when the causal graph varies across instances (Fig.~\ref{fig:size_experiment_plots}) or when existing methods are far from exact on a fixed SCM; knowing the graph does not close this gap, as the baselines in Table~\ref{tab:explicit_results} receive the true graph or its ordering. For a known graph with simple mechanisms (M-shape LIN), specialised methods are already near-exact and cheaper: they train faster (Table~\ref{tab:runtime_comparison}) and need no counterfactual training pairs.
}

%% file: conclusion.tex
\section{Conclusion}
We propose CFGNP, an end-to-end architecture designed to learn counterfactual distributions directly from observational data without knowledge of the causal graph. CFGNP applies GNNs to enforce structural constraints on the dual graph implied by the notion of backtracking counterfactuals. By leveraging observational context samples alongside counterfactual queries, CFGNP directly approximates the Bayesian model-averaged counterfactual distribution.
Evaluations on image and synthetic counterfactual queries have shown strong performance compared to baselines. 
Further experiments on synthetic datasets show that CFGNP retains strong performance on larger causal graphs compared with retraining from scratch, \change{while performing well compare to baselines.}
In terms of limitations, the preceding capabilities come with a computational trade-off. The required observational context increases both training and inference costs. \change{Further, its out-of-distribution performance could be improved.}
Moreover, if the counterfactual distribution is non-identifiable, CFGNP cannot recover the true distribution. However, this also applies to the baseline methods.
In the future, one could incorporate priors about the causal graph. While the entire causal graph may not be known, an expert may have knowledge about the existence and absence of certain causal relationships. Taking this even further, one could provide simple distributions over edges, for instance obtained by amortized causal structure learning methods \citep{lorch2022amortizedinferencecausalstructure}.

%% file: appendix/appendix.tex
\input{appendix/proof_appendix}

\input{appendix/extended_model_architecture}

\input{appendix/data_generation}

\input{appendix/baseline_description}

\input{appendix/additional_results}

%% file: appendix/proof_appendix.tex
\input{appendix/dual_model_appendix}

\input{appendix/time_complexity_appendix}

\input{appendix/permutation_appendix}

\input{appendix/graph_size_appendix}

%% file: appendix/dual_model_appendix.tex
\subsection{Equivalence of a Direct MPNN to a Non-directional MPNN on the Dual Graph}
This section provides an expressivity result justifying the removal of explicit exogenous nodes from the dual-graph architecture.

\begin{proposition}[Expressive sufficiency after eliminating exogenous nodes]\label{proposition_bidirectional}
Let $G_{\mathrm{exo}}$ denote the dual graph containing explicit exogenous nodes, and let
$G_{\mathrm{dir}}$ denote the reduced graph obtained by removing the exogenous nodes and directly
connecting each factual--counterfactual pair $(x_i, x_i^\ast)$.

Assume that the endogenous subgraph of $G_{\mathrm{dir}}$ is fully connected. Then, for every
function of the endogenous nodes that can be approximated arbitrarily well by a bidirectional MPNN
on $G_{\mathrm{exo}}$, there exists a direct MPNN on $G_{\mathrm{dir}}$, with sufficiently large hidden
dimension, that approximates the same function arbitrarily well.
\end{proposition}

\begin{proof}

A bidirectional MPNN layer on $G_{\mathrm{exo}}$ has the form

\[
h_v^{(\ell)}
=
\Phi_\ell\!\left(
h_v^{(\ell-1)},
\sum_{(w,v)\in E}
M_\ell^{\rightarrow}
(h_v^{(\ell-1)},h_w^{(\ell-1)}),
\sum_{(v,w)\in E}
M_\ell^{\leftarrow}
(h_v^{(\ell-1)},h_w^{(\ell-1)})
\right)
\]
where $M_\ell^{\rightarrow}$ and $ M_\ell^{\leftarrow}$ are neural networks handling incoming and outcoming edges respectively.

Since the values of the exogenous variables are not observed, the initial exogenous embeddings are fixed constants,

\[
h_{u_i}^{(0)}=c_U,
\qquad
h_{u_i^*}^{(0)}=c_{U^*}.
\]

Consequently, the hidden state of every exogenous node is completely determined by the hidden states of neighbouring endogenous and exogenous nodes. 
$G_{dir}$ does not have exogenous nodes, so we model them by doubling the hidden dimension of the endogenous nodes.
Instead of explicitly representing $u_i$ and $u_i^*$, the direct MPNN stores their hidden states inside the corresponding endogenous representations,

\[
\widetilde h_{x_i}^{(\ell)}
=
\bigl(h_{x_i}^{(\ell)},h_{u_i}^{(\ell)}\bigr),
\qquad
\widetilde h_{x_i^*}^{(\ell)}
=
\bigl(h_{x_i^*}^{(\ell)},h_{u_i^*}^{(\ell)}\bigr).
\]

Assume inductively that these augmented representations reproduce the hidden states of the bidirectional MPNN after layer $\ell-1$.

The bidirectional update of an exogenous node is a continuous function

\[
h_{u_i}^{(\ell)}
=
F_{\ell,i}
\left(
h_{u_i}^{(\ell-1)},
h_{u_i^*}^{(\ell-1)},
h_{x_i}^{(\ell-1)}
\right),
\]

and analogously for $u_i^*$. Since the endogenous graph is fully connected, node $x_i$ receives the representation of $x_i^*$ through an ordinary directed message. Hence, all arguments of $F_{\ell,i}$ are available locally in the augmented representations. By the universal approximation theorem \citep{hornik1989multilayer}, an MLP can approximate $F_{\ell,i}$ arbitrarily well, allowing the direct MPNN to update the virtual exogenous states.

Next, consider the update of an endogenous node,

\[
h_{x_i}^{(\ell)}
=
G_{\ell,i}
\left(
h_{x_i}^{(\ell-1)},
\{h_y^{(\ell-1)}:y\in N_{\mathrm{end}}(x_i)\},
h_{u_i}^{(\ell-1)}
\right).
\]

The virtual state $h_{u_i}^{(\ell-1)}$ is stored locally in
$\widetilde h_{x_i}^{(\ell-1)}$. Because the endogenous graph is fully connected, node $x_i$ receives the hidden state of every endogenous node. Hence, all arguments of $G_{\ell,i}$ are available to the direct MPNN. Again, the update function can be approximated arbitrarily well by an MLP.

After each layer,

\[
\widetilde h_{x_i}^{(\ell)}
=
\bigl(
h_{x_i}^{(\ell)},
h_{u_i}^{(\ell)}
\bigr),
\]

and similarly for $x_i^*$. By induction, the hidden representations of all endogenous nodes produced by the bidirectional MPNN can be reproduced arbitrarily closely by the direct MPNN.

Since the readout is performed only on endogenous nodes, reproducing the exogenous hidden states explicitly is unnecessary.

\end{proof}

%% file: appendix/time_complexity_appendix.tex
\subsection{CFGNP Complexity Analysis}\label{app:complexity_proof}

We calculate the complexity in each part of the module. We first note that the GAT layer has computational complexity $\Theta(|\mathcal{V}|d d' + |\mathcal{E}|d')$ for a graph with $|\mathcal{V}|$ nodes and $|\mathcal{E}|$ edges and the hidden dimensions go from $d$ to $d'$ \citep{brody2022attentivegraphattentionnetworks}. Let $|A|$ be the number of edges in the graph with adjacency matrix $A$.

\subsubsection*{Changing Data Dimension}

Applying the MLP to each entry in
$
O \in \mathbb{R}^{(N_{\mathrm{obs}} \cdot n)\times d_{\mathrm{data}}}
$
and
$
z,x^*\in
\mathbb{R}^{(N_{\mathrm{count}} \cdot n)\times d_{\mathrm{data}}}
$
takes
\[
(N_{\mathrm{obs}}n + 2N_{\mathrm{count}} n)
\, d_{\mathrm{data}} d_{\mathrm{hidden}}
\in
\Theta\!\left(
n\,d_{\mathrm{data}}d_{\mathrm{hidden}}
(N_{\mathrm{count}}+N_{\mathrm{obs}})
\right)
\]
operations. Since we use $N_{count} = 1$ throughout this work we get complexity
\[
\Theta\!\left(
n\,d_{\mathrm{data}}d_{\mathrm{hidden}}
N_{\mathrm{obs}}
\right)
\]

\subsubsection*{Propagation of Observational Information}
Each layer of the GAT uses adjacency matrix 
\[
G_c=
\begin{pmatrix}
M_n & \mathbf{0}_n & \cdots & \mathbf{0}_n & I_n \\
\mathbf{0}_n & M_n & \ddots & \vdots & \vdots \\
\vdots & \ddots & \ddots & \mathbf{0}_n & \vdots \\
\mathbf{0}_n & \cdots & \mathbf{0}_n & M_n & I_n \\
\mathbf{0}_n & \cdots & \cdots & \mathbf{0}_n & \mathbf{0}_n
\end{pmatrix}.
\]

With $|M_n| = n (n-1)$ and $|I_n| = n$ we get $|G_c| = N_{obs}n (n-1) + N_{obs} n$.
Further, we have $|\mathcal{V}| = (N_{obs} + 1)n$.
For one layer we then get computational complexity 
\[
\Theta((N_{obs}+1) n d_{hidden}^2 + (N_{obs}n (n-1) + N_{obs} n) d_{hidden}).
\]
For all \(K\) layers in the block, we have computational complexity
\[
\Theta(K(N_{obs}+1) n d_{hidden}^2 + K(N_{obs}n^2) d_{hidden}) =\Theta(KN_{obs} n d_{hidden} (d_{hidden} + n)).
\]

\subsubsection*{Mixing Factual Information}

Each layer in the module for mixing factual information uses adjacency matrix
\[
G_C' =
\begin{pmatrix}
    \mathbf{0}_n & I_n \\
    \mathbf{0}_n & I_n
\end{pmatrix}.
\]
With $|G_C'|= 2n$ and $|\mathcal{V}|=2n$ we get computational complexity $\Theta(n d_{hidden}^2 + nd_{hidden})$.
All three layers for this module thus have complexity $\Theta(n d_{hidden}^2 + nd_{hidden})$.

\subsubsection*{Estimating Counterfactual Graph}

Each layer in this block uses adjacency matrix
\[
G_C'' = \begin{pmatrix}
M_n & I_n \\
I_n & M_n
\end{pmatrix}.
\]

We then get $|G_C''| = 2 n (n-1) + 2n = 2n^2$ and $|\mathcal{V}|=2n$. This gives us computational complexity of $\Theta(n d_{hidden}^2 + n^2 d_{hidden})$ in each layer.

All $\min\{\lfloor\frac{n}{2}\rfloor, 5\}$ layers together give us computational complexity $\Theta(n d_{hidden}^2 + n^2 d_{hidden})$ for the entire block.

The MLP layer needs
$
\Theta(n\,d_{\mathrm{hidden}}N_C)
$
operations.

Overall, the computational complexity of CFGNP is
\[
\begin{aligned}
\Theta\!\Big(
& n\,d_{\mathrm{data}}d_{\mathrm{hidden}}N_{\mathrm{obs}}
+ KN_{\mathrm{obs}} n d_{\mathrm{hidden}}(d_{\mathrm{hidden}} + n) \\
& + n d_{\mathrm{hidden}}^2 + n^2 d_{\mathrm{hidden}}
+ n d_{\mathrm{hidden}} N_C
\Big).
\end{aligned}
\]

%% file: appendix/permutation_appendix.tex
\subsection{Permutation Properties of CFGNP}\label{ssec:permutation_proof}
\begin{proposition}
Let
\[
F(z,T^*,t^*,Y^*,O)
\]
denote the output of the CFGNP model restricted to the target nodes \(Y^*\).
Then the CFGNP is

\begin{enumerate}
    \item permutation equivariant with respect to permutations of the nodes in
    \(Y^*\); and
    \item permutation invariant with respect to permutations of the nodes in
    \((Y^*)^c\).
\end{enumerate}
\end{proposition}

\begin{proof}
In the following, let $M_n$ be the $n \times n$ matrix that has $0$ on the diagonal and everywhere else $1$. Let $I_n$ be the $n \times n$ identity matrix. Finally, let $\mathbf{0}_n$ be the $n \times n$ matrix with only $0$.

Let \(P\in\{0,1\}^{n\times n}\) be the permutation matrix associated with a
permutation \(\pi\) of the \(n\) variables. We use the convention
$
Pe_i=e_{\pi(i)},
$
where \(e_i\) denotes the \(i\)-th standard basis vector. Consequently, for
every node-indexed matrix \(H\in\mathbb{R}^{n\times d}\),
\[
(PH)_j=H_{\pi^{-1}(j)}.
\]

A permutation represents only a relabelling of the variables and must
therefore be applied consistently to every node-indexed input. We define
\[
z_P:=Pz,
\qquad
x_P^*:=Px^*,
\qquad
O_{k,P}:=PO_k
\]
for every observational sample \(O_k\). We further define
\[
T^*_\pi:=\pi(T^*),
\qquad
Y_\pi^*:=\pi(Y^*).
\]

Let
$
\tau_{T^*}\in\{0,1\}^n
$
be the indicator vector of the intervention set \(T^*\), i.e.,
\[
(\tau_{T^*})_i
=
\begin{cases}
1, & i\in T^*,\\
0, & i\notin T^*.
\end{cases}
\]
The indicator vector of the permuted intervention set \(T^*_\pi\) is
$
\tau_{T^*,P}:=P\tau_{T^*}.
$
Indeed, for every \(j\in\{1,\ldots,n\}\),
\[
(\tau_{T^*,P})_j
=
(P\tau_{T^*})_j
=
(\tau_{T^*})_{\pi^{-1}(j)},
\]
which is equal to \(1\) precisely when
\(\pi^{-1}(j)\in T^*\), or equivalently when \(j\in\pi(T^*)\).

Recall that the counterfactual input \(x^*\) contains the intervention values
at the indices in \(T^*\) and is zero elsewhere:
\[
x_i^*
=
\begin{cases}
t_i^*, & i\in T^*,\\
0, & i\notin T^*.
\end{cases}
\]
For every \(j\),
\[
(x_P^*)_j
=
(Px^*)_j
=
x^*_{\pi^{-1}(j)}.
\]

\[
(x_P^*)_j
=
\begin{cases}
t_{\pi^{-1}(j)}^*, & \pi^{-1}(j)\in T^*,\\
0, & \pi^{-1}(j)\notin T^*.
\end{cases}
\quad\Leftrightarrow\quad
(x_P^*)_j
=
\begin{cases}
(Pt^*)_j, & j\in T^*_\pi,\\
0, & j\notin T^*_\pi.
\end{cases}
\]
Thus, the intervention values and the intervention indices are permuted
consistently.

We now show the characteristics for each of the model components.

\paragraph{Equivariance of nodewise MLPs.} We first consider the nodewise MLPs applied at the beginning of the model.
For every observational sample \(O_k\), the same MLP \(h_O\) is applied
independently to every node. Therefore,
$
h_O(O_{k,P})
=
Ph_O(O_k).
$

For the factual observation \(z\), define the nodewise transformation
\[
h_z(z,\tau_{T^*})_i
=
(\tau_{T^*})_i h_{z,\mathrm{int}}(z_i)
+
\bigl(1-(\tau_{T^*})_i\bigr)h_{z,\mathrm{common}}(z_i).
\]
We now verify permutation equivariance coordinatewise. For an arbitrary
\(j\in\{1,\ldots,n\}\):
\[
\begin{aligned}
h_z(z_P,\tau_{T^*,P})_j
&=
(\tau_{T^*,P})_j
h_{z,\mathrm{int}}\bigl((z_P)_j\bigr)
+
\bigl(1-(\tau_{T^*,P})_j\bigr)
h_{z,\mathrm{common}}\bigl((z_P)_j\bigr)\\
&=
(\tau_{T^*})_{\pi^{-1}(j)}
h_{z,\mathrm{int}}\bigl(z_{\pi^{-1}(j)}\bigr)\\
&\quad+
\bigl(1-(\tau_{T^*})_{\pi^{-1}(j)}\bigr)
h_{z,\mathrm{common}}\bigl(z_{\pi^{-1}(j)}\bigr)\\
&=
h_z(z,\tau_{T^*})_{\pi^{-1}(j)}\\
&=
\bigl(Ph_z(z,\tau_{T^*})\bigr)_j.
\end{aligned}
\]
Since this equality holds for every \(j\),
\[
h_z(Pz,P\tau_{T^*})
=
Ph_z(z,\tau_{T^*}).
\]

The same argument applies to the counterfactual input.
\[
h_x(Px^*,P\tau_{T^*})
=
Ph_x(x^*,\tau_{T^*}).
\]

\paragraph{Equivariance of the GATs.}
We now consider the GATs. A GAT \(m_\theta\) is permutation equivariant in the
sense that, for a node-feature matrix \(H\), adjacency matrix \(G\), and
permutation matrix \(Q\),
\[
m_\theta(QH,QGQ^\top)
=
Qm_\theta(H,G).
\]

For the observational graph, the same permutation is applied to every
observational sample and to the aggregation nodes \(q\). The corresponding
block-diagonal permutation matrix is
\[
Q_O
=
\operatorname{diag}
\left(
\underbrace{P,\ldots,P}_{N_{\mathrm{obs}}+1\text{ times}}
\right)
=
I_{N_{\mathrm{obs}}+1}\otimes P.
\]
The adjacency matrix \(G_c\) consists only of blocks \(M_n\), \(I_n\), and
\(\mathbf{0}_n\). Since
\[
PM_nP^\top=M_n,
\qquad
PI_nP^\top=I_n,
\qquad
P\mathbf{0}_nP^\top=\mathbf{0}_n,
\]
we have
$
Q_OG_cQ_O^\top=G_c.
$
This suppresses the unchanged adjacency matrix from the notation,
\[
\begin{aligned}
m_\theta^{(1)}
\left(
Q_O
\begin{pmatrix}
O'\\
q
\end{pmatrix}
\right)
&=
Q_O
m_\theta^{(1)}
\left(
\begin{pmatrix}
O'\\
q
\end{pmatrix}
\right)=
Q_O
\begin{pmatrix}
\widetilde O\\
q'
\end{pmatrix}.
\end{aligned}
\]
It follows in particular that
$
q'_P = Pq'.
$

For the graph containing the factual representation \(z\) and the aggregated
observational representation \(q'\), define the permutation
\[
Q_{zq}
=
\begin{pmatrix}
P & \mathbf{0}_n\\
\mathbf{0}_n & P
\end{pmatrix}.
\]
Because
$
Q_{zq}G_C'Q_{zq}^\top=G_C',
$
permutation equivariance gives
\[
\begin{aligned}
m_\theta^{(2)}
\left(
Q_{zq}
\begin{pmatrix}
z\\
q'
\end{pmatrix}
\right)
&=
Q_{zq}
m_\theta^{(2)}
\left(
\begin{pmatrix}
z\\
q'
\end{pmatrix}
\right)\\
&=
Q_{zq}
\begin{pmatrix}
z'\\
q''
\end{pmatrix}\\
&=
\begin{pmatrix}
Pz'\\
Pq''
\end{pmatrix}.
\end{aligned}
\]

For the factual--counterfactual graph, define
\[
Q_{zx^*}
=
\begin{pmatrix}
P & \mathbf{0}_n\\
\mathbf{0}_n & P
\end{pmatrix}.
\]
With
\[
Q_{zx^*}G_C''Q_{zx^*}^\top=G_C'',
\]
we similarly obtain
\[
\begin{aligned}
m_\theta^{\mathrm{cf}}
\left(
Q_{zx^*}
\begin{pmatrix}
z'\\
x^*
\end{pmatrix}
\right)
&=
Q_{zx^*}
m_\theta^{\mathrm{cf}}
\left(
\begin{pmatrix}
z'\\
x^*
\end{pmatrix}
\right)\\
&=
Q_{zx^*}
\begin{pmatrix}
z''\\
{x^*}'
\end{pmatrix}\\
&=
\begin{pmatrix}
Pz''\\
P{x^*}'
\end{pmatrix}.
\end{aligned}
\]
In particular, the counterfactual representation satisfies
$
{x^*}'_P = P{x^*}'.
$

\paragraph{The final MLP.} The final MLP is applied independently to every node of \({x^*}'\), with the
same parameters at each node. Therefore,
\[
m_\theta^{(3)}(P{x^*}')
=
Pm_\theta^{(3)}({x^*}').
\]
Let
\[
R:=m_\theta^{(3)}({x^*}')
\]
denote the nodewise output containing the mixture parameters. 
The preceding arguments imply that the nodewise output of the model transforms
as
\[
R_P = PR.
\]

The final nodewise output
$
R\in\mathbb{R}^{n\times 3N_c}
$
contains one row for each node, and the CFGNP returns only the rows indexed by
the target set:
\[
F(z,T^*,t^*,Y^*,O)
=
R\big|_{Y^*}.
\]

Moreover, the target set is relabelled according to
$
Y_\pi^*=\pi(Y^*)$.
Since
\[
(PR)_{\pi(i)}
=
R_i
\qquad
\text{for every } i\in Y^*,
\]
restricting \(PR\) to the indices in \(Y_\pi^*\) yields precisely the same
nodewise outputs as restricting \(R\) to \(Y^*\), but indexed according to the
permuted target nodes. Hence,
\[
F(z_P,T^*_\pi,t_P^*,Y_\pi^*,O_P)
=
(PR)\big|_{Y_\pi^*}
=
P_{Y^*} \left(R\big|_{Y^*}\right),
\]
where \(P_{Y^*}\) denotes the permutation induced by \(\pi\) on the target
outputs. Therefore,
\[
F(z_P,T^*_\pi,t_P^*,Y_\pi^*,O_P)
=
P_{Y^*}F(z,T^*,t^*,Y^*,O).
\]

This establishes permutation equivariance with respect to the target nodes.
If the permutation acts only on nodes in \((Y^*)^c\), then every target node
is fixed and the induced permutation \(P_{Y^*}\) is the identity. The same
identity therefore reduces to
\[
F(z_P,T^*_\pi,t_P^*,Y^*,O_P)
=
F(z,T^*,t^*,Y^*,O),
\]
which establishes permutation invariance with respect to permutations of the
non-target nodes.
\end{proof}

%% file: appendix/graph_size_appendix.tex
\subsection{Graph-Size and Observation-Count Transferability of CFGNP}\label{app:graph_size_trans}

The CFGNP architecture can be applied to causal graphs with different numbers of nodes and to observational datasets containing different numbers of samples. In particular, the learnable parameters of the model do not depend on either the graph size \(n\) or the number of observational samples \(N_{\mathrm{obs}}\). Consequently, a model trained using graphs with \(n\) nodes and \(N_{\mathrm{obs}}\) observational samples can be applied, without modifying its learned parameters, to graphs with \(m\) nodes and \(N_{\mathrm{obs}}'\) observational samples.

We formalize this property in the following proposition.

\begin{proposition}[Graph-size and observation-count transferability]
Let
\[
F_{\theta}^{(n,N_{\mathrm{obs}})}
\]
denote the CFGNP applied to a causal query with \(n\) endogenous variables and \(N_{\mathrm{obs}}\) observational samples. Suppose that
the number of message-passing layers is fixed depending on \(n\) and independently of \(N_{\mathrm{obs}}\), or, alternatively, repeated message-passing iterations reuse the same parameters.

Then the parameter space \(\Theta\) of \(F_{\theta}^{(n,N_{\mathrm{obs}})}\) is independent of both \(n\) and \(N_{\mathrm{obs}}\). Consequently, for arbitrary
\[
n,m,N_{\mathrm{obs}},N_{\mathrm{obs}}' \in \mathbb{N},
\]
the same parameter vector \(\theta \in \Theta\) defines both
\[
F_{\theta}^{(n,N_{\mathrm{obs}})}
\qquad\text{and}\qquad
F_{\theta}^{(m,N_{\mathrm{obs}}')}.
\]
Hence, a CFGNP trained on one graph size and one number of observational samples can be evaluated on different graph sizes and different numbers of observational samples without introducing, removing, or resizing any learned parameters.
\end{proposition}

\begin{proof}
We show that every learnable component of the CFGNP is independent of both the number of endogenous variables \(n\) and the number of observational samples \(N_{\mathrm{obs}}\).

First, consider the node-wise input transformations. Let
\[
h_O :
\mathbb{R}^{d_{\mathrm{data}}}
\rightarrow
\mathbb{R}^{d_{\mathrm{hidden}}}
\]
denote the MLP applied to the observational data. For
\[
O \in
\mathbb{R}^{N_{\mathrm{obs}} \times n \times d_{\mathrm{data}}},
\]
the same map \(h_O\) is applied independently to every node in every observational sample,
\[
O_{k,i}
\mapsto
h_O(O_{k,i}),
\qquad
k=1,\ldots,N_{\mathrm{obs}},
\qquad
i=1,\ldots,n.
\]
Thus, changing either \(n\) or \(N_{\mathrm{obs}}\) changes only the number of evaluations of \(h_O\), but not its input dimension, output dimension, or learned parameters. The same argument applies to the node-wise maps
$
h_{z,\mathrm{int}},
$ $
h_{z,\mathrm{common}},
$ $
h_{x,\mathrm{int}},
$ and $
h_{x,\mathrm{common}}$.

Next, consider a GAT layer. For a node \(i\), its update is of the form
\[
h_i'
=
\sigma\left(
\sum_{j\in\mathcal{N}(i)}
\alpha_{ij}Wh_j
\right),
\]
where
\[
\alpha_{ij}
=
\frac{\exp(e_{ij})}
{\sum_{k\in\mathcal{N}(i)}\exp(e_{ik})},
\]
and
\[
e_{ij}
=
\mathrm{LeakyReLU}
\left(
a^\top
\left[
Wh_i \Vert Wh_j
\right]
\right).
\]
The learnable quantities \(W\) and \(a\) depend only on the feature dimensions. They do not depend on the total number of nodes in the graph, the number of edges, or the number of observational samples represented in the graph. Changing \(n\) or \(N_{\mathrm{obs}}\) therefore changes only the set of nodes and neighborhoods on which the same local update rule is evaluated.

We now consider the observational aggregation graph \(G_c\). For \(n\) endogenous variables and \(N_{\mathrm{obs}}\) observational samples, its adjacency matrix is
\[
G_c^{(n,N_{\mathrm{obs}})}
=
\begin{pmatrix}
M_n & \mathbf{0}_n & \cdots & \mathbf{0}_n & I_n \\
\mathbf{0}_n & M_n & \ddots & \vdots & I_n \\
\vdots & \ddots & \ddots & \mathbf{0}_n & \vdots \\
\mathbf{0}_n & \cdots & \mathbf{0}_n & M_n & I_n \\
\mathbf{0}_n & \cdots & \cdots & \mathbf{0}_n & \mathbf{0}_n
\end{pmatrix}.
\]
Thus, \(G_c^{(n,N_{\mathrm{obs}})}\) contains \(N_{\mathrm{obs}}\) diagonal blocks \(M_n\), one for each observational sample, and \(N_{\mathrm{obs}}\) corresponding blocks \(I_n\) connecting the observational nodes to the aggregation nodes \(q\).

Changing \(N_{\mathrm{obs}}\) only changes the number of \(M_n\)- and \(I_n\)-blocks, while changing \(n\) only changes their dimensions. The construction rule itself remains unchanged. Since the GAT \(m_\theta^{(1)}\) applies the same attention and message-passing functions to every node and edge, its learned parameters are independent of both \(n\) and \(N_{\mathrm{obs}}\).

In particular, if
\[
[O',q']
=
m_\theta^{(1)}([O',q];G_c),
\]
then changing the number of observational samples only changes the size of \(O'\) and the number of incoming messages available to the aggregation nodes \(q\). It does not change the dimensionality of the representation
$
q_i' \in \mathbb{R}^{d_{\mathrm{hidden}}}
$
associated with each endogenous variable. Hence, the remainder of the architecture receives the same type of representation regardless of \(N_{\mathrm{obs}}\).

Next, consider the graph used to combine the observational representation with the factual observation. For \(n\) variables,
\[
G_C'^{(n)}
=
\begin{pmatrix}
    \mathbf{0}_n & I_n \\
    \mathbf{0}_n & I_n
\end{pmatrix}.
\]
For \(m\) variables, the corresponding graph is
\[
G_C'^{(m)}
=
\begin{pmatrix}
    \mathbf{0}_m & I_m \\
    \mathbf{0}_m & I_m
\end{pmatrix}.
\]
These graphs have different dimensions but are generated according to the same construction rule. Moreover, the number \(N_{\mathrm{obs}}\) no longer appears explicitly at this stage, because the observational samples have already been aggregated into the \(n\) representations contained in \(q'\). Therefore, the parameters of the GAT used in this stage are independent of both \(n\) and \(N_{\mathrm{obs}}\).

The same argument applies to the graph used for counterfactual inference. For \(n\) variables,
\[
G_C''^{(n)}
=
\begin{pmatrix}
    M_n & I_n \\
    I_n & M_n
\end{pmatrix},
\]
whereas for \(m\) variables,
\[
G_C''^{(m)}
=
\begin{pmatrix}
    M_m & I_m \\
    I_m & M_m
\end{pmatrix}.
\]
Changing the graph size changes only the number of nodes and edges on which the shared message-passing functions are evaluated. The learned parameters of the GAT are unchanged. Furthermore, \(N_{\mathrm{obs}}\) does not affect the dimensionality of the node representations entering this stage.

Finally, the output network is applied independently and identically to each target node,
\[
m_\theta^{(3)}
:
\mathbb{R}^{d_{\mathrm{hidden}}}
\rightarrow
\mathbb{R}^{3N_c}.
\]
Changing the number of endogenous variables changes only the number of node-wise predictions, while changing \(N_{\mathrm{obs}}\) has no effect on the dimensions of the output map. Hence, neither quantity changes the learned parameters of \(m_\theta^{(3)}\).

Consequently, every learned component of the CFGNP is parameterized only by quantities such as
$
d_{\mathrm{data}},
$ $
d_{\mathrm{hidden}},
$ $
N_c,
$
and not by \(n\) or \(N_{\mathrm{obs}}\). The matrices
$
M_n,
$ $
I_n,
$ and $
\mathbf{0}_n
$
and the repetition of the observational blocks determine only the structure and size of the computational graph; they contain no learned parameters. Therefore,
\[
\Theta_{n,N_{\mathrm{obs}}}
=
\Theta_{m,N_{\mathrm{obs}}'}
=
\Theta
\]
for arbitrary
\[
n,m,N_{\mathrm{obs}},N_{\mathrm{obs}}' \in \mathbb{N}.
\]
Thus, the same parameter vector \(\theta\) can be applied to causal graphs with different numbers of nodes and to observational datasets containing different numbers of samples.
\end{proof}

This establishes architectural transferability across both graph size and observation count. It does not, by itself, guarantee identical predictive performance when \(n\) or \(N_{\mathrm{obs}}\) differs substantially from the values encountered during training, since such performance additionally depends on the distribution of training and test instances.

%% file: appendix/extended_model_architecture.tex
\section{Extended Model Architecture}
In the following, let $M_n$ be the $n \times n$ matrix that has $0$ on the diagonal and everywhere else $1$. Let $I_n$ be the $n \times n$ identity matrix. Finally, let $\mathbf{0}_n$ be the $n \times n$ matrix with only $0$.

\subsection{Model Input}
Our input is a tuple $((z, T^*, t^*, Y^*), O)$ with $O  \in \mathbb{R}^{N_{obs} \times n \times d_{data}}$, $z \in \mathbb{R}^{1 \times n \times d_{data}}$, and $x^* \in \mathbb{R}^{N_{count} \times n \times d_{data}}$.

\subsection{Changing data dimension}
We apply MLPs on each node.
On $O$ we apply the same MLP $h_O$ across all nodes. However, for the other parts we want the model to learn to distinguish between nodes we intervene on and the rest. Therefore, we apply an MLP $h_{z, int}$ on the indices $T^*$ in $z$ and an MLP $h_{z, common}$ on the other nodes.
Similarly, we apply the MLP $h_{x, int}$ on the indices $T^*$ of $x^*$ and an MLP $h_{x, common}$ on the other indices.

\subsection{Propagation of Observational Information}
We reshape $O$ into $O' \in \mathbb{R}^{(N_{obs} \cdot n) \times d_{hidden}}$. Furthermore, we introduce a set of nodes $q \in \mathbb{R}^{n \times d_{hidden}}$, initialized as $q=\mathbf{0}$.

For the GAT we use the adjacency matrix
\[
G_c=
\begin{pmatrix}
M_n & \mathbf{0}_n & \cdots & \mathbf{0}_n & I_n \\
\mathbf{0}_n & M_n & \ddots & \vdots & \vdots \\
\vdots & \ddots & \ddots & \mathbf{0}_n & \vdots \\
\mathbf{0}_n & \cdots & \mathbf{0}_n & M_n & I_n \\
\mathbf{0}_n & \cdots & \cdots & \mathbf{0}_n & \mathbf{0}_n
\end{pmatrix}.
\]

We then propagate information on $G_c$ using a Graph Attention Network (GAT) $m^{(1)}_\theta$ consisting of $K$ layers. A single GAT layer updates each node representation by computing attention coefficients over its neighbors. Let $h_i^{(t)}$ denote the representation of node $i$ at layer $k$. The attention coefficients are computed as
\[
\alpha_{ij}^{(k)}
=
\frac{\exp(e_{ij}^{(k)})}
{\sum_{k\in\mathcal{N}(i)}\exp(e_{ik}^{(k)})},
\qquad
e_{ij}^{(k)}
=
\mathrm{LeakyReLU}\!\left(a^\top[Wh_i^{(k)} \Vert Wh_j^{(k)}]\right),
\]
where $W$ and $a$ are learnable parameters. The node representations are then updated according to
\[
h_i^{(k+1)}
=
\sigma\!\left(
\sum_{j\in\mathcal{N}(i)}
\alpha_{ij}^{(k)}Wh_j^{(k)}
\right).
\]

Applying the entire GAT to the initial node representations yields
\[
[O',q'] = m^{(1)}_\theta([O',q]).
\]

\subsection{Mixing Factual Observation}
We construct a new graph $G_C'$ and apply a GAT $m^{(2)}_\theta$ consisting of three layers.

The resulting adjacency matrix is
\[
G_C' =
\begin{pmatrix}
    \mathbf{0}_n & I_n \\
    \mathbf{0}_n & I_n
\end{pmatrix}.
\]

Applying the GAT three times yields
\[
[z',q''] = m^{(2)}_\theta([z,q']).
\]

\subsection{Estimating Counterfactual Distribution}
We build a new graph $G_C''$ with adjacency matrix
\[
G_C'' = \begin{pmatrix}
M_n & I_n \\
I_n & M_n
\end{pmatrix}.
\]

To propagate information through the entire graph, we use $\min\{\lfloor\frac{n}{2}\rfloor, 5\}$ GAT layers, collected in the MPNN $m_\theta^{(2)}$, which yields
\[
[z'', {x^*}'] = m_\theta^{(2)}([z', x^*]).
\]

Finally, we apply an MLP $m_\theta^{(3)}$ with hidden dimension $3N_c$. After reshaping, the network outputs the parameters $\{\mu_k,\sigma_k,w_k\}_{k=1}^{N_c}$ for each node.
\change{During training we use all outputs of the model to additionally obtain the variance which we can use to quantify uncertainty. When reporting the scores during evaluation, we use the weighted sum of $\mu$ and $w$.}

%% file: appendix/data_generation.tex
\section{Experimental Setup}

\input{appendix/model_params}

\subsection{CheXpert Experiment Setup}\label{app:chexpert_generation}
We use $N_{obs}=100$ for the CheXpert samples.
For each image, the disease information was represented by the 13 CheXpert findings.
We combine them into one \texttt{No Finding} attribute. The \texttt{Disease} attribute is just the complement.
The resulting dataset has some data imbalances. The Sex feature has $41.24\%$ female values and the rest is male. Only $17.24\%$ of samples have PA as value and only $9.05\%$ of the samples exhibit a disease. Finally, the average age in the samples is $60.66$ years.

\paragraph{Counterfactual Evaluation.} \label{app:details_classifier}
To evaluate which attributes a generated counterfactual image exhibits, we use classifiers/regression models.
For the classifiers and regression models of attributes Sex, AP/PA, Age, and Disease we stick closely to \cite{gichoya2022ai}. We apply the pretrained ResNet \citep{he2016deep} and subsequently apply multiple MLP layers. For regression attributes we have final dimension 1 and for the classification attributes we apply the softmax activation function in the end to normalize the outputs to probabilities.
We add here the performance of the classification/regression models of the features. The Sex classifier has a test accuracy of $0.9901$ and a ROC-AUC of $0.9995$. The Age regression model has MAE $6.1473$. Further, the AP/PA classifier has both accuracy and ROC-AUC of $1.0$. Finally, the \texttt{No Finding} classifier has accuracy $0.9010$ and ROC-AUC $0.9314$.
Figure \ref{fig:training_histories} shows the training curves for each of the classifiers.

\begin{figure}
    \centering
    \includegraphics[width=0.7\linewidth]{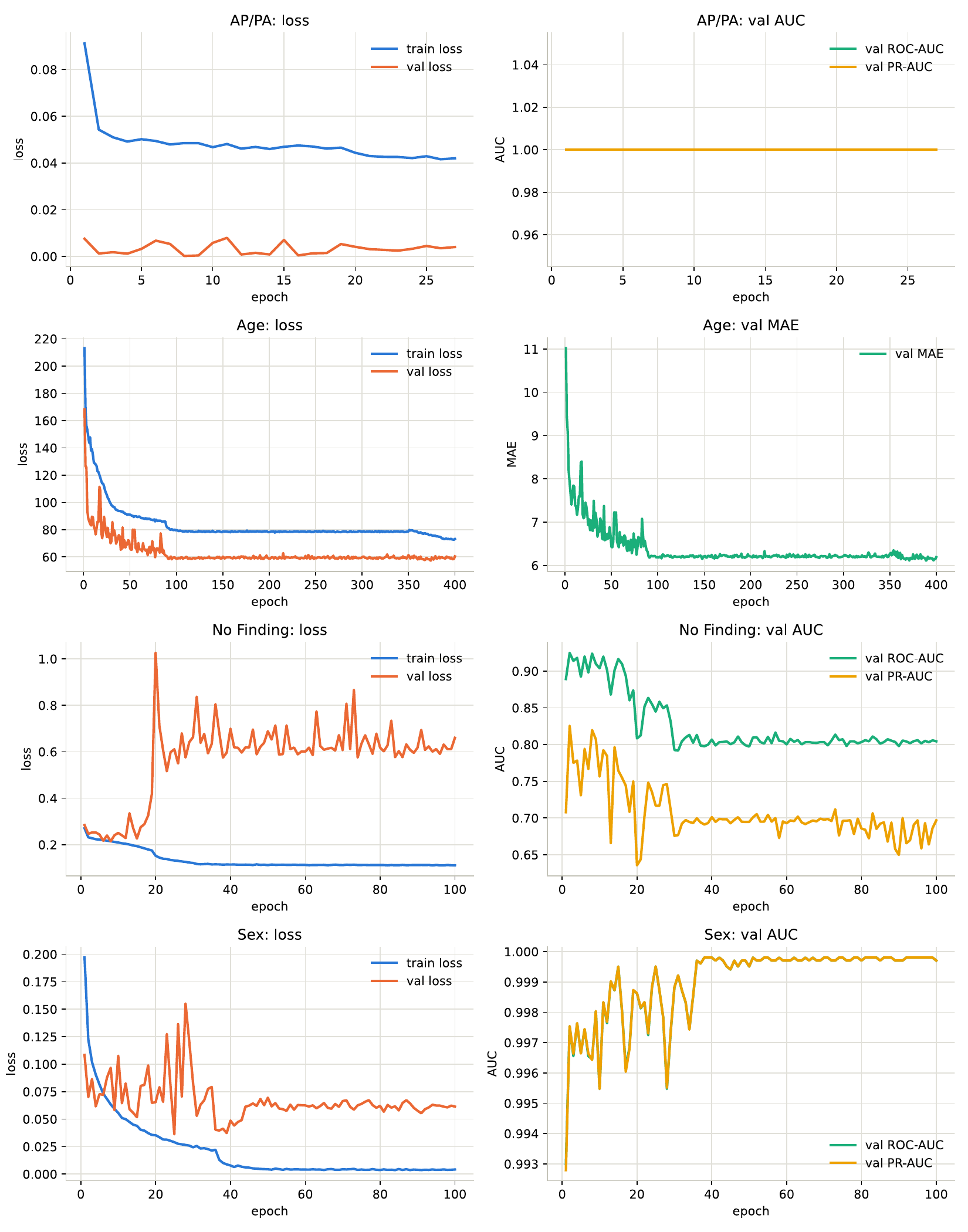}
    \caption{The training curves of the classifiers and regression models used for the classification of the generated CheXpert images.}
    \label{fig:training_histories}
\end{figure}

\paragraph{CheXpert training objective.}
To the normal classification/regression losses we also add a reconstruction loss to ensure that the images look like actual images. We use MSE, LPIPS\citep{zhang2018unreasonableeffectivenessdeepfeatures}, and the total variation \citep{javanmardi2018unsupervisedtotalvariationloss} of $\hat{x}_{\mathrm{cf}}$, to get
\begin{equation}
    \mathcal{L}_{\mathrm{rec}}
    =
    \lambda_M \operatorname{MSE}\left(x_{\mathrm{real}}, \hat{x}_{\mathrm{cf}}\right)
    +
    \lambda_L \operatorname{LPIPS}\left(x_{\mathrm{real}}, \hat{x}_{\mathrm{cf}}\right)
    +
    \lambda_{TV} \operatorname{TV}\left(\hat{x}_{\mathrm{cf}}\right),
\end{equation}

with $\lambda_M = 15$, $\lambda_L = 30$, and $\lambda_{TV} = 30$. 
The resulting total loss combines the classification, regression, and reconstruction objectives:
\begin{equation}
    \mathcal{L}
    =
    \sum_{i \in \mathcal{C}}
    \operatorname{CE}\left(y_i, \hat{y}_i\right)
    +
    \sum_{r \in \mathcal{R}}
    \frac{1}{t_r}
    \operatorname{MSE}\left(y_r, \hat{y}_r\right)
    +
    \mathcal{L}_{\mathrm{rec}}.
\end{equation}
where $\mathcal{C}$ and $\mathcal{R}$ denote the sets of classification and regression features, respectively.

\subsection{Synthetic Experiments Setup}\label{app:synth_exp_gen}
This appendix gives the complete definitions of the SCMs used in the (semi-) synthetic experiments. Their definitions follow the benchmark specification in
Appendix~E of \cite{sanchez-martin_vaca_2022}. They are reproduced here so that
the experimental data-generating processes can be understood without
the original paper. \change{Each of the datasets consists of $50000$ counterfactual samples, following \cite{sanchez-martin_vaca_2022}. They are split into $26250$ train samples, $11250$ validation samples, and $12500$ test samples.} We also keep to \cite{sanchez-martin_vaca_2022} for the computation of the standard deviation of the test MSE. We train the model with 5 different seeds and then compute mean and standard deviation of the runs.
We use $N_{obs}=100$ in all synthetic experiments. \change{We also note that for the datasets in this subsection (i.e. \ref{app:synth_exp_gen}) the counterfactual are interventionist counterfactuals. This has the reason that we want to exactly reproduce the setting in \cite{sanchez-martin_vaca_2022}.}

\subsubsection{Triangle}
\label{app:triangle}

\begin{figure}[H]
    \centering
    \includegraphics[width=0.3\linewidth]{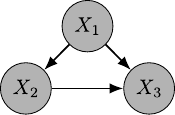}
    \caption{The causal graph of the triangle dataset. The causal graph has three nodes and three edges.}
    \label{fig:graph_triangle}
\end{figure}

The Triangle SCM has the underlying causal graph as depicted in Figure \ref{fig:graph_triangle}.
Hence $X_1$ affects $X_3$ through the direct path
$X_1\rightarrow X_3$ and through the mediated path
$X_1\rightarrow X_2\rightarrow X_3$.

\paragraph{Exogenous distributions.}

For the LIN and NLIN variants,
\begin{align}
U_1 &\sim \tfrac{1}{2}\mathcal{N}(-2,1.5)
       +\tfrac{1}{2}\mathcal{N}(1.5,1).
\end{align}
For NADD,
\begin{align}
U_1 &\sim \tfrac{1}{2}\mathcal{N}(-2.5,1)
       +\tfrac{1}{2}\mathcal{N}(2.5,1).
\end{align}
The remaining noises are
\begin{align}
\text{LIN:}\quad&
U_2\sim\mathcal{N}(0,1),&
U_3\sim\mathcal{N}(0,1),\\
\text{NLIN:}\quad&
U_2\sim\mathcal{N}(0,0.1),&
U_3\sim\mathcal{N}(0,1),\\
\text{NADD:}\quad&
U_2\sim\mathcal{N}(0,0.25),&
U_3\sim\mathcal{N}(0,0.0625).
\end{align}
Here the second parameter of $\mathcal{N}(\mu,\cdot)$ follows the convention
used in the benchmark specification.

\paragraph{Linear (LIN).}

\begin{align}
X_1 &= U_1,\\
X_2 &= -X_1+U_2,\\
X_3 &= X_1+0.25X_2+U_3.
\end{align}

\paragraph{Nonlinear additive (NLIN).}

\begin{align}
X_1 &= U_1,\\
X_2 &= -1+\frac{3}{1+\exp(-2X_1)}+U_2,\\
X_3 &= X_1+0.25X_2^2+U_3.
\end{align}

The causal mechanisms are nonlinear, but each exogenous variable enters its
structural equation additively.

\paragraph{Nonlinear non-additive (NADD).}

\begin{align}
X_1 &= U_1,\\
X_2 &= 0.25\,\operatorname{sgn}(U_2)
       \frac{X_1^2}{1+U_2^2},\\
X_3 &= -1+0.1\,\operatorname{sgn}(U_3)
       \left(X_1^2+X_2^2\right)+U_3.
\end{align}
The function $\operatorname{sgn}(\cdot)$ returns the sign of its argument.

\subsubsection{M-graph}
\label{app:mgraph}

\begin{figure}[H]
    \centering
    \includegraphics[width=0.4\linewidth]{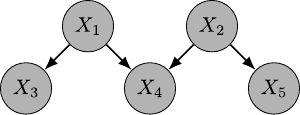}
    \caption{The causal graph of the M-shape data. It has five nodes and four edges.}
    \label{fig:graph_mshape}
\end{figure}

The M-graph contains five endogenous variables as shown in Figure \ref{fig:graph_mshape}.
All exogenous variables are independent standard normal variables:
\begin{equation}
    U_i\sim\mathcal{N}(0,1),\qquad i=1,\ldots,5.
\end{equation}

\paragraph{Linear (LIN).}

\begin{align}
X_1 &= U_1,\\
X_2 &= U_2,\\
X_3 &= X_1+U_3,\\
X_4 &= -X_2+0.5X_1+U_4,\\
X_5 &= -1.5X_2+U_5.
\end{align}

\paragraph{Nonlinear additive (NLIN).}

\begin{align}
X_1 &= U_1,\\
X_2 &= U_2,\\
X_3 &= X_1+0.5X_1^2+U_3,\\
X_4 &= -X_2+0.5X_1^2+U_4,\\
X_5 &= -1.5X_2^2+U_5.
\end{align}

\paragraph{Nonlinear non-additive (NADD).}

\begin{align}
X_1 &= U_1,\\
X_2 &= U_2,\\
X_3 &= X_1U_3,\\
X_4 &= \left(-X_2+0.5X_1^2\right)U_4,\\
X_5 &= -1.5X_2^2U_5.
\end{align}

\subsubsection{German Loan}
\label{app:loan}

\begin{figure}[H]
    \centering
    \includegraphics[width=0.5\linewidth]{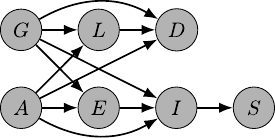}
    \caption{The causal graph of the German loan data. The nodes are real attributes used for German credit rating system. they are connected based on the real-world connections.}
    \label{fig:graph_loan}
\end{figure}

The Loan benchmark is a semi-synthetic SCM from the causal-recourse literature
and is inspired by the German Credit domain. It contains seven endogenous
variables as shown in Figure \ref{fig:graph_loan}.
In the figure, $G$ is gender, $A$ age, $E$ education, $L$ loan amount, $D$ loan
duration, $I$ income, and $S$ savings.

The structural assignments are
\begin{align}
G &= U_G,\\
A &= -35+U_A,\\
E &= -0.5+
\left(
1+\exp\left[
1-0.5G-\left(1+\exp(-0.1A)\right)^{-1}-U_E
\right]
\right)^{-1},\\
L &= 1+0.01(A-5)(5-A)+G+U_L,\\
D &= -1+0.1A+2G+L+U_D,\\
I &= -4+0.1(A+35)+2G+GE+U_I,\\
S &= -4+1.5\,\mathbb{I}_{\{I>0\}}I+U_S.
\end{align}
Here $\mathbb{I}_{\{I>0\}}$ is one if $I>0$ and zero otherwise.

The exogenous distributions are
\begin{align}
U_G &\sim \operatorname{Bernoulli}(0.5),\\
U_A &\sim \operatorname{Gamma}(10,3.5),\\
U_E &\sim \mathcal{N}(0,0.25),\\
U_L &\sim \mathcal{N}(0,4),\\
U_D &\sim \mathcal{N}(0,9),\\
U_I &\sim \mathcal{N}(0,4),\\
U_S &\sim \mathcal{N}(0,25).
\end{align}

\subsection{Structural Equations for Artificial Graphs}\label{datageneration_sizeexp}
\change{
We now describe the generation of counterfactuals on the artificial graphs used to get the results in Figure \ref{fig:size_experiment_plots}. The counterfactuals in this dataset stem from many different causal graphs and are backtracking counterfactuals. The cross-world coupling for node $i$ is $p(U_i^*|U_i) = \mathcal{N}(U_i, \sigma_i^2)$. $\sigma_i$ is chosen individually for each sample and node.
}

Variables are generated as $X=f(\operatorname{Pa}(X))+\epsilon$, where $f$ is sampled from the following set of transformations:
\begin{align*}
\text{Linear:}\qquad
f_1(\operatorname{Pa}(X))
    &= a \sum_{X' \in \operatorname{Pa}(X)} X' + b,
    && a=0.5,\; b=0,\\
\text{Sum-sqrt:}\qquad
f_2(\operatorname{Pa}(X))
    &= a \sum_{X' \in \operatorname{Pa}(X)} \sqrt{|X'|} + b,
    && a=0.5,\; b=0,\\
\text{Sum-sine:}\qquad
f_3(\operatorname{Pa}(X))
    &= a \sum_{X' \in \operatorname{Pa}(X)} \sin(cX') + b,
    && a=1,\; b=0,\; c=0.5,\\
\text{Sum-tanh:}\qquad
f_4(\operatorname{Pa}(X))
    &= a \sum_{X' \in \operatorname{Pa}(X)} \tanh(cX') + b,
    && a=1,\; b=0,\; c=2,\\
\text{Geometric mean:}\qquad
f_5(\operatorname{Pa}(X))
    &= a
    \left(
        \prod_{X' \in \operatorname{Pa}(X)} |X'|
    \right)^{\frac{1}{|\operatorname{Pa}(X)|}}
    + b,
    && a=3,\; b=0.1,\\
\text{Log-sum-exp:}\qquad
f_6(\operatorname{Pa}(X))
    &= a \log\left(
        \sum_{X' \in \operatorname{Pa}(X)} e^{X'}
    \right) + b,
    && a=1,\; b=\log 2,\\
\text{Sqrt-sum:}\qquad
f_7(\operatorname{Pa}(X))
    &= a \sqrt{
        \left|
            \sum_{X' \in \operatorname{Pa}(X)} X'
        \right|
    } + b,
    && a=1,\; b=0.
\end{align*}

\input{appendix/two_node_generation}

%% file: appendix/model_params.tex
\subsection{Model Hyperparameters}

For the (semi-)synthetic datasets, we use a hidden dimension of 80, 5 attention heads in the graph attention network (GAT) layers \citep{brody2022attentivegraphattentionnetworks}, and a batch size of 128. For the predicted MoG we use $N_C=3$ Gaussian mixture components.
For the image-based counterfactual datasets, which consist of chest X-rays, we use a hidden dimension of 256, 9 attention heads in the GAT layers, and a batch size of 32. We choose a hidden dimension of 256 to match the dimensionality of the encoded image representations, providing a natural choice for the width of the network.

We train each model for 200 epochs although we include early stopping after $25$ epochs. 
Further, we use the AdamW \citep{loshchilov2019decoupledweightdecayregularization} with initial learning rate $10^{-4}$ and a weight decay of $10^{-5}$.

%% file: appendix/two_node_generation.tex
\subsection{Identifiability in a Two-Node SCM}
\label{sec:id_experiments_proof}

\paragraph{Setup.}
Let $\bm X=(X_0,X_1)$ follow the SCM
\begin{equation}
    X_0 := U_0,
    \qquad
    X_1 := aX_0+\langle \bm g_\lambda(X_0),\bm W\rangle,
    \label{eq:scm}
\end{equation}
where $U_0\sim\mathcal N(0,1)$,
$\bm W\sim\mathcal N(\bm 0,I_2)$ are independent, and
\begin{equation}
    s(x):=\sqrt{\sigma_u^2+\sigma_v^2x^2},
    \qquad
    \varphi(x):=\arctan(\sigma_vx/\sigma_u),
\end{equation}
\begin{equation}
    \bm g_\lambda(x)
    :=
    s(x)
    \big(
        \cos(\lambda\varphi(x)),
        \sin(\lambda\varphi(x))
    \big),
    \qquad
    \lambda\in[0,1].
    \label{eq:g}
\end{equation}
Here $a$ controls the observational mechanism, whereas $\lambda$ controls
the cross-world coupling.

\begin{proposition}[Observational invariance]
\label{prop:obs-invariance}
For every $\lambda$,
\begin{equation}
    X_0\sim\mathcal N(0,1),
    \qquad
    X_1\mid X_0=x
    \sim
    \mathcal N\!\left(ax,s(x)^2\right).
    \label{eq:obs-law}
\end{equation}
Hence the observational distribution is independent of $\lambda$.
The same is true of the intervention
$P_{\mathrm{do}(X_0=\alpha)}$.
\end{proposition}

\begin{proof}
Since
\[
    \|\bm g_\lambda(x)\|^2
    =
    s(x)^2
    \big(
        \cos^2(\lambda\varphi(x))
        +
        \sin^2(\lambda\varphi(x))
    \big)
    =
    s(x)^2,
\]
we have
$\langle\bm g_\lambda(x),\bm W\rangle
\sim\mathcal N(0,s(x)^2)$,
which gives \eqref{eq:obs-law}.
Under $\mathrm{do}(X_0=\alpha)$,
\[
    X_1
    =
    a\alpha+\langle\bm g_\lambda(\alpha),\bm W\rangle
    \sim
    \mathcal N(a\alpha,s(\alpha)^2),
\]
again independently of $\lambda$.
\end{proof}

\paragraph{Backtracking semantics.}
Let the factual evidence be
\[
    e=\{X_0=x_0,\;X_1=x_1\},
    \qquad
    r:=x_1-ax_0.
\]
To construct a counterfactual world satisfying $X_0'=\alpha$, we use the
minimum-change exogenous coupling
\begin{equation}
    (U_0',\bm W')
    =
    (\alpha,\bm W).
    \label{eq:bt-coupling}
\end{equation}
Equivalently, \eqref{eq:bt-coupling} is the solution of
\[
    \min_{u_0',\bm w'}
    c_0|u_0'-U_0|^2+c_W\|\bm w'-\bm W\|^2
    \quad
    \text{s.t. }
    u_0'=\alpha,
\]
for any $c_0,c_W>0$.
Thus backtracking changes only the exogenous realization necessary to
produce $X_0'=\alpha$ while preserving the downstream noise $\bm W$.

\begin{theorem}[Backtracking counterfactual]
\label{thm:bt}
Under \eqref{eq:scm}--\eqref{eq:bt-coupling},
\begin{equation}
    X_1^*(\alpha)\mid e
    \sim
    \mathcal N
    \left(
        a\alpha+k_\lambda r,
        \tau_\lambda^2
    \right),
    \label{eq:bt-law}
\end{equation}
where
\begin{equation}
    k_\lambda
    :=
    \frac{s(\alpha)}{s(x_0)}
    \cos(\lambda\Delta\varphi),
    \qquad
    \tau_\lambda^2
    :=
    s(\alpha)^2
    \sin^2(\lambda\Delta\varphi),
    \label{eq:k-tau}
\end{equation}
and
\[
    \Delta\varphi
    :=
    \varphi(\alpha)-\varphi(x_0).
\]
Moreover, \eqref{eq:bt-law} is identical to the corresponding
abduction--action--prediction counterfactual.
\end{theorem}

\begin{proof}
The evidence fixes $U_0=x_0$ and imposes
\[
    \langle\bm g_\lambda(x_0),\bm W\rangle=r.
\]
With
$\hat{\bm g}_0:=\bm g_\lambda(x_0)/s(x_0)$,
Gaussian conditioning gives
\begin{equation}
    \bm W\mid e
    \sim
    \mathcal N
    \left(
        \frac{r}{s(x_0)}\hat{\bm g}_0,
        I-\hat{\bm g}_0\hat{\bm g}_0^\top
    \right).
    \label{eq:w-post}
\end{equation}
Backtracking changes $U_0$ to $\alpha$ while preserving $\bm W$, hence
\[
    X_1^*(\alpha)
    =
    a\alpha+\langle\bm g_\lambda(\alpha),\bm W\rangle.
\]
Therefore
\[
    \mathbb E[X_1^*(\alpha)\mid e]
    =
    a\alpha
    +
    \frac{r}{s(x_0)}
    \langle
        \bm g_\lambda(\alpha),
        \hat{\bm g}_0
    \rangle,
\]
and
\[
    \operatorname{Var}
    [X_1^*(\alpha)\mid e]
    =
    s(\alpha)^2
    -
    \frac{
        \langle
            \bm g_\lambda(\alpha),
            \bm g_\lambda(x_0)
        \rangle^2
    }{s(x_0)^2}.
\]
Since
\[
    \langle
        \bm g_\lambda(\alpha),
        \bm g_\lambda(x_0)
    \rangle
    =
    s(\alpha)s(x_0)
    \cos(\lambda\Delta\varphi),
\]
substitution yields \eqref{eq:bt-law}--\eqref{eq:k-tau}.
The interventional abduction--action--prediction construction evaluates the same
quantity
$a\alpha+\langle\bm g_\lambda(\alpha),\bm W\rangle$
under the same conditional law \eqref{eq:w-post}, so the two
counterfactual distributions coincide.
\end{proof}

\begin{corollary}[Cross-world non-identifiability]
\label{cor:nonid}
Although $\lambda$ is invisible to every observational and interventional
distribution, it changes the backtracking counterfactual whenever
$\lambda\mapsto(k_\lambda,\tau_\lambda)$ is non-constant.
In particular,
\[
    \tau_\lambda^2=0
    \iff
    \lambda\Delta\varphi\in\pi\mathbb Z.
\]
For $\lambda=0$,
\begin{equation}
    X_1^*(\alpha)\mid e
    =
    \delta_{
        a\alpha+
        \frac{s(\alpha)}{s(x_0)}r
    },
\end{equation}
whereas for $\lambda>0$ the counterfactual has positive conditional
variance for almost every $(x_0,\alpha)$ with $\alpha\neq x_0$.
\end{corollary}

For $\lambda=1$,
\[
    \bm g_1(x)
    =
    (\sigma_u,\sigma_vx),
\]
and
\begin{equation}
    k_1
    =
    \frac{
        \sigma_u^2+\sigma_v^2x_0\alpha
    }{
        s(x_0)^2
    },
    \qquad
    \tau_1^2
    =
    \frac{
        \sigma_u^2\sigma_v^2(\alpha-x_0)^2
    }{
        s(x_0)^2
    }.
\end{equation}

Use the prior
\[
    a\sim\mathcal N(0,\sigma_a^2),
    \qquad
    \lambda\sim\pi_\lambda,
    \qquad
    a\perp\lambda,
\]
and let
\[
    D_{\mathrm{obs}}
    =
    \{
        (x_0^{(m)},x_1^{(m)})
    \}_{m=1}^M
\]
be observational context data generated from \eqref{eq:obs-law}.
Writing
\[
    \mathcal E
    :=
    D_{\mathrm{obs}}\cup\{(x_0,x_1)\},
\]
we get for the optimal backtracking predictive
\begin{equation}
    p_{\mathrm{BCM}}
    (y\mid x_0,x_1,X_0'=\alpha,D_{\mathrm{obs}})
    =
    \int
    p_{\mathrm{bt}}
    (y\mid e,\alpha,a,\lambda)
    p(a,\lambda\mid\mathcal E)
    \,\mathrm da\,\mathrm d\lambda.
    \label{eq:bcm}
\end{equation}

\begin{proposition}[Exact posterior and model average]
\label{prop:bcm}
The likelihood is independent of $\lambda$, so
\[
    p(a,\lambda\mid\mathcal E)
    =
    p(a\mid\mathcal E)\pi_\lambda(\lambda).
\]
With
\[
    w_i:=s(x_0^{(i)})^{-2},
\]
we have
\begin{equation}
    a\mid\mathcal E
    \sim
    \mathcal N
    \left(
        \mu_a,
        \sigma^2_{a\mid\mathcal E}
    \right),
\end{equation}
where
\begin{equation}
    \sigma^{-2}_{a\mid\mathcal E}
    =
    \sigma_a^{-2}
    +
    \sum_{i\in\mathcal E}
    w_i(x_0^{(i)})^2,
    \qquad
    \mu_a
    =
    \sigma^2_{a\mid\mathcal E}
    \sum_{i\in\mathcal E}
    w_i x_0^{(i)}x_1^{(i)}.
    \label{eq:a-post}
\end{equation}
Consequently,
\begin{equation}
    p_{\mathrm{BCM}}(y\mid\cdot)
    =
    \int
    \mathcal N
    \left(
        y;
        \mu_a\beta_\lambda+k_\lambda x_1,
        \tau_\lambda^2
        +
        \sigma^2_{a\mid\mathcal E}\beta_\lambda^2
    \right)
    \pi_\lambda(\mathrm d\lambda),
    \label{eq:bcm-closed}
\end{equation}
where
\[
    \beta_\lambda
    :=
    \alpha-k_\lambda x_0.
\]
\end{proposition}

\begin{proof}
Proposition~\ref{prop:obs-invariance} implies that the likelihood does not
depend on $\lambda$, which yields the posterior factorization.
\eqref{eq:a-post} is the Gaussian posterior for weighted linear
regression.
Finally, substituting
$r=x_1-ax_0$
into \eqref{eq:bt-law} gives
\[
    a\alpha+k_\lambda r
    =
    a(\alpha-k_\lambda x_0)+k_\lambda x_1
    =
    a\beta_\lambda+k_\lambda x_1.
\]
Marginalizing the Gaussian posterior over $a$ gives
\eqref{eq:bcm-closed}.
\end{proof}

\paragraph{Identifiable versus non-identifiable regimes.}
Let
\[
    p_{\mathrm{BCM}}^*
    :=
    p_{\mathrm{bt}}
    (\cdot\mid e,\alpha,a^*,\lambda^*)
\]
denote the backtracking counterfactual under the true mechanism.
As $M\to\infty$,
\[
    \sigma^2_{a\mid\mathcal E}
    =
    \Theta(M^{-1})
    \to0,
    \qquad
    \mu_a\to a^*
    \quad\text{a.s.},
\]
and therefore
\begin{equation}
    p_{\mathrm{BCM}}
    \Longrightarrow
    \int
    \mathcal N
    \left(
        y;
        a^*\alpha+k_\lambda r^*,
        \tau_\lambda^2
    \right)
    \pi_\lambda(\mathrm d\lambda),
    \qquad
    r^*:=x_1-a^*x_0.
    \label{eq:limit}
\end{equation}

If $\pi_\lambda=\delta_{\lambda^*}$, the limit in \eqref{eq:limit} equals
$p_{\mathrm{BCM}}^*$.
If instead $\pi_\lambda$ is non-degenerate, then for almost every
$(x_0,\alpha)$ with $\alpha\neq x_0$ the map
\[
    \lambda
    \mapsto
    (k_\lambda,\tau_\lambda)
\]
is non-constant, so the limiting Bayes predictor remains a genuine mixture
over observationally indistinguishable cross-world mechanisms.

Thus even infinite observational data identifies $a$ but not $\lambda$:
uncertainty about the observational mechanism vanishes, while uncertainty
about the backtracking counterfactual need not.

%% file: appendix/baseline_description.tex
\change{\section{Counterfactual supervision of the baselines}
\label{app:baseline_cf_supervision}

CFGNP is trained on triples $(\mathbf{x}, \mathrm{do}(X_j = \alpha), \mathbf{x}^*)$ of a factual unit, an
intervention and the true counterfactual. The baselines, by design, learn only from factual data. To separate the
benefit of this supervision from the benefit of CFGNP's architecture, we train every baseline in two versions on the
same training split.

\paragraph{Original.} Each method optimises its own objective $\mathcal{L}_{\text{orig}}$ on the factual units
$\mathbf{x}$ of the training tasks and never sees a counterfactual.

\paragraph{+CF.} We add the squared error between the method's own counterfactual prediction and the true
counterfactual:
\[
  \mathcal{L}_{\text{CF}}(\theta) \;=\; \mathcal{L}_{\text{orig}}(\theta)
  \;+\; \lambda \, \frac{1}{|\mathcal{B}|} \sum_{(\mathbf{x}, j, \alpha, \mathbf{x}^*) \in \mathcal{B}}
  \big\lVert \hat{\mathbf{x}}^*_\theta(\mathbf{x}, \mathrm{do}(X_j = \alpha)) - \mathbf{x}^* \big\rVert_2^2,
  \qquad \lambda = 1 .
\]

\paragraph{+CF NLL.} Instead of the squared error, we add the Gaussian negative
log-likelihood of the true counterfactual:
\[
  \mathcal{L}_{\text{CF-NLL}}(\theta) \;=\; \mathcal{L}_{\text{orig}}(\theta)
  \;-\; \lambda \, \frac{1}{|\mathcal{B}|} \sum_{(\mathbf{x}, j, \alpha, \mathbf{x}^*) \in \mathcal{B}}
  \sum_{d=1}^{D} \log \mathcal{N}\big(x^*_d \,\big|\,
  \hat{x}^*_{\theta,d}(\mathbf{x}, \mathrm{do}(X_j = \alpha)),\, \sigma^2_{\theta,d}\big).
\]
VACA and VCI take $\sigma^2_{\theta,d}$ from their decoder. BGM and CFM predict
counterfactuals deterministically, so we give them a learned log-variance per
variable, with $\sigma^2_{\theta,d} \geq 10^{-4}$.

The prediction $\hat{\mathbf{x}}^*_\theta$ is computed with each method's own counterfactual procedure.
Gradients flow through both abduction and prediction,
so the supervision shapes the inferred noise as well as the mechanisms.

Both versions use the method's architecture, hyperparameters and training budget unchanged. The \emph{+CF} baselines therefore receive exactly the counterfactual supervision CFGNP
receives, and differ from CFGNP only in the model. We compare all variants, i.e. original, CF, and NLL, in table \ref{tab:baseline_cf_comparison}

\begin{table*}[h]
\centering
\small
\setlength{\tabcolsep}{2pt}
\newcommand{\res}[3]{\begin{tabular}{@{}c@{}}$#1 \times 10^{#3}$\\[-2pt]{\scriptsize$\pm #2 \times 10^{#3}$}\end{tabular}}
\newcommand{\best}[3]{{\boldmath\res{#1}{#2}{#3}}}
\begin{adjustbox}{max width=\linewidth}
\begin{tabular}{@{}lccccccc@{}}
\toprule
\multirow{2}{*}{Method} & \multicolumn{3}{c}{Triangle} & \multicolumn{3}{c}{M-shape} & \multirow{2}{*}{German Loan} \\
\cmidrule(lr){2-4} \cmidrule(lr){5-7}
 & LIN & NLIN & NADD & LIN & NLIN & NADD & \\
\midrule
BGM       & \res{4.52}{1.42}{-4} & \res{9.72}{7.23}{-4} & \res{2.11}{0.73}{-4} & \res{2.35}{0.93}{-5} & \res{2.52}{1.36}{-5} & \res{2.18}{0.09}{-3} & \res{1.91}{0.25}{-3} \\
VACA      & \res{7.02}{0.79}{-3} & \res{4.12}{0.45}{-3} & \res{7.86}{0.13}{-3} & \res{3.97}{0.09}{-3} & \res{3.93}{0.06}{-3} & \res{3.78}{0.04}{-3} & \res{11.41}{1.34}{-2} \\
VCI       & \res{2.05}{0.16}{-2} & \res{2.05}{0.86}{-2} & \res{1.23}{0.22}{-2} & \res{1.15}{0.61}{-2} & \res{1.60}{1.29}{-2} & \res{1.42}{2.63}{-1} & \res{1.95}{0.48}{-1} \\
CFM       & \res{1.72}{0.86}{-4} & \res{5.54}{2.54}{-4} & \res{5.26}{3.87}{-4} & \res{3.71}{0.50}{-5} & \res{5.16}{1.41}{-5} & \res{2.07}{0.05}{-3} & \res{1.52}{0.15}{-3} \\
\midrule
BGM-CF    & \best{1.56}{0.39}{-5} & \best{1.02}{0.28}{-5} & \best{8.63}{0.66}{-5} & \res{1.43}{0.69}{-5} & \best{1.55}{0.76}{-5} & \res{2.06}{0.15}{-3} & \res{2.11}{0.28}{-4} \\
VACA-CF   & \res{7.16}{2.16}{-5} & \res{1.47}{0.39}{-4} & \res{4.28}{0.89}{-4} & \res{6.03}{2.9}{-5} & \res{1.07}{0.28}{-4} & \best{1.37}{0.28}{-3} & \res{2.63}{0.43}{-4} \\
VCI-CF    & \res{1.62}{0.67}{-4} & \res{5.00}{0.49}{-4} & \res{4.46}{0.99}{-4} & \res{2.13}{1.24}{-3} & \res{9.40}{1.98}{-5} & \res{5.76}{2.76}{-3} & \res{2.34}{0.36}{-4} \\
CFM-CF    & \res{4.26}{1.12}{-5} & \res{2.65}{0.35}{-5} & \res{2.62}{0.79}{-4} & \res{3.89}{1.47}{-5} & \res{5.78}{1.54}{-5} & \res{1.94}{0.04}{-3} & \res{1.06}{0.10}{-3} \\
\midrule
BGM-NLL   & \res{2.86}{2.31}{-5} & \res{3.39}{1.40}{-5} & \res{2.17}{0.12}{-2} & \best{5.20}{0.96}{-6} & \res{3.05}{1.98}{-5} & \res{4.60}{0.20}{-1} & \best{1.04}{0.35}{-4} \\
VACA-NLL  & \res{2.69}{0.77}{-4} & \res{1.14}{0.30}{-3} & \res{3.53}{2.70}{-1} & \res{6.62}{1.44}{-5} & \res{6.76}{2.72}{-3} & \res{3.72}{0.89}{-1} & \res{1.24}{0.54}{-3} \\
VCI-NLL   & \res{2.50}{1.46}{-1} & \res{3.46}{0.92}{0}  & \res{2.84}{0.99}{0}  & \res{9.99}{1.46}{-1} & \res{2.02}{0.49}{0}  & \res{1.16}{0.19}{0}  & \res{2.26}{0.84}{1} \\
CFM-NLL   & \res{5.00}{0.95}{-4} & \res{4.97}{1.12}{-4} & \res{1.28}{0.45}{-2} & \res{1.12}{0.25}{-4} & \res{3.40}{1.08}{-4} & \res{4.44}{0.12}{-1} & \res{9.96}{1.10}{-3} \\
\bottomrule
\end{tabular}
\end{adjustbox}
\caption{Mean $\pm$ standard deviation achieved by BGM, VACA, CFM, and VCI across the Triangle, M-shape, and German Loan datasets. Methods with suffix -CF and -NLL are trained on counterfactual pairs with MSE and NLL loss respectively. 
    Lower values indicate better performance. Best scores are marked in bold.}
\label{tab:baseline_cf_comparison}
\end{table*}
}

%% file: appendix/additional_results.tex
\section{Additional Results}

\subsection{Additional CheXpert Results}\label{add_results_chexpert}
\subsubsection{CheXpert Experiments}
We include some counterfactual images generated by the model.
\begin{figure}[H]
    \centering
    \includegraphics[width=0.8\linewidth]{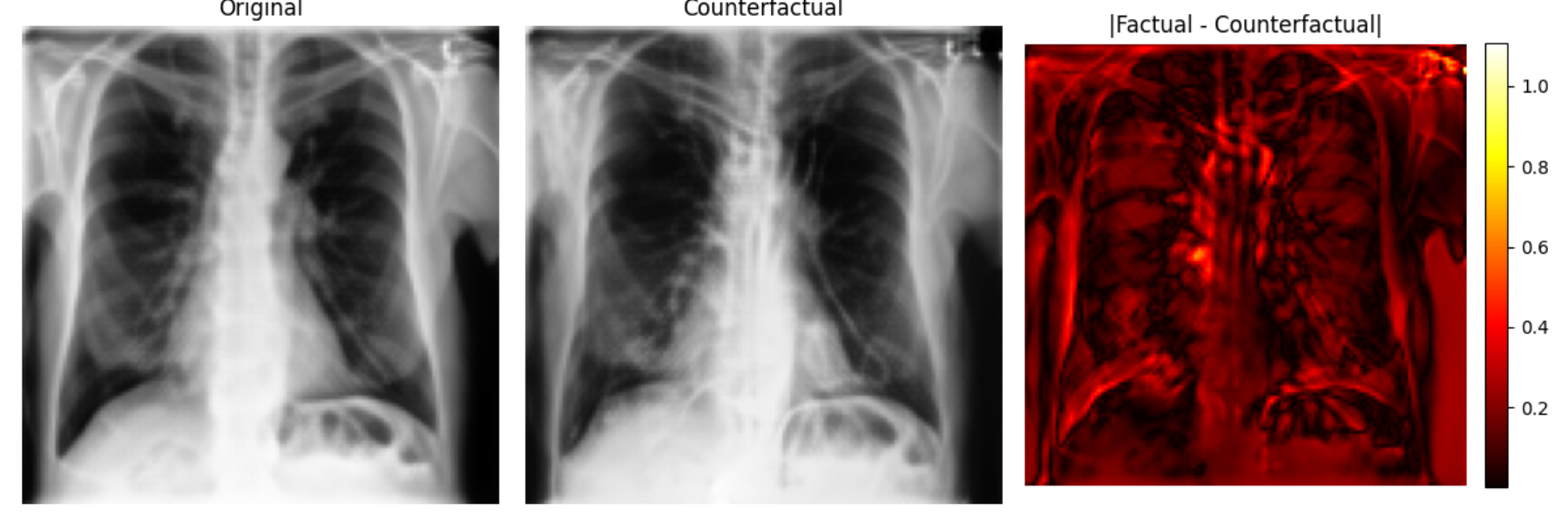}
    \caption{Example counterfactual of a CheXpert image. Observed factual unit (left) has age $80$ and the intervention value is $58$ in the counterfactual (middle). The difference between the pixels is on the right.}
    \label{fig:ch_age_example}
\end{figure}

\begin{figure}[H]
    \centering
    \includegraphics[width=0.8\linewidth]{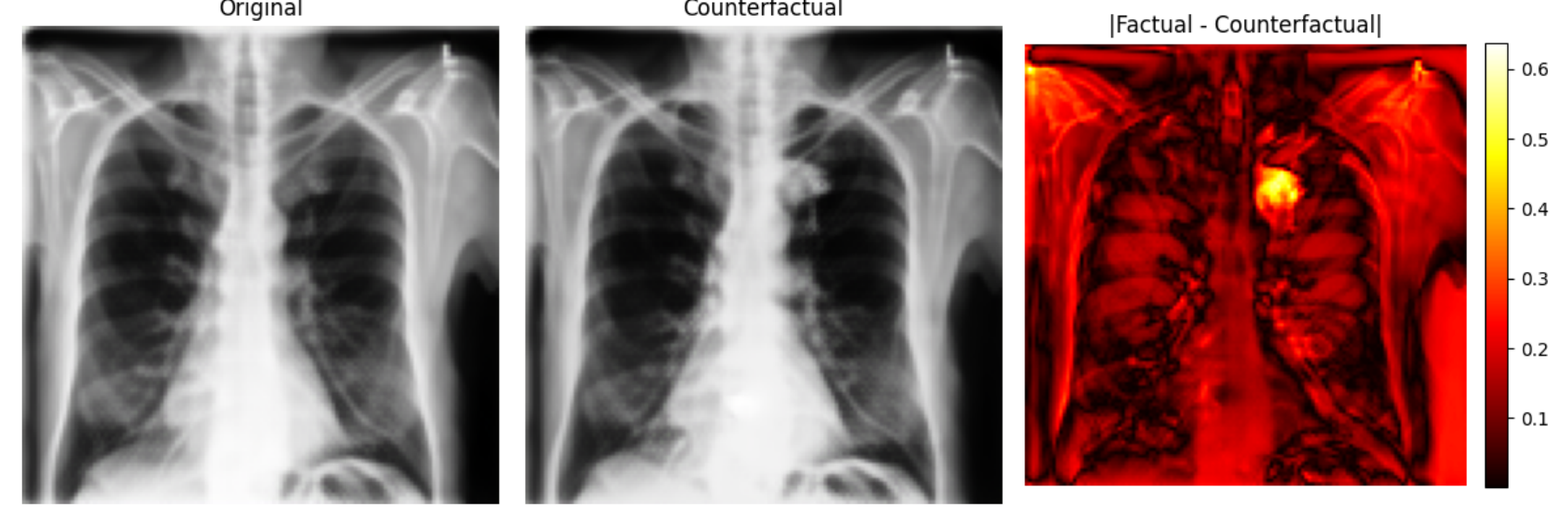}
    \caption{Example counterfactual of a CheXpert image. Observed factual unit (left) has AP/PA as $1$ and the intervention value is $0$ in the counterfactual (middle). The difference between the pixels is on the right.}
    \label{fig:ch_appa_example}
\end{figure}

\begin{figure}[H]
    \centering
    \includegraphics[width=0.8\linewidth]{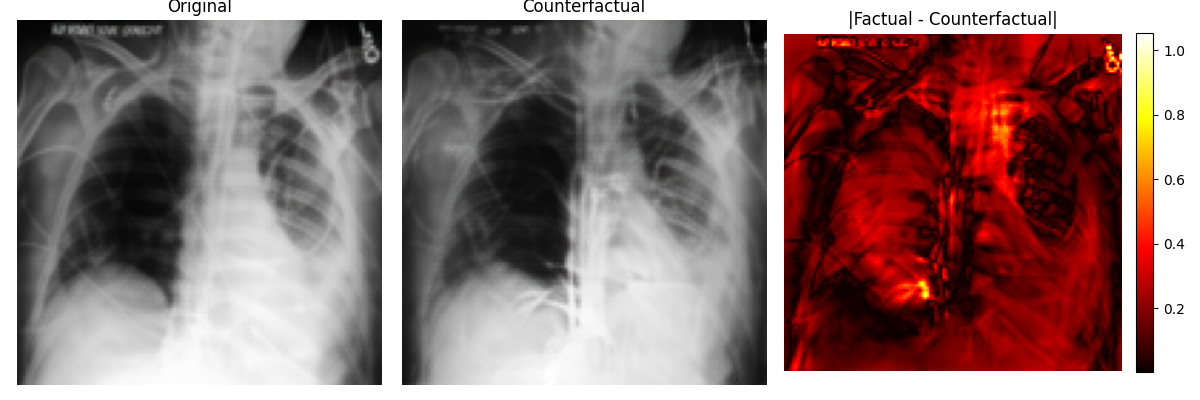}
    \caption{An image where the observed factual unit of age is $74$ and the counterfactual intervention is $85$.}
    \label{fig:normal_age}
\end{figure}

\subsubsection{CheXpert Out-of-Distribution Experiment}\label{app:chexpert_ood_app}

\change{
We evaluate CFGNP's generalization to unseen features on CheXpert. For each OOD setting, we exclude one feature $k_{\mathrm{OOD}}$, by removing the corresponding node from the graph, during training and evaluate on the full test set. This leverages CFGNP's ability to handle graphs of different sizes.
Table~\ref{tab:chexpert_ood} shows the results in that case. It seems that the OOD models converged towards random guessing, even though the age variable seems to be more accuracte.

\begin{table}[t]
\centering
\resizebox{\textwidth}{!}{%
\begin{tabular}{@{}l cccc cc cccc cccc@{}}
\toprule
\multirow{3}{*}{Method}
 & \multicolumn{4}{c}{Sex}
 & \multicolumn{2}{c}{Age (MAE) $\downarrow$}
 & \multicolumn{4}{c}{AP/PA}
 & \multicolumn{4}{c}{No Finding} \\
\cmidrule(lr){2-5} \cmidrule(lr){6-7} \cmidrule(lr){8-11} \cmidrule(lr){12-15}
 & \multicolumn{2}{c}{Changed} & \multicolumn{2}{c}{Preserved}
 & \multirow{2}{*}{Changed} & \multirow{2}{*}{Preserved}
 & \multicolumn{2}{c}{Changed} & \multicolumn{2}{c}{Preserved}
 & \multicolumn{2}{c}{Changed} & \multicolumn{2}{c}{Preserved} \\
\cmidrule(lr){2-3} \cmidrule(lr){4-5} \cmidrule(lr){8-9} \cmidrule(lr){10-11} \cmidrule(lr){12-13} \cmidrule(lr){14-15}
 & AUC $\uparrow$ & Acc.\ $\uparrow$ & AUC $\uparrow$ & Acc.\ $\uparrow$
 & & 
 & AUC $\uparrow$ & Acc.\ $\uparrow$ & AUC $\uparrow$ & Acc.\ $\uparrow$
 & AUC $\uparrow$ & Acc.\ $\uparrow$ & AUC $\uparrow$ & Acc.\ $\uparrow$ \\
\midrule
copy & $0.001$ & $0.016$ & $0.999$ & $0.990$ & $21.34$ & $6.00$ & $0.000$ & $0.000$ & $1.000$ & $1.000$ & $0.104$ & $0.278$ & $0.881$ & $0.897$ \\
\midrule
chexpert\_ood\_0 & $0.002_{\pm 0.001}$ & $0.031_{\pm 0.015}$ & $0.993_{\pm 0.004}$ & $0.951_{\pm 0.012}$ & $13.50_{\pm 0.69}$ & $10.72_{\pm 1.54}$ & $0.000_{\pm 0.000}$ & $0.000_{\pm 0.000}$ & $1.000_{\pm 0.000}$ & $0.991_{\pm 0.007}$ & $0.106_{\pm 0.014}$ & $0.278_{\pm 0.000}$ & $0.869_{\pm 0.010}$ & $0.903_{\pm 0.010}$ \\
chexpert\_ood\_1 & $0.002_{\pm 0.000}$ & $0.016_{\pm 0.000}$ & $0.993_{\pm 0.004}$ & $0.960_{\pm 0.016}$ & $20.82_{\pm 0.27}$ & $11.55_{\pm 0.10}$ & $0.000_{\pm 0.000}$ & $0.000_{\pm 0.000}$ & $1.000_{\pm 0.000}$ & $0.989_{\pm 0.008}$ & $0.091_{\pm 0.000}$ & $0.278_{\pm 0.000}$ & $0.878_{\pm 0.014}$ & $0.894_{\pm 0.013}$ \\
chexpert\_ood\_2 & $0.002_{\pm 0.002}$ & $0.041_{\pm 0.025}$ & $0.995_{\pm 0.004}$ & $0.963_{\pm 0.017}$ & $12.03_{\pm 0.92}$ & $9.45_{\pm 1.87}$ & $0.000_{\pm 0.000}$ & $0.010_{\pm 0.021}$ & $1.000_{\pm 0.000}$ & $0.997_{\pm 0.006}$ & $0.101_{\pm 0.044}$ & $0.278_{\pm 0.039}$ & $0.881_{\pm 0.016}$ & $0.911_{\pm 0.004}$ \\
chexpert\_ood\_3 & $0.002_{\pm 0.001}$ & $0.036_{\pm 0.004}$ & $0.996_{\pm 0.001}$ & $0.963_{\pm 0.005}$ & $11.08_{\pm 0.87}$ & $8.19_{\pm 0.46}$ & $0.000_{\pm 0.000}$ & $0.000_{\pm 0.000}$ & $1.000_{\pm 0.000}$ & $1.000_{\pm 0.001}$ & $0.094_{\pm 0.014}$ & $0.289_{\pm 0.025}$ & $0.886_{\pm 0.016}$ & $0.911_{\pm 0.009}$ \\
\bottomrule
\end{tabular}%
}
\caption{Counterfactual evaluation on CheXpert. \emph{Changed}: the intervened attribute scored against its new value. \emph{Preserved}: each unedited attribute scored against its original value, averaged over the other interventions. ROC-AUC and accuracy: higher is better; Age MAE (years): lower is better. Mean$_{\pm\mathrm{std}}$ over seeds (sample std).}
\label{tab:chexpert_ood}
\end{table}

}

\paragraph{Out-of-Distribution Feature Sex.}
We here include the results when the model was trained with OOD feature sex.
\begin{figure}[H]
    \centering
    \includegraphics[width=0.8\linewidth]{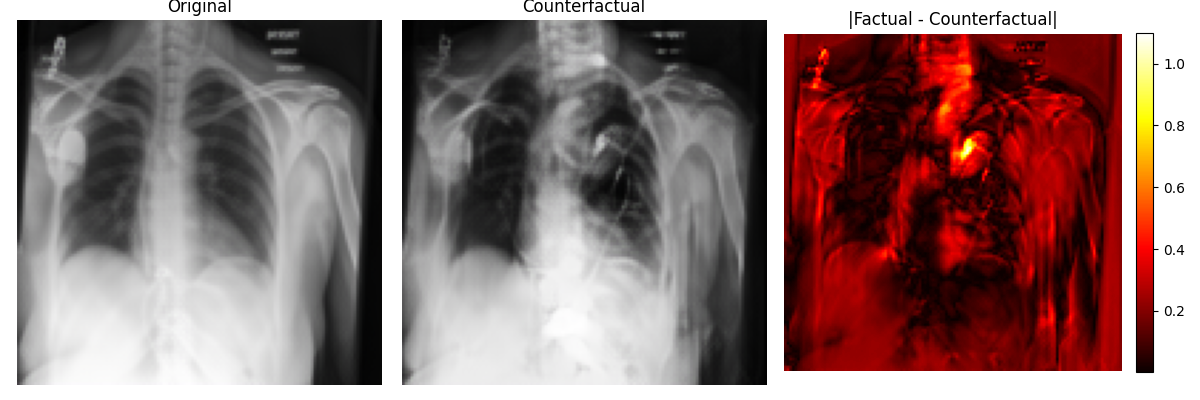}
    \caption{An image where the observed factual unit of age is $65$ and the counterfactual intervention is $72$. The model was trained with OOD index sex.}
    \label{fig:ood0_age}
\end{figure}
\begin{figure}[H]
    \centering
    \includegraphics[width=0.8\linewidth]{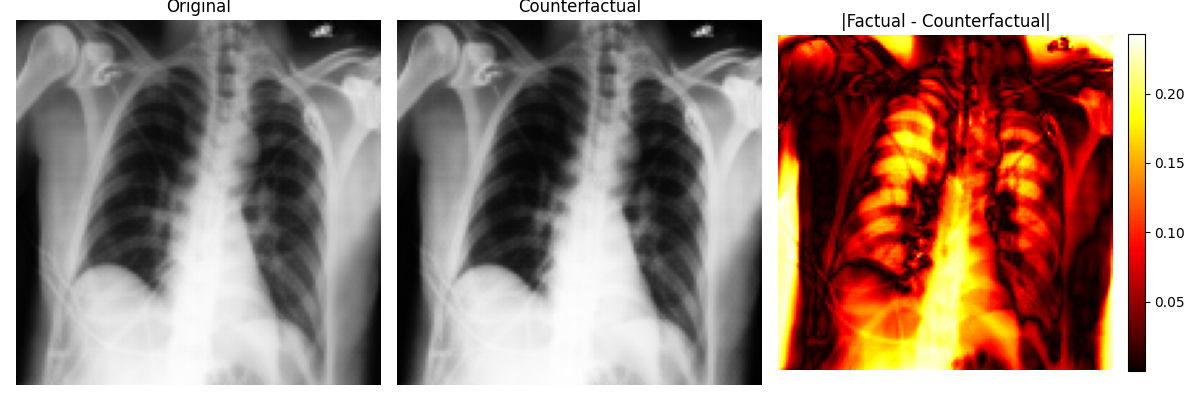}
    \caption{An image where the observed factual unit of AP/PA is 0 and the counterfactual intervention is 1. The model was trained with OOD index sex.}
    \label{fig:ood0_appa}
\end{figure}

\paragraph{Out-of-Distribution Feature Age.}
We here include the results when the model was trained with OOD feature age.

\begin{figure}[H]
    \centering
    \includegraphics[width=0.8\linewidth]{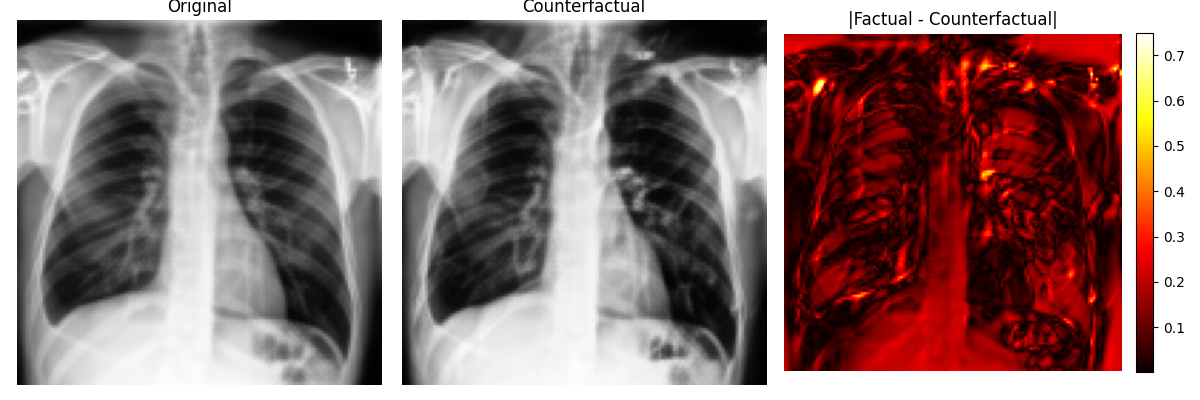}
    \caption{An image where the observed factual unit is age $30$ and the counterfactual intervention is age $52$. The model was trained with OOD index age.}
    \label{fig:ood1_age}
\end{figure}

\begin{figure}[H]
    \centering
    \includegraphics[width=0.8\linewidth]{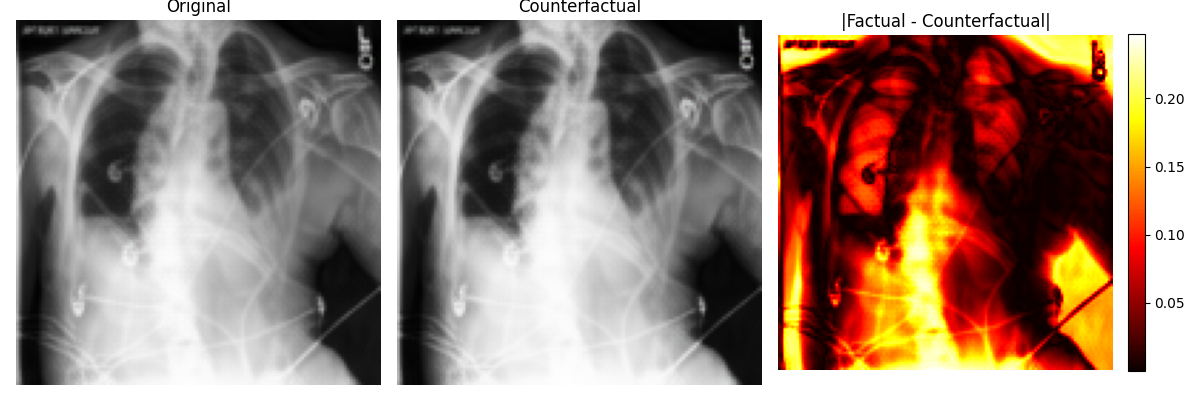}
    \caption{An image where the observed factual unit of sex is 1 and the counterfactual intervention is 0. The model was trained with OOD index age.}
    \label{fig:ood1_sed}
\end{figure}

\paragraph{Out-of-Distribution Feature AP/PA.}
We here include the results when the model was trained with OOD feature AP/PA.

\begin{figure}[H]
    \centering
    \includegraphics[width=0.8\linewidth]{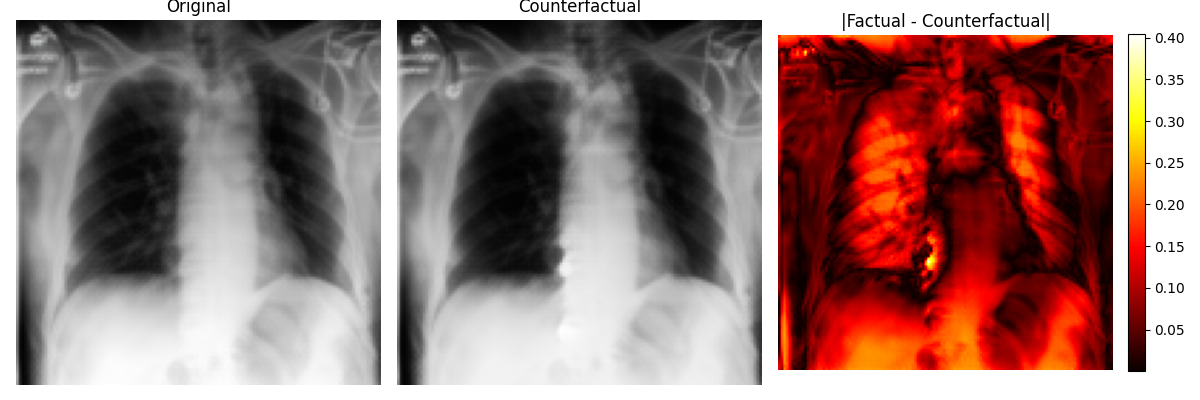}
    \caption{n image where the observed factual unit of sex is 0 and the counterfactual intervention is 1. The model was trained with OOD index AP/PA.}
    \label{fig:ood2_sex}
\end{figure}

\begin{figure}[H]
    \centering
    \includegraphics[width=0.8\linewidth]{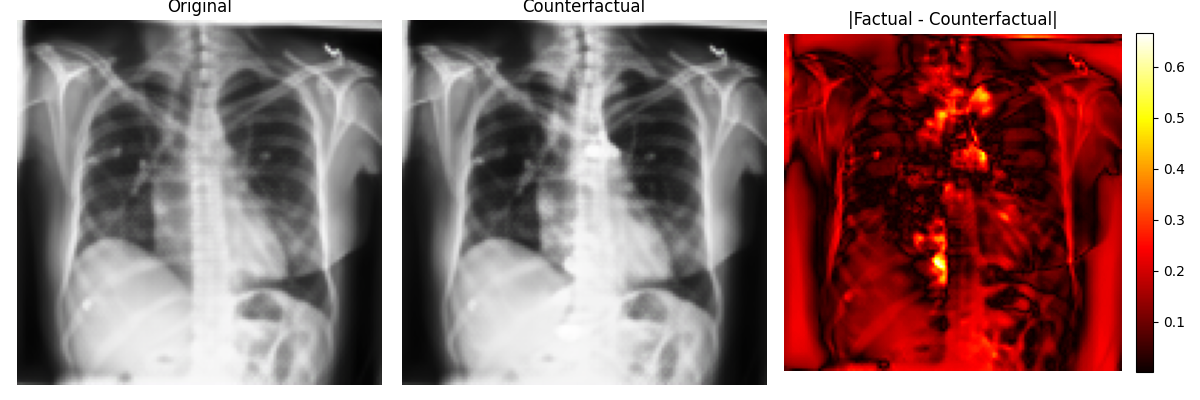}
    \caption{An image where the observed factual unit is age $26$ and the counterfactual intervention is age $64$. The model was trained with OOD index AP/PA.}
    \label{fig:ood2_age}
\end{figure}

\paragraph{Out-of-Distribution Feature Disease.}
We here include the results when the model was trained with OOD feature disease.

\begin{figure}[H]
    \centering
    \includegraphics[width=0.8\linewidth]{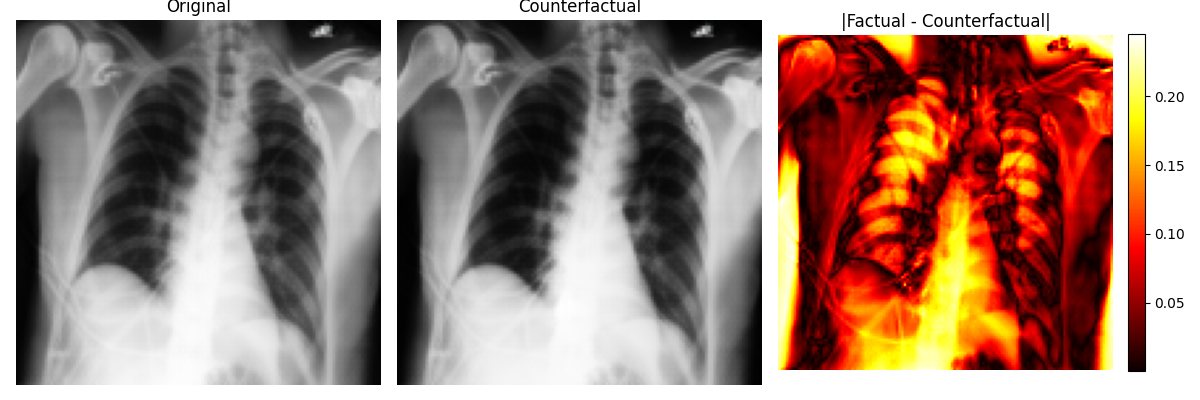}
    \caption{An image where the observed factual unit of AP/PA is 0 and the counterfactual intervention is 1. The model was trained with OOD index disease.}
    \label{fig:ood3_appa}
\end{figure}

\begin{figure}[H]
    \centering
    \includegraphics[width=0.8\linewidth]{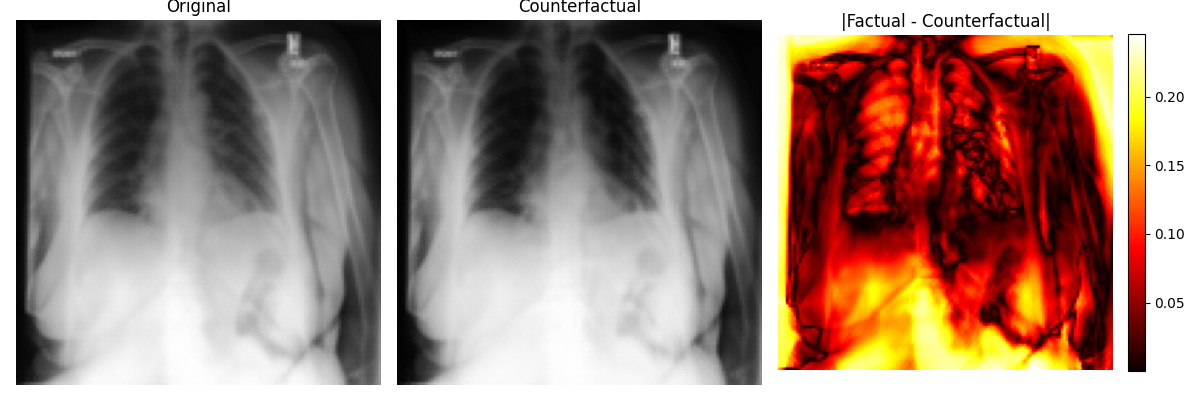}
    \caption{An image where the observed factual unit is age $67$ and the counterfactual intervention is age $37$. The model was trained with OOD index disease.}
    \label{fig:ood3_age}
\end{figure}

\input{appendix/size_results}

\input{appendix/runtime_appendix}

\input{appendix/two_node_results}

\input{appendix/candle_results}

\input{appendix/uncertainty_results}

\input{appendix/ablation_appendix}

\input{appendix/amortization_experiments_appendix}

%% file: appendix/size_results.tex
\subsection{Graph Size Results}\label{addition_size_results}

\begin{figure}[H]
    \centering
    \includegraphics[width=0.3\linewidth]{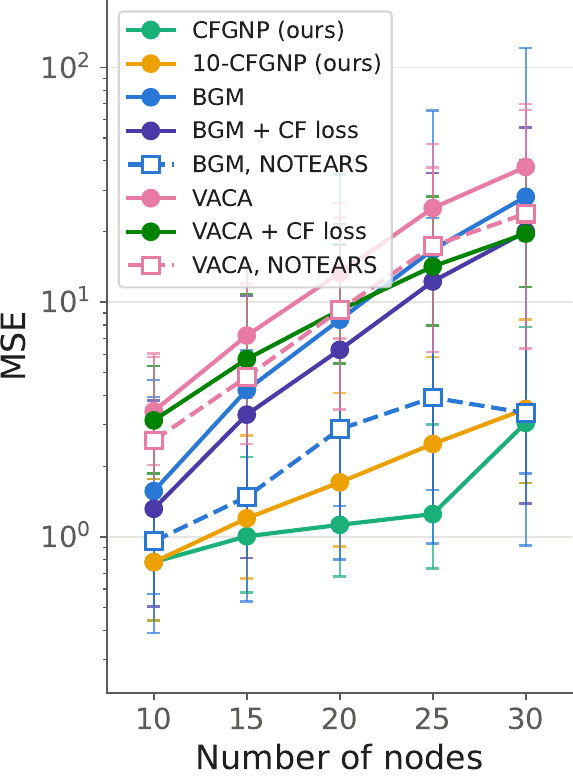}
    \hfill
    \includegraphics[width=0.3\linewidth]{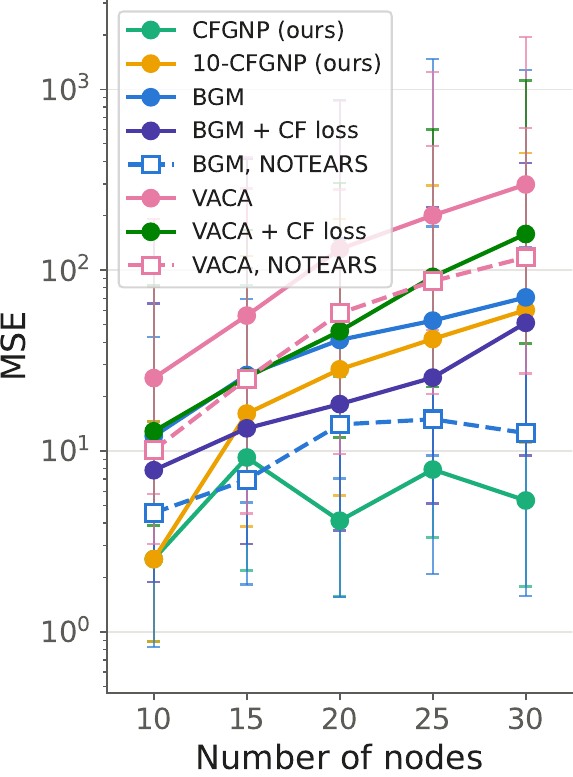}
    \caption{Performance (MSE) of models with regard to the size of the causal graph for sum-tanh (left) and geometric-mean (right) structural equations. The MSE is averaged over distinct causal graphs. Vertical lines show the 25th and 75th quantiles.}
    \label{fig:size-sqrt}
\end{figure}

%% file: appendix/runtime_appendix.tex
\subsection{Runtime of Experiments}\label{app:runtime_exp}
As before, we use a batch size of $128$ for the (semi-)synthetic experiments and a batch size of $32$ for the CheXpert dataset. All experiments are run on an NVIDIA H100 GPU.

\begin{table}[H]
    \centering
    \begin{tabular}{lrrrr}
        \toprule
        \textbf{Dataset} & \textbf{CFGNP (Ours)} & \textbf{BGM} & \textbf{CFM} & \textbf{VACA} \\
        \midrule
        German Loan     & 31.5719s & 5.2537s  & 115.1188s & 5.3015s \\
        Triangle LIN    &  8.1853s & 2.9878s  &  26.3716s & 3.7258s \\
        Triangle NLIN   &  7.9175s & 3.0323s  &  26.2195s & 3.7365s \\
        Triangle NADD   &  7.9084s & 3.0364s  &  26.3470s & 3.7963s \\
        M-Shape LIN     & 12.9088s & 3.3865s  &  41.5352s & 3.8361s \\
        M-Shape NLIN    & 12.9274s & 3.3941s  &  41.8969s & 3.8385s \\
        M-Shape NADD    & 12.9784s & 3.4396s  &  41.9043s & 3.8209s \\
        \midrule
        Artificial Size 10   & 18.4545s & 3.8820s  & -- & 2.4226s \\
        Artificial Size 15   & 27.6381s & 5.3534s  & -- & 2.6127s \\
        Artificial Size 20   & 39.2469s & 9.2543s  & -- & 2.6588s \\
        Artificial Size 25   & 54.7471s & 8.5072s  & -- & 2.5591s \\
        Artificial Size 30   & 74.2878s & 13.9569s & -- & 2.7384s \\
        \bottomrule
    \end{tabular}
    \caption{Runtime comparison of CFGNP (Ours), BGM, CFM, and VACA across (semi-) synthetic datasets. The datasets were used throughout the results section.}
    \label{tab:runtime_comparison}
\end{table}
\begin{table}[H]
    \centering
    \begin{tabular}{lrrr}
        \toprule
        \textbf{Dataset} & \textbf{CFGNP (Ours)} & \textbf{VCI} & \textbf{DeepBC} \\
        \midrule
        CheXpert & 47.7369s & 3.8506s & 6032.4057s \\
        \bottomrule
    \end{tabular}
    \caption{Runtime comparison of CFGNP (Ours), VCI, and DeepBC on the CheXpert dataset.}
    \label{tab:chexpert_runtime}
\end{table}

%% file: appendix/two_node_results.tex
\subsection{Convergence in (Non-)Identifiable Cases}
\label{sec:id_experiments}

We empirically evaluate the convergence behaviour of CFGNP.

To this end, we evaluate CFGNP on an SCM for which the ground-truth counterfactual distribution can be computed analytically. We consider a causal graph with two nodes and at most one edge. In Appendix~\ref{sec:id_experiments_proof}, we show that, under certain choices of parameters, such an SCM is identifiable. Furthermore, we analytically compute the ground-truth counterfactual distribution
\[
p_T^* := p_T(Y^* = y^* \mid z, T^*=t^*, \mathbf{f}, \mathcal{G}),
\]
given the true causal graph and structural mechanisms.

We compare the distribution predicted by CFGNP to both $p_T^*$ and
\[
p_T := p_T(Y^* = y^* \mid O, z, T^*=t^*),
\]
which conditions only on the observations rather than on knowledge of the SCM $(\mathbf{f},\mathcal{G})$.

Subsequently, we repeat the experiment for a non-identifiable parameterization of the two-node SCM.

In both cases, we use $40000$ samples, split into $21000$ train samples, $9000$ validation samples, and $10000$ test samples.

\begin{figure}[H]
    \centering
    \includegraphics[width=\linewidth]{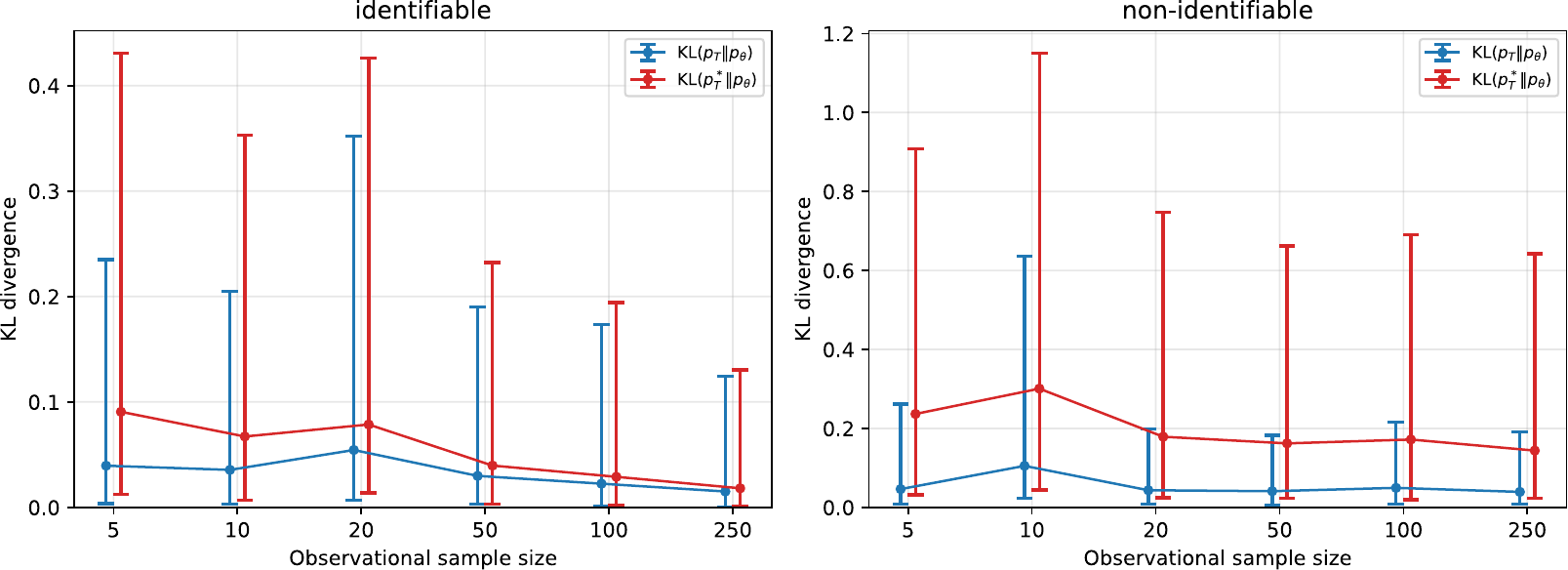}
    \caption{KL divergences for the two-node experiments across different numbers $N_{\mathrm{obs}}$ of observations supplied to each sample. Vertical lines show the 10th and 90th quantiles. Results are shown for the identifiable (left) and non-identifiable (right) settings.}
    \label{fig:id_experiments}
\end{figure}

%% file: appendix/candle_results.tex
\change{
\subsection{Semi-synthetic image benchmark: CANDLE}
\label{app:candle}

\paragraph{Data.}
CANDLE~\citep{reddy2022causally} consists of 12{,}546 rendered images of a single object in a scene, each labelled
with six generative factors: object shape (5 values), colour (5), size (3), rotation (6), scene (16) and light
direction (3). We resize all images to $64\times 64$ and use a seeded split: 25\% of the images for testing and,
of the remainder, 30\% for validation and 70\% for training. In the causal graph of CANDLE the factors do not
cause each other (they are related only through confounders) and all of them cause the image. Changing one factor
should therefore change exactly that factor in the image and leave all other factors unchanged, which makes the
correctness of a counterfactual image measurable.

\paragraph{Queries and metrics.}
For every test image and every factor $j$ we pose one query, setting factor $j$ to a value drawn uniformly from
its other values (six queries per image). The generated images are scored by a factor classifier
(``oracle'': an ImageNet-pretrained ResNet-18 with one linear head per factor), trained on the real training
images; its mean accuracy on the real test images is $0.82$, which bounds the classifier-based metrics. We report
\begin{itemize}
    \item \textbf{effectiveness}: the fraction of queries for which the oracle detects the new value of the
          intervened factor in the counterfactual image;
    \item \textbf{preservation}: the fraction of non-intervened factors whose factual value the oracle still
          detects in the counterfactual image;
    \item \textbf{composition}: the pixel MSE (and LPIPS) between a test image and its counterfactual under the
          null intervention (a factor set to its own value), i.e.\ how well a method reproduces the unit itself.
\end{itemize}
Both accuracies are averaged over the six factors; all numbers are mean $\pm$ std over 5 seeds (seed-dependent
split, models and queries).

\paragraph{CFGNP.}
We use the image model of the CheXpert experiments without changes to architecture, loss or optimisation.
A VAE with the CheXpert architecture, trained on the CANDLE training images and then frozen, maps an image to
8 latent channels of $8\times 8$ values. CFGNP receives an observational context of 100 random training images. It
predicts the change of the latent channels, which is added to the factual latent and decoded by the VAE. To keep the evaluation independent of this training signal, the classifier used in training is a
separately trained CNN whose architecture differs from the ResNet-18 oracle used for evaluation.

\paragraph{Baselines.}
VCI~\citep{wu_counterfactual_2025} and DeepBC~\citep{kladny_deep_2024} are trained on the same split and share one
convolutional encoder/decoder. VCI uses its published objective (individual likelihood,
adversarial covariate term and latent divergence between factual and counterfactual posteriors).
We also include a copy baseline that copies the original input image. As we use a VAE in CFGNP to reduce image size, we have also added a copy rule that copies inside of the VAE.

Table \ref{tab:candle} shows the results of these experiments. While DeepBC does achieve slightly higher effectiveness and slightly lower composition MSE, CFGNPs preservation is much higher. Further, we can see CFGNP performs better than just copying. Especially compared to copying inside the VAE, which is what CFGNP would do, the preservation is only slightly lower in CFGNP but effectiveness is much higher. Naturally, Copying has a lower composition MSE.

\begin{table}[h]
\centering
\begin{tabular}{lccc}
\toprule
Method & Effectiveness $\uparrow$ & Preservation $\uparrow$ & Composition MSE $\downarrow$ \\
\midrule
VCI    & 0.192 $\pm$ 0.004 & 0.347 $\pm$ 0.022 & 0.0132 $\pm$ 0.0016 \\
DeepBC & 0.360 $\pm$ 0.003 & 0.413 $\pm$ 0.004 & 0.0015 $\pm$ 0.0001 \\
CFGNP  & 0.316 $\pm$ 0.011 & 0.694 $\pm$ 0.006 & 0.0021 $\pm$ 0.0002 \\
\midrule
Copy (input)     & 0.046 $\pm$ 0.001 & 0.820 $\pm$ 0.002 & 0.0000 $\pm$ 0.0000 \\
Copy (VAE) & 0.069 $\pm$ 0.004 & 0.765 $\pm$ 0.009 & 0.0010 $\pm$ 0.0000 \\
\bottomrule
\end{tabular}
\caption{Counterfactual image generation on CANDLE (mean $\pm$ std). Effectiveness and
preservation are oracle accuracies (higher is better, bounded by the oracle accuracy on real images);
composition is the reconstruction error under the null intervention (lower is better).}
\label{tab:candle}
\end{table}

\paragraph{Dataset Analysis.} 
We measure the dependence between the CANDLE factors with Cram\'er's $V$ for every pair of factor labels,
together with its bias-corrected version $\tilde V$ \citep{bergsma2013bias}. Chance levels and $p$-values come
from 1{,}000 permutations of one factor, Bonferroni-corrected for the 15 pairs. The conditional shift
$\Delta_{a \to b}$ is the total variation distance between $P(b \mid a = v)$ and $P(b)$, averaged over the
values $v$ of $a$, i.e.\ how much knowing $a$ changes the distribution of $b$.

Table \ref{tab:cramer_v_table} shows the values between each pair.
Only four of the 15 pairs are dependent, all with
moderate strength; the strongest is object--size ($\tilde V = 0.36$, $\Delta_{\text{size} \to \text{object}}
= 0.29$). All other pairs are exactly balanced ($\tilde V = 0$), so for most interventions the remaining factors are
unaffected under either counterfactual semantics, and preservation measures the same quantity for
interventional and backtracking methods.

\begin{table}[h]
\centering
\small
\begin{tabular}{lccccc}
\toprule
Pair $(a, b)$ & $V$ & $\tilde V$ & Chance $V$ (95\%) & $p$ & $\Delta_{a\to b}$ / $\Delta_{b\to a}$ \\
\midrule
object -- size     & 0.362 & 0.362 & 0.017 (0.026) & $\le$0.001$^*$ & 0.124 / 0.288 \\
size -- scene      & 0.193 & 0.190 & 0.034 (0.042) & $\le$0.001$^*$ & 0.082 / 0.091 \\
color -- scene     & 0.169 & 0.165 & 0.034 (0.040) & $\le$0.001$^*$ & 0.131 / 0.131 \\
object -- scene    & 0.119 & 0.114 & 0.034 (0.040) & $\le$0.001$^*$ & 0.076 / 0.074 \\
color -- size      & 0.010 & 0.000 & 0.018 (0.025) & 0.973 & 0.006 / 0.005 \\
object -- color    & 0.008 & 0.000 & 0.018 (0.023) & 1.000 & 0.006 / 0.006 \\
object -- rotation & 0.000 & 0.000 & 0.020 (0.025) & 1.000 & 0.000 / 0.000 \\
object -- light    & 0.000 & 0.000 & 0.017 (0.025) & 1.000 & 0.000 / 0.000 \\
color -- rotation  & 0.000 & 0.000 & 0.020 (0.025) & 1.000 & 0.000 / 0.000 \\
color -- light     & 0.000 & 0.000 & 0.017 (0.025) & 1.000 & 0.000 / 0.000 \\
size -- rotation   & 0.000 & 0.000 & 0.020 (0.027) & 1.000 & 0.000 / 0.000 \\
size -- light      & 0.000 & 0.000 & 0.012 (0.020) & 1.000 & 0.000 / 0.000 \\
rotation -- scene  & 0.000 & 0.000 & 0.035 (0.039) & 1.000 & 0.000 / 0.000 \\
rotation -- light  & 0.000 & 0.000 & 0.020 (0.027) & 1.000 & 0.000 / 0.000 \\
scene -- light     & 0.000 & 0.000 & 0.034 (0.042) & 1.000 & 0.000 / 0.000 \\
\bottomrule
\end{tabular}
\caption{Pairwise dependence of the CANDLE factor labels over all 12{,}546 images. $V$ is Cramér's V and $\tilde V$ is bias-corrected V. Chance is mean (95th percentile) of $V$ under 1{,}000 permutations. $p$ is the permutation p-value, $^*$ marks significant at the Bonferroni-corrected level $0.05/15$. $\Delta_{a\to b}$ is the total variation distance between $P(b \mid a = v)$ and $P(b)$, averaged over the values $v$ of $a$.}
\label{tab:cramer_v_table}
\end{table}

}

%% file: appendix/uncertainty_results.tex
\change{
\subsection{Predictive uncertainty under backtracking counterfactuals}
\label{app:uncertainty}

\paragraph{Setup.}
We test whether the uncertainty of CFGNP's predictive distribution reflects how much the observational context
reveals about the mechanism. We use a two-node SCM family
\[
  X_0 := U_0, \qquad X_1 := a X_0 + \sigma U_1, \qquad U_0, U_1 \sim \mathcal{N}(0,1),
\]
with known $\sigma = 0.5$ and a fresh edge weight $a \sim \mathcal{N}(0,1)$ per task. Counterfactuals are
\emph{backtracking} counterfactuals~\citep{von_kugelgen_backtracking_2023}: both structural equations are kept, and the
antecedent $X^*_j = \alpha$ is reached by changing the exogenous state. The counterfactual noise is drawn from $U^*_i \mid U_i \sim \mathcal{N}(\rho U_i, 1-\rho^2)$ with $\rho = 0.625$, and is then conditioned
on $X^*_j = \alpha$. Each task queries the root ($j=0$, predict $X^*_1$) or the leaf ($j=1$, predict $X^*_0$) with
equal probability, with $\alpha \sim \mathcal{U}(-2.5, 2.5)$. For $\rho < 1$ the counterfactual is stochastic even
under the true mechanism.

Two references are available in closed form:
\begin{itemize}
  \item $p^*_T$, the counterfactual distribution under the true $(a, \rho)$;
  \item $p_T$, the Bayes-optimal prediction given the context, which integrates $a$ over its conjugate posterior
        given the $M$ observational samples and the factual unit.
\end{itemize}
We train CFGNP with the negative log-likelihood of its mixture output for $M \in \{5, 10, 20, 50, 100, 250\}$,
on 30{,}000 tasks (70/15/15 split) and five seeds.

\paragraph{Metrics.}
\begin{itemize}
  \item The \emph{predictive standard deviation}: the standard deviation of the predicted mixture, averaged over
        test queries.
  \item $\mathrm{KL}(p_T \Vert p_\theta)$ and $\mathrm{KL}(p^*_T \Vert p_\theta)$, each estimated with 512 Monte
        Carlo samples per query and averaged over test queries.
\end{itemize}

\paragraph{Results.}
Figure~\ref{fig:uncertainty_repeat} shows the results.
\begin{itemize}
  \item CFGNP's predictive standard deviation falls from $0.59$ at $M=5$ to $0.44$ at $M=250$. It closely follows
        the Bayes-optimal one ($0.57 \to 0.44$) and stays slightly above it, so CFGNP is slightly conservative.
  \item It approaches the irreducible spread of the true counterfactual ($0.44$), which is non-zero because
        backtracking counterfactuals are stochastic.
  \item The KL divergence to the true counterfactual drops from $0.34$ to $0.07$.
  \item The KL divergence to the Bayes-optimal prediction drops from $0.13$ to $0.06$.
\end{itemize}
CFGNP's uncertainty therefore exceeds that of the true mechanism only by the epistemic uncertainty about $a$ that
remains given the context. This excess shrinks at the same rate as for the Bayes-optimal predictor and disappears as
the context grows.

\begin{figure}[t]
  \centering
  \includegraphics[width=\linewidth]{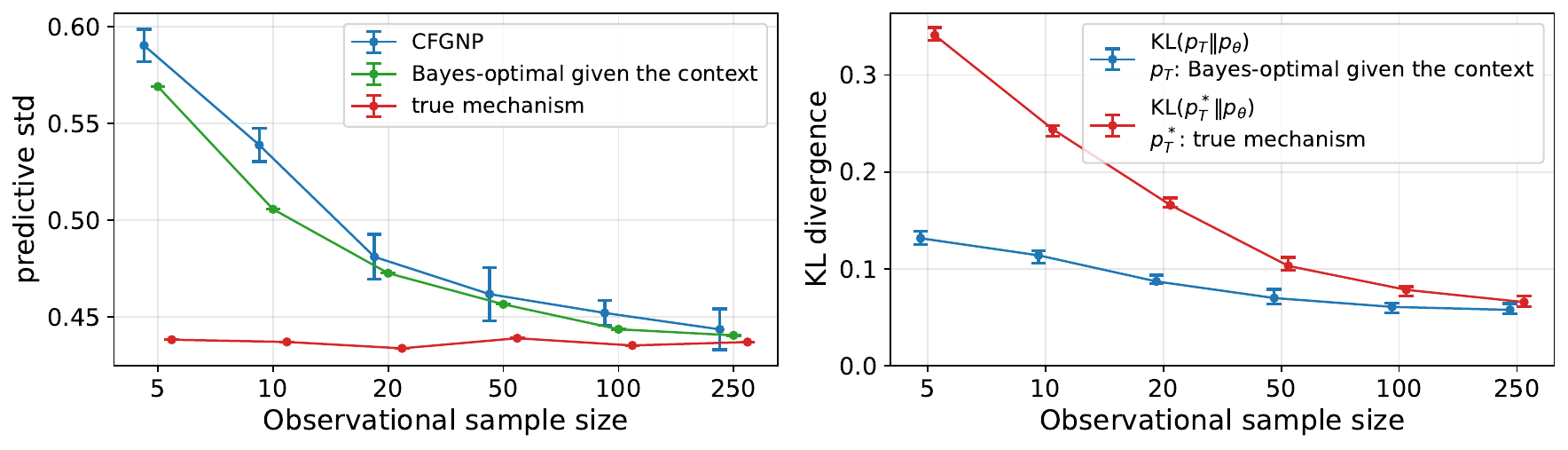}
  \caption{Two-node backtracking task. \textbf{Left:} predictive standard deviation of CFGNP, the Bayes-optimal
  prediction given the context ($p_T$), and the true mechanism ($p^*_T$), as mean $\pm$ std over five runs.
  \textbf{Right:} KL divergence of CFGNP's predictive distribution to $p_T$ and to $p^*_T$, as median and
  10--90\% range over five runs.}
  \label{fig:uncertainty_repeat}
\end{figure}

\subsection{Evaluation as Function of Observational Context Size}
\label{app:context-size-exp}

All CFGNP models are trained with $k = 100$ observational samples per task. To test how many of these samples the
model needs at test time, we evaluate the trained models with only the first
$k \in \{1, 5, 10, 25, 50, 75, 100\}$ observational samples of each test task, without retraining.
The error falls with $k$ for every dataset. This shows that CFGNP does in fact use the observational context. This can be seen in Figure \ref{fig:context-size}.

\begin{figure}[t]
  \centering
  \includegraphics[width=0.5\linewidth]{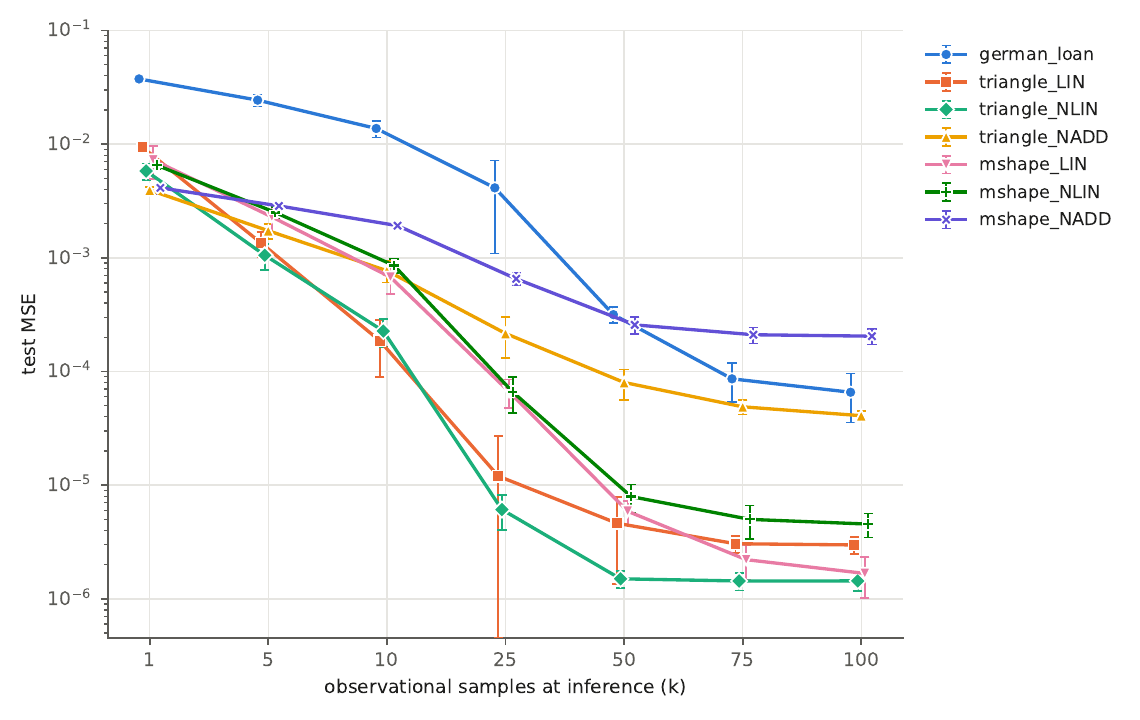}
  \caption{Test MSE of CFGNP (trained with 100 observational samples per task) when only the first $k$ observational
    samples of each test task are given at inference. Mean $\pm$ standard deviation.}
  \label{fig:context-size}
\end{figure}

\subsection{Predictive Uncertainty on Synthetic Datasets} \label{app:pred_uncertainty_synth}
In this section we examine the predictive uncertainty of CFGNP.
First, we compare the calibration error of the predicted uncertainty from CFGNP to the baselines. We compare against all baselines with the NLL counterfactual supervision as well as vanilla VACA and VCI. The other vanilla baselines do not provide probability distributions. while we can see that CFGNPs calibration error is low, especially the vanilla VACA baseline has slightly lower calibration error. Since VACA is trained to output proper distributions this does not mean that CFGNP is not calibrated.

\begin{table}[H]
\centering
\small
\resizebox{\linewidth}{!}{%
\begin{tabular}{lccccccc}
\toprule
Dataset & CFGNP & BGM (NLL) & CFM (NLL) & VACA (NLL) & VCI (NLL) & VACA & VCI \\
\midrule
Loan          & $\mathbf{0.037 \pm 0.019}$ & $0.282 \pm 0.015$ & $0.098 \pm 0.022$ & $0.084 \pm 0.050$ & $0.056 \pm 0.005$ & $0.111 \pm 0.023$ & $0.046 \pm 0.010$ \\
Triangle LIN  & $0.055 \pm 0.016$ & $0.375 \pm 0.062$ & $0.164 \pm 0.026$ & $0.193 \pm 0.051$ & $0.136 \pm 0.063$ & $\mathbf{0.048 \pm 0.021}$ & $0.075 \pm 0.039$ \\
Triangle NLIN & $0.109 \pm 0.018$ & $0.375 \pm 0.025$ & $0.199 \pm 0.030$ & $0.098 \pm 0.074$ & $0.026 \pm 0.020$ & $\mathbf{0.011 \pm 0.007}$ & $0.066 \pm 0.042$ \\
Triangle NADD & $0.142 \pm 0.012$ & $0.160 \pm 0.027$ & $0.184 \pm 0.019$ & $0.046 \pm 0.024$ & $0.043 \pm 0.012$ & $\mathbf{0.040 \pm 0.016}$ & $0.053 \pm 0.020$ \\
M-shape LIN   & $0.057 \pm 0.033$ & $0.434 \pm 0.012$ & $0.240 \pm 0.025$ & $0.343 \pm 0.020$ & $0.026 \pm 0.029$ & $\mathbf{0.004 \pm 0.001}$ & $0.027 \pm 0.028$ \\
M-shape NLIN  & $0.083 \pm 0.042$ & $0.365 \pm 0.051$ & $0.179 \pm 0.037$ & $0.247 \pm 0.040$ & $0.013 \pm 0.006$ & $\mathbf{0.004 \pm 0.002}$ & $0.012 \pm 0.010$ \\
M-shape NADD  & $0.091 \pm 0.016$ & $0.266 \pm 0.012$ & $0.216 \pm 0.011$ & $\mathbf{0.016 \pm 0.006}$ & $0.025 \pm 0.016$ & $0.079 \pm 0.006$ & $0.108 \pm 0.068$ \\
\bottomrule
\end{tabular}}
\caption{Calibration error of the predicted counterfactual distributions (lower is better). (NLL) for baselines trained with the counterfactual NLL, VACA and VCI are vanilla models. Mean $\pm$ std with best per dataset in bold.}
\label{tab:calibration_error}
\end{table}

Additionally, to test whether CFGNP's predictive uncertainty reflects how much information the context carries about the underlying SCM, we evaluate the trained models with only the first $k \in \{1, 5, 10, 25, 50, 75, 100\}$ observational context samples of each test task. For every scored node, the model predicts a Gaussian mixture over the counterfactual value. From it we report three metrics: the mixture's standard deviation, the fraction of true counterfactual values inside its central $90\%$ interval, and the negative log-likelihood (NLL) of the true value. Each is averaged over scored nodes and test tasks. The results are shown in Figure \ref{fig:ctx_uncertainty}
The predicted standard deviation decreases monotonically with $k$ on all seven datasets. It levels off at a dataset-dependent floor, which is highest for the  NADD case, so the predicted uncertainty shrinks as the context becomes more informative. At the training context size ($k = 75$--$100$), the $90\%$ intervals are close to nominal, and the NLL reaches its minimum on every dataset. For small contexts, the model is overconfident: coverage is $0.1$--$0.35$ at $k = 1$. This setting lies outside the training regime, since the models were trained with $100$ context samples and are not retrained for smaller $k$.

\begin{figure}[H]
    \centering
    \includegraphics[width=\linewidth]{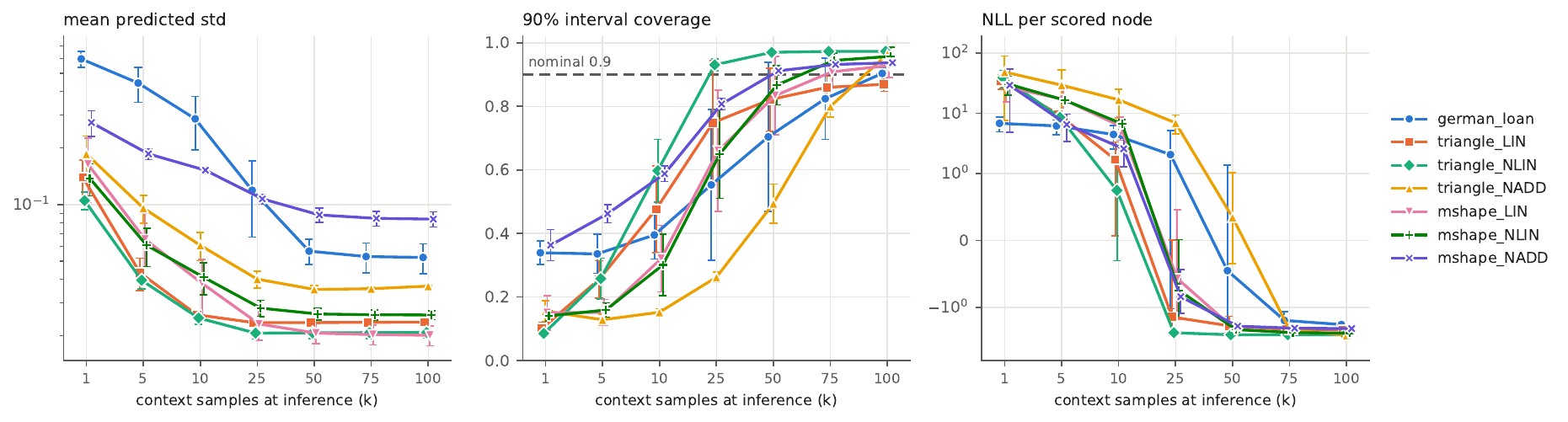}
    \caption{Predicted standard deviation, internal coverage, and NLL of CFGNP trained on $100$ observations across a range of context sizes.}
    \label{fig:ctx_uncertainty}
\end{figure}
}

%% file: appendix/ablation_appendix.tex
\subsection{Transformer Ablation}
\label{app:transformer_ablation}

To assess the importance of the CFGNP architecture, we consider a
Transformer-based ablation that replaces the core CFGNP processing mechanism
with a standard Transformer encoder while retaining the same input and output
representations. The observation, factual unit, and
counterfactual/intervention inputs are first embedded separately using the same
embedding dimensionality as CFGNP. Again, eparate embeddings are retained for
intervened and non-intervened variables. 
Then, we concatenate all these vectors and run a transformer architecture on them.
The output of the transformer for $x^*$ is handled in the same way as for CFGNP

Unless stated otherwise, all hyperparameters are kept identical to those used
for CFGNP, including the base hidden dimension, number of attention heads,
number of mixture components, and input/output dimensionality. As in CFGNP,
the effective embedding dimension is given by the base hidden dimension
multiplied by the number of attention heads. The Transformer-specific
hyperparameters are five encoder layers, a feed-forward dimension of four times
the embedding dimension, a dropout rate of $0.1$, and GELU activations.
Pre-layer normalization is used within each Transformer encoder layer, followed
by a final layer normalization after the encoder stack.

Table~\ref{tab:transformer_ablation} compares the resulting Transformer
ablation against the full CFGNP model. \change{Further, we compare CFGNP trained with an MSE loss instead of NLL loss.}

\begin{table}[h]
    \centering
    \resizebox{0.85\linewidth}{!}{%
    \begin{tabular}{llccc}
        \toprule
        \textbf{Dataset} & \textbf{Variant}
        & \textbf{CFGNP (Ours)} & \textbf{CFGNP (MSE)} & \textbf{Transformer} \\
        & & \multicolumn{3}{c}{\textbf{MSE} $\downarrow$} \\
        \hline

        \multirow{3}{*}{Triangle}
            & LIN
            & $\mathbf{(2.98 \pm 0.52) \times 10^{-6}}$
            & $(4.09 \pm 2.15) \times 10^{-5}$
            & $(8.33 \pm 0.25) \times 10^{-4}$ \\

            & NLIN
            & $\mathbf{(1.44 \pm 0.26) \times 10^{-6}}$
            & $(3.76 \pm 0.68) \times 10^{-5}$
            & $(8.37 \pm 0.46) \times 10^{-4}$ \\

            & NADD
            & $\mathbf{(4.08 \pm 0.42) \times 10^{-5}}$
            & $(1.32 \pm 0.36) \times 10^{-4}$
            & $(5.02 \pm 0.03) \times 10^{-4}$ \\

        \hline
        \addlinespace

        \multirow{3}{*}{M-shape}
            & LIN
            & $\mathbf{(1.67 \pm 0.66) \times 10^{-6}}$
            & $(5.48 \pm 4.91) \times 10^{-5}$
            & $(4.79 \pm 0.12) \times 10^{-3}$ \\

            & NLIN
            & $\mathbf{(4.55 \pm 1.08) \times 10^{-6}}$
            & $(1.42 \pm 2.38) \times 10^{-4}$
            & $(6.04 \pm 0.19) \times 10^{-3}$ \\

            & NADD
            & $\mathbf{(2.04 \pm 0.32) \times 10^{-4}}$
            & $(9.68 \pm 3.85) \times 10^{-4}$
            & $(3.04 \pm 0.17) \times 10^{-3}$ \\

        \hline
        \addlinespace

        German Loan
            & --
            & $\mathbf{(6.57 \pm 3.01) \times 10^{-5}}$
            & $(1.69 \pm 0.97) \times 10^{-4}$
            & $(46.92 \pm 7.72) \times 10^{-2}$ \\

        \bottomrule
    \end{tabular}%
    }
    \caption{Mean MSE $\pm$ standard deviation of the full CFGNP model, the
    CFGNP (MSE) variant, and the Transformer ablation across the Triangle,
    M-shape, and German Loan datasets. Lower values indicate better
    performance. Best scores are marked in bold.}
    \label{tab:transformer_ablation}
    \vspace{-10pt}
\end{table}

\change{
\subsection{Trivial controls ablation}
\label{app:trivial_controls}

To verify that the counterfactual benchmarks cannot be solved by population-level quantities, we evaluate two
controls that ignore the factual unit and the observational context. For a query that sets $X_j = \alpha$, they
predict
\begin{itemize}
    \item the \textbf{interventional mean} $\mathbb{E}[X \mid do(X_j = \alpha)]$ and
    \item the \textbf{conditional mean} $\mathbb{E}[X \mid X_j = \alpha]$,
\end{itemize}
both computed under the \emph{true} SCM by Monte Carlo ($2\cdot 10^{6}$ samples). The interventional mean is thus an oracle for the population effect of the intervention: it is the
best prediction possible without using the individual.
We have also added a copy baseline that copies everything from the factual unit except the intervention index where it uses the intervention value.

Table~\ref{tab:trivial_controls} shows the results. The two controls coincide except on Triangle LIN and NLIN,
the only datasets with queries on a node that has a parent ($X_2$), where conditioning differs from intervening.
CFGNP's error is below both controls on every dataset. Its predictions
therefore capture the individual counterfactual, i.e.\ they use the information in the factual unit, rather than
the average effect of the intervention.

\begin{table}[h]
\centering
\resizebox{\linewidth}{!}{%
\begin{tabular}{lcccc}
\toprule
Dataset & Copy & Interventional mean & Conditional mean & CFGNP \\
\midrule
Loan          & $1.43 \times 10^{-2} \pm 3.7 \times 10^{-4}$ & $1.32 \times 10^{-1} \pm 4.4 \times 10^{-3}$ & $1.33 \times 10^{-1} \pm 4.4 \times 10^{-3}$ & $\mathbf{6.57 \times 10^{-5} \pm 3.01 \times 10^{-5}}$ \\
Triangle LIN  & $1.10 \times 10^{-2} \pm 1.1 \times 10^{-4}$ & $1.25 \times 10^{-2} \pm 9.7 \times 10^{-5}$ & $1.69 \times 10^{-2} \pm 1.1 \times 10^{-4}$ & $\mathbf{2.98 \times 10^{-6} \pm 5.16 \times 10^{-7}}$ \\
Triangle NLIN & $1.22 \times 10^{-2} \pm 1.2 \times 10^{-4}$ & $1.16 \times 10^{-2} \pm 8.5 \times 10^{-5}$ & $1.59 \times 10^{-2} \pm 2.0 \times 10^{-4}$ & $\mathbf{1.44 \times 10^{-6} \pm 2.6 \times 10^{-7}}$ \\
Triangle NADD & $1.61 \times 10^{-2} \pm 1.6 \times 10^{-3}$ & $7.81 \times 10^{-3} \pm 1.1 \times 10^{-4}$ & $7.81 \times 10^{-3} \pm 1.1 \times 10^{-4}$ & $\mathbf{4.08 \times 10^{-5} \pm 4.2 \times 10^{-6}}$ \\
M-shape LIN   & $6.60 \times 10^{-3} \pm 1.3 \times 10^{-4}$ & $5.14 \times 10^{-3} \pm 6.1 \times 10^{-5}$ & $5.14 \times 10^{-3} \pm 6.1 \times 10^{-5}$ & $\mathbf{1.67 \times 10^{-6} \pm 6.58 \times 10^{-7}}$ \\
M-shape NLIN  & $9.44 \times 10^{-3} \pm 3.2 \times 10^{-4}$ & $5.38 \times 10^{-3} \pm 5.7 \times 10^{-5}$ & $5.39 \times 10^{-3} \pm 5.7 \times 10^{-5}$ & $\mathbf{4.55 \times 10^{-6} \pm 1.08 \times 10^{-6}}$ \\
M-shape NADD  & $8.89 \times 10^{-3} \pm 5.4 \times 10^{-4}$ & $3.78 \times 10^{-3} \pm 3.9 \times 10^{-5}$ & $3.78 \times 10^{-3} \pm 3.9 \times 10^{-5}$ & $\mathbf{2.04 \times 10^{-4} \pm 3.18 \times 10^{-5}}$ \\
\bottomrule
\end{tabular}%
}
\caption{Test MSE of the trivial controls and of CFGNP (mean $\pm$ std). Copy returns the factual unit with the intervened variable set to its intervention value. Interventional and conditional mean use the true SCM but have no access to the factual unit. Best results are marked in bold.}
\label{tab:trivial_controls}
\end{table}

}

\change{
\subsection{Baselines with an estimated causal graph.}\label{app:notear}
VACA, BGM and CFM need the causal graph of the data-generating process as input, which CFGNP does not.
To compare the methods with the same information, we replace the true graph by one estimated with
NOTEARS~\cite{zheng_tears} (linear, $\lambda_1 = 0.1$, edge threshold $0.3$, on standardised variables;
remaining cycles are broken by removing their weakest edge) from observational data only, in two ways.
\begin{itemize}
    \item \emph{Pooled}: a single graph estimated from the observational contexts of all training samples taken together.
    \item \emph{Per-sample}: one graph for every test sample, estimated from its own observational context, i.e.\ exactly the observational information CFGNP receives for that query.
\end{itemize}

NOTEARS returns 6--14 distinct graphs per seed on Triangle and 40--69 on
M-shape; to bound the training cost, models are trained only for the ten most frequent graphs of each seed, and the
remaining test samples (at most $0.1\%$ on Triangle, $1.0$--$4.6\%$ on M-shape) are evaluated with the kept graph closest
in structural Hamming distance (SHD). The per-sample setting is run on Triangle and M-shape; Loan uses the pooled graph only.
The estimated graphs are far from the true ones: the pooled estimate has an SHD of 1 (Triangle LIN) to 9 (Loan) and is
empty for both NADD variants, and the mean per-sample SHD ranges from 1.1 to 4.0.
Table~\ref{tab:notears_estimated_graph} reports the counterfactual test MSE, mean $\pm$ standard deviation over five seeds. Again, +CF denotes the baseline trained with counterfactual supervision. On all datasets except for Mshape-LIN the methods with sampled graphs perform worse.

\begin{table*}[t]
\centering
\small
\setlength{\tabcolsep}{4pt}
\begin{adjustbox}{max width=\linewidth}
\begin{tabular}{llcccccc}
\toprule
 & & \multicolumn{3}{c}{Original} & \multicolumn{3}{c}{+ CF loss} \\
\cmidrule(lr){3-5} \cmidrule(lr){6-8}
Dataset & Scale & VACA & BGM & CFM & VACA & BGM & CFM \\
\midrule
\multicolumn{8}{l}{\textit{Pooled graph}} \\
Loan & $10^{-1}$ & $8.89 \pm 0.47$ & $8.02 \pm 0.50$ & $8.14 \pm 0.47$ & $2.83 \pm 0.13$ & $7.42 \pm 0.48$ & $4.47 \pm 0.89$ \\
Triangle LIN & $10^{-3}$ & $2.15 \pm 0.06$ & $4.32 \pm 0.12$ & $3.92 \pm 0.10$ & $0.0266 \pm 0.0095$ & $4.13 \pm 0.10$ & $3.69 \pm 0.09$ \\
Triangle NLIN & $10^{-3}$ & $7.76 \pm 0.24$ & $3.96 \pm 0.20$ & $3.72 \pm 0.35$ & $1.44 \pm 0.02$ & $2.73 \pm 0.09$ & $2.02 \pm 0.05$ \\
Triangle NADD & $10^{-3}$ & $7.81 \pm 0.10$ & $16.19 \pm 1.64$ & $13.59 \pm 1.45$ & $0.0196 \pm 0.0004$ & $0.442 \pm 0.026$ & $0.0349 \pm 0.0019$ \\
M-shape LIN & $10^{-3}$ & $7.04 \pm 0.10$ & $3.72 \pm 0.06$ & $3.47 \pm 0.05$ & $3.66 \pm 0.06$ & $3.72 \pm 0.06$ & $3.45 \pm 0.05$ \\
M-shape NLIN & $10^{-3}$ & $5.71 \pm 0.11$ & $6.14 \pm 0.34$ & $5.21 \pm 0.31$ & $2.65 \pm 0.16$ & $6.14 \pm 0.34$ & $4.74 \pm 0.31$ \\
M-shape NADD & $10^{-3}$ & $3.72 \pm 0.11$ & $8.88 \pm 0.66$ & $7.63 \pm 0.55$ & $3.99 \pm 0.53$ & $8.88 \pm 0.66$ & $7.47 \pm 0.55$ \\
\midrule
\multicolumn{8}{l}{\textit{Per-sample graph}} \\
Triangle LIN & $10^{-3}$ & $2.23 \pm 0.09$ & $4.43 \pm 0.12$ & $4.02 \pm 0.10$ & $0.105 \pm 0.012$ & $4.23 \pm 0.10$ & $3.77 \pm 0.09$ \\
Triangle NLIN & $10^{-3}$ & $6.96 \pm 0.25$ & $3.57 \pm 0.20$ & $3.33 \pm 0.30$ & $1.06 \pm 0.02$ & $2.34 \pm 0.09$ & $1.84 \pm 0.03$ \\
Triangle NADD & $10^{-3}$ & $7.81 \pm 0.10$ & $16.19 \pm 1.64$ & $13.59 \pm 1.45$ & $0.0196 \pm 0.0004$ & $0.439 \pm 0.019$ & $0.0342 \pm 0.0019$ \\
M-shape LIN & $10^{-3}$ & $6.40 \pm 0.28$ & $3.05 \pm 0.05$ & $2.85 \pm 0.05$ & $3.01 \pm 0.05$ & $3.04 \pm 0.05$ & $2.84 \pm 0.05$ \\
M-shape NLIN & $10^{-3}$ & $6.38 \pm 0.16$ & $8.07 \pm 0.37$ & $6.98 \pm 0.34$ & $4.42 \pm 0.19$ & $8.07 \pm 0.37$ & $6.50 \pm 0.35$ \\
M-shape NADD & $10^{-3}$ & $3.72 \pm 0.11$ & $8.81 \pm 0.65$ & $7.57 \pm 0.55$ & $3.97 \pm 0.51$ & $8.81 \pm 0.65$ & $7.41 \pm 0.54$ \\
\midrule
\multicolumn{8}{l}{\textit{True graph}} \\
Loan & $10^{-3}$ & $114.1 \pm 13.4$ & $1.91 \pm 0.25$ & $1.52 \pm 0.15$ & $0.263 \pm 0.043$ & $0.211 \pm 0.028$ & $1.06 \pm 0.10$ \\
Triangle LIN & $10^{-4}$ & $70.2 \pm 7.9$ & $4.52 \pm 1.42$ & $1.72 \pm 0.86$ & $0.716 \pm 0.216$ & $0.156 \pm 0.039$ & $0.426 \pm 0.112$ \\
Triangle NLIN & $10^{-4}$ & $41.2 \pm 4.5$ & $9.72 \pm 7.23$ & $5.54 \pm 2.54$ & $1.47 \pm 0.39$ & $0.102 \pm 0.028$ & $0.265 \pm 0.035$ \\
Triangle NADD & $10^{-4}$ & $78.6 \pm 1.3$ & $2.11 \pm 0.73$ & $5.26 \pm 3.87$ & $4.28 \pm 0.89$ & $0.863 \pm 0.066$ & $2.62 \pm 0.79$ \\
M-shape LIN & $10^{-5}$ & $397 \pm 9$ & $2.35 \pm 0.93$ & $3.71 \pm 0.50$ & $6.03 \pm 2.9$ & $1.43 \pm 0.69$ & $3.89 \pm 1.47$ \\
M-shape NLIN & $10^{-5}$ & $393 \pm 6$ & $2.52 \pm 1.36$ & $5.16 \pm 1.41$ & $10.7 \pm 2.8$ & $1.55 \pm 0.76$ & $5.78 \pm 1.54$ \\
M-shape NADD & $10^{-3}$ & $3.78 \pm 0.04$ & $2.18 \pm 0.09$ & $2.07 \pm 0.05$ & $1.37 \pm 0.28$ & $2.06 \pm 0.15$ & $1.94 \pm 0.04$ \\
\bottomrule
\end{tabular}
\end{adjustbox}
\caption{Counterfactual test MSE of the graph-based baselines when the causal graph is estimated with NOTEARS,
from the pooled observational data of the training set or separately from the context of every test sample.
For reference, the last block shows the same baselines with the true causal graph (Table~\ref{tab:explicit_results}).
Mean $\pm$ standard deviation over five seeds. Values in a row are in units of the scale factor
(e.g.\ $8.89$ with scale $10^{-1}$ is $0.889$).}
\label{tab:notears_estimated_graph}
\end{table*}

\change{\subsection{CFGNP with True Causal Graph}\label{app:true_cg_cfgnp_ablation}}
In this section we want to determine the preferred way of incorporating the true causal graph in CFGNP. There are two approaches we try out. In the first, we replace the fully connected graphs in CFGNP with the true causal graph. However, this breaks the property of CFGNP in App. \ref{proposition_bidirectional}. In the second approach, we keep the fully connected graph but add edge labels to supply the model with the information about the causal graph. The labels are: [\texttt{parent}, \texttt{child}, \texttt{ancestor}, \texttt{descendant}, \texttt{self}, \texttt{twin}, \texttt{to\_context}]

Looking at the results in figure \ref{tab:ablation_graph_edge_features} we can see that supplying CFGNP with the labels performs better across most datasets (6 out of 7).

\begin{table}[H]
\centering
\resizebox{0.9\linewidth}{!}{
\begin{tabular}{llcccc}
\toprule
\multirow{2}{*}{Dataset} & \multirow{2}{*}{Variant} & \multicolumn{2}{c}{w/o context} & \multicolumn{2}{c}{w/ context} \\
\cmidrule(lr){3-4} \cmidrule(lr){5-6}
 & & true graph & graph edge features & true graph & graph edge features \\
\midrule
\multirow{3}{*}{Triangle}
 & LIN  & $(3.78 \pm 0.84) \times 10^{-6}$ & $\mathbf{(2.74 \pm 0.81) \times 10^{-6}}$ & $(1.26 \pm 0.50) \times 10^{-5}$ & $(2.94 \pm 0.87) \times 10^{-6}$ \\
 & NLIN & $(9.29 \pm 1.94) \times 10^{-6}$ & $(4.22 \pm 1.73) \times 10^{-6}$ & $(5.21 \pm 1.30) \times 10^{-6}$ & $\mathbf{(1.49 \pm 0.25) \times 10^{-6}}$ \\
 & NADD & $(4.13 \pm 1.10) \times 10^{-5}$ & $\mathbf{(2.99 \pm 0.87) \times 10^{-5}}$ & $(5.96 \pm 1.14) \times 10^{-5}$ & $(3.82 \pm 0.44) \times 10^{-5}$ \\
\midrule
\multirow{3}{*}{M-shape}
 & LIN  & $(9.67 \pm 5.34) \times 10^{-4}$ & $(7.26 \pm 0.45) \times 10^{-4}$ & $(8.23 \pm 2.47) \times 10^{-6}$ & $\mathbf{(1.56 \pm 0.39) \times 10^{-6}}$ \\
 & NLIN & $(4.97 \pm 0.22) \times 10^{-4}$ & $(5.48 \pm 0.76) \times 10^{-4}$ & $(6.76 \pm 1.00) \times 10^{-6}$ & $\mathbf{(4.60 \pm 0.71) \times 10^{-6}}$ \\
 & NADD & $(7.44 \pm 0.28) \times 10^{-4}$ & $(5.26 \pm 2.56) \times 10^{-4}$ & $\mathbf{(1.19 \pm 0.21) \times 10^{-4}}$ & $(1.91 \pm 0.27) \times 10^{-4}$ \\
\midrule
German Loan & -- & $(7.60 \pm 0.57) \times 10^{-5}$ & $(6.92 \pm 1.08) \times 10^{-5}$ & $(2.31 \pm 0.39) \times 10^{-4}$ & $\mathbf{(5.44 \pm 0.47) \times 10^{-5}}$ \\
\bottomrule
\end{tabular}
}
\caption{Comparison of providing the causal graph explicitly (true graph) versus as graph edge features, with and without observational context (mean MSE $\pm$ standard deviation). All variants are trained on the same counterfactual pairs. Best scores are marked in bold.}
\label{tab:ablation_graph_edge_features}
\end{table}
}

%% file: appendix/amortization_experiments_appendix.tex
\change{
\section{Training CFGNP on a Different Prior over SCMs}\label{app:prior_scms}

We now do amortized learning on the German loan, Triangle, and M-Shape datasets which are used throughout this paper. Therefore, we generate random synthetic causal graphs to train CFGNP.
During test time, we then apply CFGNP to the test dataset of German loan, Triangle, and M-Shape. Each test sample has an observational context of $100$ and a counterfactual query. The observational context represents the "dataset" of observational samples we can use to make a counterfactual prediction.
Again, we compare to the baselines. For a fair comparison, the baseline also only get the $100$ observations for training. Hence, for each sample in the test-set we train a new baseline model. Further, we utilize NOTEARS \citep{zheng_tears} to sample causal graphs on which we subsequently train the baselines. The synthetic train dataset looks as follows:

\paragraph{Random-DAG training prior.}
We train CFGNP on a prior of SCMs with random DAGs on $d \in \{3,5,7\}$ nodes.
Each SCM has a random causal order, and every forward pair of nodes is an edge with probability $p \sim \mathcal{U}(0.3, 1)$.
Graphs that are not weakly connected are redrawn.
Root nodes are drawn from the noise distributions Gaussian, Gaussian mixture, Gamma or Bernoulli.
Every other node follows
\[
  X_j = g_j\bigl(f_j(\mathbf{X}_{\mathrm{pa}(j)}),\, U_j\bigr),
  \qquad
  f_j = \sum_{p \in \mathrm{pa}(j)} w_{jp}\, \phi_{jp}(X_p),
\]
where the parents enter standardised, $f_j$ is normalised to unit scale, and $U_j$ is Gaussian noise.
The per-parent terms $\phi_{jp}$ and the combination $g_j$ are drawn from a set of building blocks. The building blocks are linear, nonlinear additive, non-additive, or a mix of all three. In total, we sample 7350 SCMs which yield 112000 train samples.

For this train set, there is no train/test split, since evaluation uses only the benchmark test sets. We do however split off 5000 validation samples.

Table \ref{tab:random-dag} shows the result when applying all approaches on the test sets. CFGNP performs best across all datasets, although it is close to the baselines on the German Loan dataset.

\begin{table}[H]
  \centering
  \small
  \setlength{\tabcolsep}{3pt}
  \begin{adjustbox}{max width=\linewidth}
  \begin{tabular}{@{}lccccc@{}}
    \toprule
    Dataset & CFGNP (zero-shot) & BGM & VACA & CFM & VCI \\
    \midrule
    German Loan   & $\mathbf{10.33 \pm 0.23}$ & $11.04 \pm 0.14$ & $135.32 \pm 3.35$ & $10.67 \pm 0.18$ & $359.28 \pm 6.86$ \\
    Triangle LIN  & $\mathbf{7.44 \pm 0.13}$  & $18.13 \pm 0.31$ & $17.20 \pm 0.16$  & $17.15 \pm 0.27$ & $20.99 \pm 0.09$ \\
    Triangle NLIN & $\mathbf{9.17 \pm 0.14}$  & $12.46 \pm 0.20$ & $15.54 \pm 0.12$  & $11.43 \pm 0.18$ & $17.57 \pm 0.26$ \\
    Triangle NADD & $\mathbf{2.26 \pm 0.05}$  & $16.03 \pm 1.64$ & $7.78 \pm 0.12$   & $15.97 \pm 1.64$ & $15.40 \pm 0.38$ \\
    M-shape LIN   & $\mathbf{1.31 \pm 0.02}$  & $3.14 \pm 0.07$  & $7.06 \pm 0.03$   & $3.11 \pm 0.07$  & $6.98 \pm 0.07$ \\
    M-shape NLIN  & $\mathbf{1.55 \pm 0.03}$  & $8.05 \pm 0.31$  & $7.18 \pm 0.05$   & $7.84 \pm 0.31$  & $9.64 \pm 0.28$ \\
    M-shape NADD  & $\mathbf{2.62 \pm 0.05}$  & $8.87 \pm 0.52$  & $3.93 \pm 0.04$   & $8.85 \pm 0.53$  & $4.96 \pm 0.13$ \\
    \bottomrule
  \end{tabular}
  \end{adjustbox}
  \caption{Test loss ($\times 10^{-3}$, lower is better) on the benchmark test sets with mean $\pm$ std.
    CFGNP is trained only on the random-DAG prior and applied zero-shot.
    BGM, VACA and CFM are fit to each test sample's own context of 100 observations, using a NOTEARS-estimated graph; VCI is fit to the same context without a graph.
    Best per row in bold.}
  \label{tab:random-dag}
\end{table}
}